\documentclass[nohyperref]{article}

\usepackage{microtype}
\usepackage{graphicx}
\usepackage{subfigure}
\usepackage{booktabs} 

\usepackage{hyperref}

\usepackage[accepted]{icml2023toc}

\usepackage{amsmath}
\usepackage{amssymb}
\usepackage{mathtools}
\usepackage{amsthm}

\usepackage[capitalize]{cleveref}

\theoremstyle{plain}
\newtheorem{theorem}{Theorem}[section]

\newtheorem{lemma}[theorem]{Lemma}
\newtheorem{corollary}[theorem]{Corollary}
\theoremstyle{definition}

\newtheorem{assumption}[theorem]{Assumption}
\theoremstyle{remark}
\newtheorem{remark}[theorem]{Remark}

\usepackage[textsize=tiny]{todonotes}

\usepackage{xcolor}
\usepackage{colortbl}
\usepackage{tabularx, makecell, booktabs}

\usepackage[toc,page,header]{appendix}
\usepackage[]{minitoc}
\renewcommand{\partname}{}
\renewcommand{\thepart}{}
\usepackage{fancyhdr}
\usepackage{nicefrac}

\makeatletter
\newcommand{\vast}{\bBigg@{3}}
\newcommand{\Vast}{\bBigg@{5}}
\makeatother

\newcommand{\bbE}{\mathbb{E}}
\newcommand{\bbI}{\mathbb{I}}
\newcommand{\bbR}{\mathbb{R}}
\newcommand{\bbG}{\mathbb{G}}

\newcommand{\calF}{\mathcal{F}}

\newcommand{\calM}{\mathcal{M}}

\newcommand{\calS}{\mathcal{S}}
\newcommand{\calO}{\mathcal{O}}
\newcommand{\calA}{\mathcal{A}}
\newcommand{\calP}{\mathcal{P}}
\newcommand{\wt}{\widetilde}

\newcommand{\ind}{\mathbb{I}}

\newcommand{\indevent}[1]{\ind \{ #1 \}}
\newcommand{\regret}{R_K}

\newcommand{\sinit}{s_{\text{init}}}

\newcommand{\logterm}{\iota}
\renewcommand{\dim}{n}
\newcommand{\trace}{\mathsf{tr}}

\newcommand{\norm}[1]{\left\|{#1}\right\|}

\renewcommand{\l}{\left}
\renewcommand{\r}{\right}

\newcommand{\ol}[1]{\overline{#1}}
\newcommand{\ul}[1]{\underline{#1}}

\newcommand{\dmax}{d_{max}}
\newcommand{\trajset}{\mathcal{T}^j}

\icmltitlerunning{Delay-Adapted Policy Optimization and Improved Regret for Adversarial MDP with Delayed Bandit Feedback}

\begin{document}

\doparttoc[n]
\faketableofcontents 

\twocolumn[
\icmltitle{Delay-Adapted Policy Optimization and Improved Regret\\for Adversarial MDP with Delayed Bandit Feedback}



\icmlsetsymbol{intr}{*}

\begin{icmlauthorlist}
\icmlauthor{Tal Lancewicki}{tau}
\icmlauthor{Aviv Rosenberg}{amzn}
\icmlauthor{Dmitry Sotnikov}{amzn}
\end{icmlauthorlist}

\icmlaffiliation{tau}{Tel Aviv University (Research conducted while the author was an intern at Amazon Science)}
\icmlaffiliation{amzn}{Amazon Science}

\icmlcorrespondingauthor{Tal Lancewicki}{lancewicki@mail.tau.ac.il}
\icmlcorrespondingauthor{Aviv Rosenberg}{avivros007@gmail.com}

\icmlkeywords{Online Learning, Regret Minimization, Adversarial MDP, Reinforcement Learning Theory, Delayed Feedback}

\vskip 0.3in
]



\printAffiliationsAndNotice{}  

\begin{abstract}
Policy Optimization (PO) is one of the most popular methods in Reinforcement Learning (RL). Thus, theoretical guarantees for PO algorithms have become especially important to the RL community. In this paper, we study PO in adversarial MDPs with a challenge that arises in almost every real-world application -- \textit{delayed bandit feedback}. We give the first near-optimal regret bounds for PO in tabular MDPs, and may even surpass state-of-the-art (which uses less efficient methods). Our novel Delay-Adapted PO (DAPO) is easy to implement and to generalize, allowing us to extend our algorithm to: (i) infinite state space under the assumption of linear $Q$-function, proving the first regret bounds for delayed feedback with function approximation. (ii) deep RL, demonstrating its effectiveness in experiments on MuJoCo domains.
\end{abstract}

\section{Introduction}
\label{sec:intro}





Policy Optimization (PO) is one of the most widely-used methods in Reinforcement Learning (RL). 
It has demonstrated impressive empirical success \cite{levine2013guided,schulman2017proximal,haarnoja2018soft}, leading to increasing interest in understanding its theoretical guarantees.
While in recent years we have seen great advancement in theory of PO \cite{shani2020optimistic,luo2021policy,chen2022policy}, our understanding is still very limited when considering \textit{delayed feedback} -- an important challenge that occurs in most practical applications. For example, recommendation systems often learn the utility of a recommendation based on the number of user conversions, which may happen with a variable delay after the recommendation was issued. Other notable examples include communication between agents \citep{chen2020delay}, video streaming \citep{changuel2012online} and robotics \citep{mahmood2018setting}.
To mitigate the gap in the PO literature, we study PO in the challenging adversarial MDP model (i.e., costs change arbitrarily) under bandit feedback with arbitrary unrestricted delays.

PO with delays was previously studied by \citet{lancewicki2020learning}, but their regret bounds are far from optimal and scale with $(K+D)^{2/3}$, where $K$ is the number of episodes and $D$ is the total delay.
Recently, \citet{jin2022near} achieved near-optimal $\wt \calO (H^2 S \sqrt{AK} + (HSA)^{1/4} H \sqrt{D})$ regret (ignoring logarithmic factors), where $S$, $A$ and $H$ are the number of states, actions, and the episode length, respectively.
However, their algorithm is not based on PO, but on the O-REPS method \cite{zimin2013online} which requires solving a computationally expensive global optimization problem and cannot be extended to function approximation (FA).
On the other hand, PO algorithms build on highly efficient local-search and extend naturally to FA \cite{tomar2020mirror}.

\textbf{Our Contributions.}
In this paper, we vastly expand our understanding of PO and delayed feedback.
We propose a novel Delay-Adapted PO method, called DAPO, which measures changes in the agent's policy over the time of the delays and adapts its updates accordingly. 
First, we establish the power of DAPO in tabular MDPs, i.e., finite number of states and actions.
We prove DAPO attains the first near-optimal regret bound for PO with delayed feedback $\wt \calO (H^3 S \sqrt{A K} + H^3 \sqrt{D})$.
This bound is tighter than \cite{jin2022near} when the delay term is dominant and the number of states is significantly larger than the horizon (which is the common case).
Moreover, it matches the lower bound of \citet{lancewicki2020learning}  in the delay term up to factors of $H$, showing for the first time that the delay term in the regret does not need to scale with $S$ or $A$.
Importantly, if there is no delay, it matches the best known regret for PO \cite{luo2021policy}.

Next, we show DAPO is easy to implement and naturally extends to function approximation in two important settings:
\begin{enumerate}
    \item \textit{Linear-$Q$.} We extend DAPO to MDPs with linear FA under standard assumptions \cite{luo2021policy} that $Q$-functions are linear in some known low-dimensional features and also a simulator is available. We prove that DAPO achieves the first sub-linear regret for delayed feedback with FA, i.e., non-tabular MDP. 
    
    \item \textit{Deep RL.} We show that the famous PPO algorithm \cite{schulman2017proximal} can be easily combined with DAPO, and demonstrate superior empirical performance even in the presence of simple delays in experiments on MuJoCo domains \cite{todorov2012mujoco}.
\end{enumerate}

Throughout the paper we handle several technical challenges which are unique to PO algorithms with delayed feedback. The main challenge is to control the stability of the algorithm. This problem is more challenging in MDPs compared to multi-armed bandit (where there is no transition function to estimate), and is enhanced even further in Policy Optimization algorithms due to their local-search nature, as opposed to the global update of O-REPS methods -- see more details in the proof sketch of \cref{thm:tabular-main regret-bound}. Further, the linear-$Q$ setting with delayed feedback was not studied before and requires a careful new algorithmic design and analysis. In particular, a-priori, it is highly unclear how to design a delay-adapted estimator and delay-adapted bonus term which sufficiently stabilize the algorithm -- see more details in \cref{sec:linear-main}.

While the main contribution of this paper is the novel Delay-Adapted PO method, we also make substantial technical contributions that might be of independent interest.
Our algorithms are based on \textit{PO with dilated bonuses} \cite{luo2021policy}, which dilate towards further horizons and do not satisfy standard Bellman equations. However, we are able to achieve the same regret guarantees without using dilated bonuses.
Instead, we compute local bonuses and use them to construct a $Q$-function that operates as exploration bonuses.
This has an important practical benefit -- now bonuses can be approximated similarly to the $Q$-function. It also has a theoretical benefit -- it greatly simplifies the analysis, making it more natural and easy to extend to new scenarios.
Moreover, utilizing our new simplified analysis, we are able to give regret guarantees with high probability in the Linear-$Q$ setting and not just in expectation (as in \citet{luo2021policy}).
Finally, we also develop new analyses for handling delayed feedback when losses can be negative.
This was not addressed in the delayed multi-armed bandit literature (or in previous papers on delays in MDPs), but cannot be avoided in our case since exploration bonuses are crucial to guaranteeing near-optimal regret but they might turn losses to negative.

\subsection{Additional Related Work}

Due to lack of space, this section only gives a brief overview of related work - for a full literature review see \cref{appendix:related-work}.

There is a rich literature on regret minimization in tabular MDPs, initiated with the seminal UCRL algorithm \cite{jaksch2010near} for stochastic losses that is based on the fundamental concept of \textit{Optimism Under Uncertainty}. 
Their model was later extended to the more general adversarial MDP, where most algorithms are based on either the framework of occupancy measures (a.k.a, O-REPS) \cite{zimin2013online,jin2019learning} or on the more practical PO \cite{even2009online,shani2020optimistic}. 
In recent years this line of research was extended beyond the tabular model to linear function approximation. 
For stochastic losses, existing algorithms are mostly based on optimism \cite{jin2020provably}, whereas most algorithms for adversarial losses are based on PO \cite{cai2020provably,luo2021policy,neu2021online} which extends much more naturally than O-REPS to function approximation. 
On the practical side, some of the most successful deep RL algorithms are built upon PO principles. 
These include the famous Trust Region PO (TRPO; \citet{schulman2015trust}) as well as Proximal PO (PPO; \citet{schulman2017proximal}) which we will further discuss and adapt to delayed feedback in \cref{sec:PPO}. 

Regret minimization with delayed feedback was initially studied in Online Optimization and Multi-armed bandit (MAB) in both the stochastic setting \citep{agarwal2012distributed,pike2018bandits} and the adversarial setting \citep{cesa2016delay,thune2019nonstochastic}. 
As a natural extension, this line of work was generalized to delayed feedback in MDPs, where \cite{howson2021delayed} consider the more restrictive stochastic model. 
Most related to our work are the works of \citet{lancewicki2020learning,jin2022near} that were mentioned earlier, and the work of \citet{dai2022follow}.
They recently showed that Follow-The-Perturbed-Leader (FTPL) algorithms can also handle delayed feedback in adversarial MDPs.
The efficiency of FTPL is similar to PO, but their regret bound is slightly weaker than \citet{jin2022near}.
Finally, a different line of work \citep{katsikopoulos2003markov,walsh2009learning} consider delays in observing the current state.
That setting is inherently different than ours (see \cref{appendix:related-work} for more details).

\section{Preliminaries}

A finite-horizon episodic adversarial MDP is defined by a tuple $\calM = (\calS , \calA , H , p, \{ c^{k} \}_{k=1}^K)$, where
$\calS$ and $\calA$ are state and action spaces of sizes $|\calS| = S$ and $|\calA| = A$, respectively, $H$ is the horizon and $K$ is the number of episodes. 
$p: \calS \times \calA \times [H] \to \Delta_{\calS}$ is the \textit{transition function} such that $p_h(s' | s,a)$ is the probability to move to $s'$ when taking action $a$ in state $s$ at time $h$. 
$\{ c^{k} : \calS \times \calA \times [H] \to [0,1] \}_{k=1}^K$ are \textit{cost functions} chosen by an \textit{oblivious adversary}, where $c_h^k(s,a)$ is the cost for taking action $a$ at
$(s,h)$ in episode $k$.

A \textit{policy} $\pi: \calS \times [H] \to \Delta_{\calA}$ is a function that gives the probability $\pi_h(a | s)$ to take action $a$ when visiting state $s$ at time $h$.
The value $V^{\pi}_h(s ; c)$ is the expected cost of $\pi$ with respect to cost function $c$ starting from $s$ in time $h$, i.e.,
$V_{h}^{\pi}(s ; c) = \bbE \Bigl[ \sum_{h'=h}^{H} c_{h'}(s_{h'},a_{h'}) | \pi , s_{h}=s \Bigr]$,
where the expectation is with respect to policy $\pi$ and transition function $p$,
that is, $a_{h'} \sim \pi_{h'}(\cdot | s_{h'})$ and $s_{h'+1} \sim p_{h'}(\cdot | s_{h'},a_{h'})$.
The \textit{Q-function} is defined by $Q_{h}^{\pi}(s,a ; c) = c_{h}(s,a)+\langle p_h(\cdot| s,a) , V_{h+1}^{\pi}(\cdot ; c)\rangle$,  where $\langle \cdot , \cdot \rangle$ is the dot product.

The learner interacts with the environment for $K$ episodes.
At the beginning of episode $k$, it picks a policy $\pi^k$, and starts in an initial state $s^k_1 = \sinit$.
In each time $h\in [H]$, it observes the current state $s^k_h$, draws an action from the policy $a^k_h \sim \pi^k_h(\cdot | s_h^k)$ and transitions to the next state $s^k_{h+1} \sim p_h(\cdot | s^k_h,a^k_h)$. 
The feedback of episode $k$ contains the cost function over the agent's trajectory $\{ c^k_h(s^k_h,a^k_h) \}_{h=1}^H$, i.e., bandit feedback. 
This feedback is observed only at the end of episode $k+d^k$, where the \textit{delays} $\{d^k\}_{k=1}^K$ are unknown and chosen by the adversary together with the costs.\footnote{If $d^k \equiv 0$, we get standard online learning in adversarial MDP.}

The goal of the learner is to minimize the \textit{regret}, defined as the difference between the learner's cumulative expected cost and the best fixed policy in hindsight:
$$
    \regret
    =
    \sum_{k=1}^K V^{\pi^k}_1(\sinit;c^k) - \min_{\pi} \sum_{k=1}^K V^{\pi}_1(\sinit;c^k).
$$

In \cref{sec:tabular-main} we consider \textit{tabular MDPs}, i.e., MDPs with a finite number of states and actions. In \cref{sec:linear-main} we consider the more general case that allows infinite number of states but under the assumption that the $Q$-function is linear for all policies. We follow the standard definition \cite{abbasi2019politex,neu2021online,luo2021policy}, which also assumes  the number of actions is finite.

\begin{assumption}[Linear-$Q$]
    \label{ass:linear-Q}
    Let $\phi: \calS \times \calA \times [H] \to \bbR^\dim$ be a known feature mapping. Assume that for every episode $k$, policy $\pi$ and step $h$ there exist an unknown vector $\theta^{k,\pi}_h \in \bbR^\dim$ such that $Q^{\pi}_h(s,a;c^k) = \phi_h(s,a)^\top \theta^{k,\pi}_h$ for all $(s,a)$. Moreover, $\Vert \phi_h(s,a) \Vert_2 \le 1$ and $\Vert \theta^{k,\pi}_h \Vert_2 \le H\sqrt{\dim}$.
\end{assumption}

\textbf{Additional notations.}
Episode indices appear as superscripts and in-episode steps as subscripts. 
The total delay is $D=\sum_k d^k$, the maximal delay is $\dmax$ and the number of episodes that their feedback arrives in the end of episode $k$ is $m^k = |\{ j: j+d^j = k\}|$. 
The occupancy measure $q^\pi_h(s,a) = \Pr[s_h=s,a_h=a | \pi,s_1=\sinit]$ is the distribution that policy $\pi$ induces over state-action pairs in step $h$, and $q^\pi_h(s) = \sum_{a \in \calA} q^\pi_h(s,a)$.
The notations $\wt{\mathcal{O}}(\cdot)$ and $\lesssim$ hide poly-logarithmic factors including $\log (K/\delta)$ for confidence parameter $\delta$.
$[n] = \{1,2,\dots,n\}$ and the indicator of event $E$ is $\indevent{E}$.
Finally, denote by $\pi^*$ the best fixed policy in hindsight and use the notations $V^k_h(s),Q^k_h(s,a),q^k_h(s,a)$ when the policy and cost are $\pi^k$ and $c^k$, respectively.

\textbf{Simplifying assumptions.}
Similarly to \citet{jin2022near}, we assume that $K$, $D$ and $\dmax$ are known. 
This assumption simplifies presentation and can be easily removed with \textit{doubling} for delayed feedback \cite{bistritz2021no,lancewicki2020learning}.
Bounds in the main text hide low-order terms and additive dependence in $\dmax$ (see \cref{remark:dmax-dependence} on removing $\dmax$ dependence). 
For full bounds see Appendix.

\section{DAPO for Tabular MDP}
\label{sec:tabular-main}

In this section we present our novel Delayed-Adapted Policy Optimization algorithm (DAPO; presented in \cref{alg:delayed-OPPO-known-p main text}) for the tabular case, where the number of states is finite. 
We use this fundamental model to develop a generic method for handling delayed feedback with Policy Optimization. 
Our approach consists of two important algorithmic features: a new \textit{delay-adapted importance-sampling estimator} for the $Q$-function and a novel \textit{delay-adapted bonus term} to drive exploration. 
Remarkably, this method extends naturally to both linear function approximation as we show in \cref{sec:linear-main}, and to deep RL with the extremely practical PPO algorithm \cite{schulman2017proximal} as we show in \cref{sec:PPO}. 

To simplify presentation and focus on the contributions of our delay adaptation method, in this section we assume that the agent knows the transition function in advance.
Generalizing DAPO to unknown transitions in the tabular case is fairly straightforward, and follows the common approach of optimism and confidence sets \cite{jaksch2010near}.
Due to lack of space, in the main text we only provide sketches for the algorithms and proofs.
The full versions (for both known and unknown transitions), together with the detailed analyses, can be found in \cref{app:tabular known,app:tabular unknown}.

\begin{algorithm}[t]
    \caption{DAPO with Known Transitions (Tabular)}  
    \label{alg:delayed-OPPO-known-p main text}
    \begin{algorithmic}
        
        \STATE \textbf{Initialization:} 
        Set $\pi_{h}^{1}(a \mid s) = \nicefrac{1}{A}$ for every $(s,a,h)$.
        
        \FOR{$k=1,2,\dots,K$}
            
            \STATE Play episode $k$ with policy $\pi^k$.
            
            \STATE {\color{gray} \# Policy Evaluation}
            
            \FOR{$j$ such that $j + d^j = k$}
            
                \STATE Observe bandit feedback $\{ c^j_h(s^j_h,a^j_h) \}_{h=1}^H$.
                \STATE Compute delay-adapted estimator $\hat{Q}_h^j(s,a)$ defined in \cref{eq: tabular-main Q}.
                \STATE Compute delay-adapted bonus $B_h^j(s,a)$ as the $Q$-function of $\pi^j$ with respect to the costs $b_h^j(s)$ defined in \cref{eq: tabular-main b}. 
                I.e., compute recursively for $h = H,...,1$,  $B_{h}^{j}(s,a) =b_{h}^{j}(s) + \bbE_{s'\sim p_h(\cdot\mid s,a), a' \sim \pi_{h+1}^{j}(\cdot \mid s') } B^j_{h+1}(s',a')$.
                
            \ENDFOR
            
            \STATE {\color{gray} \# Policy Improvement}
            
            \STATE Define the policy $\pi^{k+1}$ for every $(s,a,h)$ by:
            \begin{align}
            \label{eq: tabular-main ewu}
                 \pi^{k+1}_h(a \mid s) \propto  e^{-\eta \sum_{j: j+d^j \leq  k}  ( \hat Q_{h}^{j}(s,a) - B_h^j(s,a) ) }
            \end{align}

        \ENDFOR
    \end{algorithmic}
\end{algorithm}

PO algorithms follow the algorithmic paradigm of Policy Iteration (see, e.g., \citet{sutton2018reinforcement}). That is, in every iteration they perform an evaluation of the current policy and then a step of policy improvement.
The improvement step is regularized to be ``soft'', and is practically implemented by running an online multi-armed bandit algorithm, such as Hedge \cite{freund1997decision}, locally in each state.
The losses that are fed to the algorithm are the estimated $Q$-functions, but in order to achieve the optimal regret, they are also combined with a bonus term which aims to stabilize the algorithm and drive exploration (keeping the estimated $Q$-function optimistic).
The actual policy update step is based on exponential weights and presented in \cref{eq: tabular-main ewu}. 

DAPO adapts to delays through the policy evaluation step.
Remarkably, it adapts to delays near-optimally by computing the following  simple ratio, which measures the local change in the agent's policy through the time of the delay,
\begin{align}
\label{eq:delay-adapted-ratio}
    r^k_h(s,a) = \frac{\pi^k_h(a \mid s)}{\max \{ \pi^k_h(a\mid s) , \pi_{h}^{k+d^k}(a \mid s) \}}.
\end{align}
In order to get our delay-adapted $Q$-function estimation, we simply multiply $r^k_h(s,a)$ by the standard importance-sampling estimator from the non-delayed setting \cite{luo2021policy}. 
The result is:
\begin{align}
     \hat{Q}_{h}^{k}(s,a) = r^k_h(s,a) \cdot \frac{ \mathbb{I}\{s_{h}^{k}=s,a_{h}^{k}=a\} L^k_h }{ q_{h}^{k}(s)\pi_{h}^{k}(a \mid s) +\gamma},
     \label{eq: tabular-main Q}
\end{align}
where $L^k_h = \sum_{h'=h}^H c^k_{h'}(s^k_{h'},a^k_{h'})$ is the realized cost-to-go from step $h$ and $\gamma$ is an exploration parameter \cite{neu2015explore} needed to guarantee regret with high probability.

Intuitively, incorporating $r^k_h(s,a)$ helps us control the variance of the estimator $\hat Q^k_h(s,a)$ in the presence of delays, since it will be used only in episode $k+d^k$ where actions are chosen according to $\pi^{k+d^k}$ and not $\pi^k$. 
The maximum in the denominator is needed in order to keep the bias small.
We note that this estimator is inspired by \citet{jin2022near}, but there are two major differences.
First, as \citet{jin2022near} perform their update globally in the space of state-action occupancy measures, their adaptation occurs in the state space as well. 
On the other hand we perform the update locally in each state, so our adaptation takes place only in the action space. 
Second, they directly change importance-sampling weighting, while we simply multiply standard estimators by the delay-adapted ratio. 
This seemingly minor
nuance is critical in more complex (non-tabular) regimes, where its generality allows to utilize existing procedures from the non-delayed case (see more details in \cref{sec:linear-main,sec:PPO}).

Finally, to complement our new estimator, we devise an appropriate delay-adapted bonus $B_h^k(s,a)$ based on the following delay-adapted local bonus (again obtained by combining $r^k_h(s,a)$ with the original local bonus of \citet{luo2021policy}),
\begin{align}
    b^k_h(s) = \sum_{a \in \calA} r^k_h(s,a) \cdot \frac{3 \gamma H \pi^{k + d^k}_h(a \mid s) }{q^{k}_h(s) \pi^k_h(a \mid s) + \gamma}.
    \label{eq: tabular-main b}
\end{align}
At this point, \citet{luo2021policy} compute $B^k_h(s,a)$ using a dilated Bellman equation that is not very intuitive.
Instead, we compute $B^k_h(s,a)$ with the regular Bellman equations, making it a proper $Q$-function.
This is an important contribution that might be of independent interest for two reasons -- theoretically the analysis becomes much simpler, and practically the bonuses can be approximated like a $Q$-function.

Next, we present the regret guarantees of DAPO in tabular MDPs, and the key steps in the analysis, which highlight the intuition behind our algorithm design.
\begin{theorem}
    \label{thm:tabular-main regret-bound}
    Running DAPO in a tabular adversarial MDP guarantees with probability $1 - \delta$, for known transition,
    $$
        \regret
        =
        \tilde \calO ( H^2 \sqrt{S A K}  + H^3 \sqrt{K + D} )
    $$
    when setting $\eta = \l( H^{2}SAK+H^{4}(K+D) \r)^{-1/2}$ and $\gamma = 2 \eta H$, and for unknown transition,
    $$
        \regret
        =
        \tilde \calO ( H^{3}S \sqrt{AK} 
                        + H^{3}\sqrt{D} 
                                        )
    $$
    when setting $\eta = H \l( H^2 SAK+H^{4}(K+D) \r)^{-1/2}$ and $\gamma = 2 \eta H$.
\end{theorem}
This is a big improvement compared to the best known regret for PO  \cite{lancewicki2020learning} that scales as $(K+D)^{2/3}$ (ignoring dependencies in $H,S,A$).
It is also better than the current state-of-the-art regret bound of \citet{jin2022near} in the case that there is significant delay and $S \gg H$ (which occurs in almost every practical application).
While their bound has better dependency in $H$ (this is a known weakness of PO \cite{chen2022policy}), we improve the dependency in $S$ and $A$. 
\citet{lancewicki2020learning} also show a lower bound of $\Omega(H^{3/2}\sqrt{SAK} + H\sqrt{D})$.
Thus, our bound shows for the first time that under the optimal regret, the delay term does not scale with $S$ or $A$. The first term in our regret bound matches the state-of-the-art regret of non-delayed PO \cite{luo2021policy}, and matches the best known regret for (non-delayed) adversarial MDPs in general up to factors of $H$ \cite{jin2019learning}.
Moreover, DAPO is the first efficient algorithm to be consistent with the optimal regret in delayed MAB, i.e., for $H = S = 1$ we get the optimal regret of \citet{thune2019nonstochastic,bistritz2019online}.
Finally, it is important to emphasize that PO algorithms are much more practical than O-REPS algorithms, and extend naturally to function approximation, as we show in \cref{sec:linear-main,sec:PPO}. 

\begin{proof}[Proof sketch of \cref{thm:tabular-main regret-bound}]
    Much of the intuition for PO algorithms stems from a classic regret decomposition known as the value difference lemma \cite{even2009online}: $\regret   =  \sum_{k,h}\bbE_{s \sim q_{h}^{*}} \langle \pi_{h}^{k} (\cdot \mid s) - \pi_{h}^{*} (\cdot \mid s),Q_{h}^{k}(s, \cdot) \rangle $.
    Fixing a state $s$ and step $h$, the sum over $k$ can be viewed as the regret of an online experts algorithm (e.g., Hedge) with respect to the losses $Q_h^k(s,\cdot)$. 
    We propose to further decompose the regret as follows,
    \begin{align}
        \nonumber
        & \regret  =  \underbrace{
        \sum_{k,h}\bbE_{s \sim q_{h}^{*}} \langle \pi_{h}^{k + d^{k}} (\cdot \mid s),Q_{h}^{k}(s, \cdot) - \hat{Q}_{h}^{k}(s, \cdot) \rangle }
        _{\textsc{Bias}_{1}} 
        \\
        \nonumber
        & + \underbrace{
        \sum_{k,h}\bbE_{s \sim q_{h}^{*}} \langle \pi_{h}^{*} (\cdot \mid s), \hat{Q}_{h}^{k}(s, \cdot) - Q_{h}^{k}(s, \cdot) \rangle 
        }_{\textsc{Bias}_{2}} 
        \\
        \label{eq:main regret decomposition}
        &
        + \underbrace{ 
        \sum_{k,h}\bbE_{s \sim q_{h}^{*}} \langle \pi_{h}^{k} (\cdot \mid s) - \pi_{h}^{*} (\cdot \mid s),B_{h}^{k}(s, \cdot) \rangle 
        }_{\textsc{Bonus}} 
        \\
        \nonumber
        & + \underbrace{ 
        \sum_{k,h}\bbE_{s \sim q_{h}^{*}} \langle \pi_{h}^{k + d^{k}} (\cdot \mid s) - \pi_{h}^{*} (\cdot \mid s), \hat{Q}_{h}^{k}(s, \cdot) - B_{h}^{k}(s, \cdot) \rangle 
        }_{\textsc{Reg}} 
        \\
        \nonumber
        & + \underbrace{ 
        \sum_{k,h}\bbE_{s \sim q_{h}^{*}} \langle \pi_{h}^{k} (\cdot \mid s) - \pi_{h}^{k + d^{k}} (\cdot \mid s),Q_{h}^{k}(s, \cdot) - B_{h}^{k}(s, \cdot) \rangle 
        }_{\textsc{Drift}}.
    \end{align}
    Indeed, the policy update step in \cref{eq: tabular-main ewu} is an Hedge-style exponential weights update.
    This allows us to bound $\textsc{Reg}$ term as it represents the regret of Hedge with respect to the losses $\hat Q_h^k(s,a) - B_h^k(s,a)$.
    Note that the delayed feedback causes a shift of the agent's policies from $\pi^k$ to $\pi^{k+d^k}$.
    As a result, we can bound (using \cref{corollary:delayed-exp-weights-regret} in \cref{app:auxil}):
\begin{align}
\nonumber
    \textsc{Reg} & \lesssim
    \frac{H}{\eta} + \eta \sum_{k,h,s,a} q_h^{*}(s) \pi_h^{k + d^k}(a| s)  (\hat Q_h^k(s,a) - B^k_h(s,a))^2
    \\
\nonumber
    & \lesssim
    \frac{H}{\eta} +  \eta H^5 K + \underbrace{ \eta \sum_{k,h,s,a} q_h^{*}(s) \pi_h^{k + d^k}(a| s) \hat Q_h^k(s,a)^2}_{(*)}
\end{align}
where the second inequality is since $|B^k_h(s,a)| \lesssim H^2$.
To bound the last term, we start with a concentration bound.
This allows us to substitute the indicator in \cref{eq: tabular-main Q} by its expectation (which is $q^k_h(s,a)$) and cancel out the denominator once.
The resulting bound is:
\begin{align}
    (*) \lesssim  \eta H^{2}\sum_{k,h,s,a}
                                                          \frac{q_h^{*}(s)\pi_{h}^{k + d^k}(a\mid s)}{q_h^k(s)\pi_{h}^{k}(a\mid s) + \gamma}  r^k_h(s,a).
    \label{eq:tabular-main first reg bound}
\end{align}
Now, the first issue we need to address is the mismatch between $q^*_h(s)$ in the nominator and $q^k_h(s)$ in the denominator. 
A similar challenge also arises in the non-delayed analysis, but in the case of delayed feedback, this requires a carefully constructed delay-adapted local bonus (defined \cref{eq: tabular-main b}).
Note that our definition of $b_h^k(s)$ with $\eta = 3\gamma/H$ implies that \cref{eq:tabular-main first reg bound} is equal to $\sum_{k,h,s} q_h^*(s)b_h^k(s)$.
Next, we apply the value difference lemma a second time, to show that $\textsc{Bonus} = \sum_{k,h,s} q_h^k(s)b_h^k(s) - \sum_{k,h,s} q_h^*(s)b_h^k(s)$. 
Essentially, this means that by summing $\textsc{Reg}$ and $\textsc{Bonus}$, we can substitute $q^*_h(s)$ in \cref{eq:tabular-main first reg bound} by $q^k_h(s)$. 

The second issue, which is unique to delayed feedback, is the mismatch between $\pi_{h}^{k + d^k}(a | s)$ and $\pi_{h}^{k}(a | s)$. 
It is important to note that while in MAB $\pi_{h}^{k + d^k}(a | s) / \pi_{h}^{k}(a | s)$ is always bounded by a constant, in MDPs this ratio can be as large as $e^{\dmax}$ (see \cref{remark:ratio} in \cref{app:tabular known}). 
Thus, the standard importance-sampling estimator (without $r_h^k(s,a)$) will not work in this type of analysis. 
The main idea behind our delay-adapted estimator is that $r_h^k(s,a)\pi_{h}^{k + d^k}(a | s) \leq \pi_{h}^{k}(a | s)$ which guarantees that the ratio is simply bounded by $1$.
Overall, we get $(*) + \textsc{Bonus} \lesssim \eta H^2 \sum_{k,h,s,a} 1 = \eta H^3 S A K$, and then, $\textsc{Reg} + \textsc{Bonus} \lesssim \frac{H}{\eta} + \eta H^5 K + \eta H^3 S A K$. 

While $r_h^k(s,a)$ in our delay-adapted estimator reduces variance, it increases bias. 
Remarkably, this additional bias scales similarly to the $\textsc{Drift}$ term. 
More specifically, for $\textsc{Bias}_1$ we first use a variant of Freedman’s inequality which is highly sensitive to the estimator's variance. 
This brings similar issues to the ones we faced in \cref{eq:tabular-main first reg bound}, which are treated in a similar manner. 
Then, we show that the additional bias introduced by the ratio $r_h^k(s,a)$ scales as
\begin{align}
\nonumber
    \sum_{k,h,s,a} & q_h^*(s) \pi^{k + d^k}_h(a | s) (1-r_h^k(s,a)) Q^k_h(s,a)
    \\
    &
    \leq 
    H \sum_{k,h,s} q_h^*(s) \Vert \pi^{k + d^k}_h( \cdot | s) - \pi^{k}_h( \cdot | s) \Vert_1,
    \label{eq:tabular-main additional bias}
\end{align}
where the inequality follows by plugging in the definition of $r_h^k(s,a)$ and some simple algebra (and $|Q^k_h(s,a)| \le H$).

Utilizing the exponential weights update form, we bound the $\ell_1$-distance above by $\sum_{j\in M^k}\sum_a \pi_h^{j + d^j}(a | s) \hat{Q}_h^j(s,a)$, where $M^k$ is the set of episodes that their feedback arrives between episodes $k$ and $k + d^k$. 
Then, we sum over $k$ and apply a concentration bound over $\hat Q_h^k(s,a)$, which is smaller than $Q_h^k(s,a)$ in expectation (since $r_h^k(s,a)\leq 1$). 
Thus, we get that the right-hand-side (RHS) of \cref{eq:tabular-main additional bias} is bounded by $\eta H^3\sum_k M^k$, which in turn we bound by $\eta H^3(K + D)$ using standard delayed feedback analysis (\cref{lemma:sum-delayed-indicators}).

We can also bound the $\textsc{Drift}$ term by the RHS of \cref{eq:tabular-main additional bias}, up to a factor of $H^2$ since $|Q_h^k(s,a)| + |B_h^k(s,a)| \lesssim H^2$.
Thus, we get that $\textsc{Drift} \lesssim \eta H^5 (K + D)$ in total.
The previous state-of-the-art \cite{jin2022near} was only able to bound the $\textsc{Drift}$ term with an additional $\sqrt{SA}$ factor, in part due to their complex update rule that requires solving a global optimization problem.
This is a great demonstration of how a simple update rule is not only beneficial on the practical side, but also for enhanced provable guarantees.

Finally, $\textsc{Bias}_2 \lesssim H/\gamma$ by standard arguments for optimistic estimators. To finish the proof, sum the regret from all terms and set $\gamma = 1 / \sqrt{SAK + H^2(K+D)}$ and $\eta = \gamma / H$.
\end{proof}

\begin{algorithm}[t]
    \caption{DAPO for Linear $Q$-function}  
    \label{alg:delayed-OPPO-linear main}
    \begin{algorithmic}
    
    \STATE \textbf{Initialization:} 
        Define $\pi_{h}^{1}(a | s) = \nicefrac{1}{A}$ for every $(s,a,h)$.
        
        \FOR{$k=1,2,\dots,K$}
            
            \STATE Play episode $k$ with policy $\pi^k$.
            
            \STATE {\color{gray} \# Policy Evaluation}
            
            \FOR{$j$ such that $j+d^j = k$}
                \STATE Observe bandit feedback $\{ c^j_h(s^j_h,a^j_h) \}_{h=1}^H$.

                \STATE Compute the estimated inverse covariance matrix $\{ \hat{\Sigma}_{h}^{j,+} \}_{h=1}^H$ via Matrix Geometric Resampling, and the estimated $Q$-function weights $\{ \hat{\theta}_{h}^{j} \}_{h=1}^H$ defined in \cref{eq:linear Q main theta}.
                \STATE Define the delay-adapted estimator $\hat{Q}_h^j(s,a)$ using \cref{eq:linear Q main Q}, and estimate the bonus $\hat B_h^j(s,a)$ using \cref{alg:bonus linear} with respect to the local bonuses defined in \cref{eq:linear Q main bonus def}.
                
            \ENDFOR
            
            \STATE {\color{gray} \# Policy Update}
            
                \STATE Define the policy $\pi^{k+1}$ for every $(s,a,h)$ by:
                \begin{align*}
                     \pi^{k+1}_h(a \mid s) \propto  e^{-\eta \sum_{j: j+d^j \leq  k}  ( \hat Q_{h}^{j}(s,a) - \hat{B}_h^j(s,a) ) }
                \end{align*}
        \ENDFOR
        
    \end{algorithmic}
\end{algorithm}

\section{DAPO for Linear-$Q$}
\label{sec:linear-main}

In this section we extend DAPO to linear function approximation under the Linear-$Q$ assumption (see \cref{ass:linear-Q}), which generalizes Linear MDPs \citep{jin2020provably} and in particular is much more general than the tabular setting. 
This enables our algorithm to scale to MDPs with a huge (possibly infinite) number of states, and gives the first regret bound for delayed feedback in non-tabular MDPs.

DAPO for Linear-$Q$ (presented in \cref{alg:delayed-OPPO-linear main}) follows the same framework as in \cref{sec:tabular-main}.
That is, in each episode there is a policy evaluation step and then a policy improvement step.
Since the improvement takes the same exponential weights form, we focus on the evaluation step.
Specifically, we describe the new estimator $\hat Q^k_h(s,a)$ and bonus $\hat B^k_h(s,a)$, as these are the only changes compared to the tabular setting.
Here we only provide sketches for the algorithm and analysis, but the full details are found in \cref{app: linear}.

Just like in the tabular case, our $Q$-function estimator will simply take the original estimator from the non-delayed setting \citep{luo2021policy} and multiply it by the delay-adapted ratio.
Here, the difference between our delay adaptation approach and that of \citet{jin2022near} becomes evident.
While their approach of directly changing the importance-sampling weights simply does not apply anymore, our delay-adapted ratio is easily computed locally in the current state.  
For completeness, we now briefly describe the estimator.

Recall that, by \cref{ass:linear-Q}, the $Q$-function of policy $\pi^k$ is parameterized by $H$ vectors $\{\theta_h^k\}_{h=1}^H$, so instead of constructing an estimate $\hat{Q}_h^k(s,a)$ for each state (which is not feasible anymore), we directly estimate $\theta_h^k$.
To that end, we first construct an estimate $\hat \Sigma_h^{k,+}$ of the inverse covariance matrix $(\Sigma_h^k + \gamma I)^{-1}$, where  $\Sigma_h^k = \bbE_{s,a\sim q^k_h}[\phi_h(s,a) \phi_h(s,a)^\top]$ and $\gamma > 0$ is an exploration parameter. 
$\hat \Sigma_h^{k,+}$ is computed via the Matrix Geometric Resampling procedure \cite{neu2021online} which samples trajectories of the policy using the simulator. 
Now, the estimator of $\theta_h^k$ is defined by
\begin{align}
\label{eq:linear Q main theta}
    \hat{\theta}_{h}^{k} = \hat{\Sigma}_{h}^{k,+}\phi_{h}(s_h^k,a_h^k)L_{h}^k,
\end{align}
where $L_h^k$ is the cost-to-go from $(s_h^k, a_h^k)$.
Now, to get $\hat Q^k_h(s,a)$ we compute the delay adapted-ratio $r^k_h(s,a)$ (as in \cref{eq:delay-adapted-ratio}) and multiply it by $\phi_h(s,a)^\top \hat{\theta}_{h}^{k}$, i.e.,
\begin{align}
\label{eq:linear Q main Q}
    \hat{Q}_h^k(s,a) = r_h^k(s,a) \cdot \phi_h(s,a)^\top \hat{\theta}_{h}^{k}.
\end{align}
Intuitively, this estimator is a direct generalization of the tabular importance-sampling estimator (\cref{eq: tabular-main Q}) since $\hat \Sigma_h^{k,+}$ corresponds to $\frac{1}{q^k_h(s,a)+ \gamma}$, making $\hat Q^k_h(s,a)$ an unbiased estimate of $Q^k_h(s,a)$ up to $\gamma$ and approximation errors.

Next, we design the local bonus $b^k_h(s,a)$ to go with our delay-adapted estimator. 
It is defined as the sum of the $6$ following local bonuses (where $\{ \beta _i\}_i$ are parameters):
\begin{align}
    \nonumber
    & b_h^{k,v}(s)
    = 
    \beta_v m^{k+d^{k}} \sum_{a} r_{h}^{k}(s,a) \pi^{k+d^{k}}_h(a| s)  \norm{\phi_h(s,a)}_{\hat{\Sigma}_{h}^{k,+}}^{2}
    \\
    \nonumber
    & b_h^{k,1}(s) 
    = 
    \beta_1 \sum_{a} r_{h}^{k}(s,a) \pi_{h}^{k + d^{k}}(a | s) \Vert \phi_h(s,a) \Vert_{\hat{\Sigma}_{h}^{k,+}}
    \\
    \nonumber
    & b_h^{k,2}(s,a) 
    = 
    \beta_2 r_{h}^{k}(s,a)\Vert\phi_h(s,a)\Vert_{\hat{\Sigma}_{h}^{k,+}}
    \\
    \label{eq:linear Q main bonus def}
    &
    b_h^{k,r}(s,a) 
    = 
    \beta_r (1-r_{h}^{k}(s,a))
    \\
    \nonumber
    & b_h^{k,f}(s) 
    = 
    \beta_f \sum_{a} r_{h}^{k}(s,a) \pi_{h}^{k + d^{k}}(a | s) \Vert \phi_h(s,a) \Vert_{\hat{\Sigma}_{h}^{k,+}}^2
    \\ &
    b_h^{k,g}(s,a) 
    = 
    \beta_g r_{h}^{k}(s,a)\Vert\phi_h(s,a)\Vert_{\hat{\Sigma}_{h}^{k,+}}^2,
    \nonumber
\end{align}
where $\| x \|_{A} = \sqrt{x^\top A x}$ for $x\in \bbR^{\dim}$ and $A \in \bbR^{\dim\times\dim}$.
Each of these terms plays a different important role in the analysis. 
$b_h^{k,v}(s)$ helps us control the variance of the estimator. It is inspired by \citet{luo2021policy} and adapted to delay via the ratio $r_h^k(s,a)$. 
$b_h^{k,1}(s)$ and $b_h^{k,2}(s,a)$ help us to control the bias of the estimator. 
These, on the other hand, are constructed in a different manner than \citet{luo2021policy}.
Importantly, the corresponding bonus terms in \cite{luo2021policy} might be of order $K^{1/3}$, while with our construction the local bonus is bounded by $O(H \sqrt{ \dim })$.
This novel construction is what allows us to avoid dilated Bellman equations in the definition of the global bonus $B^k_h(s,a)$, and by that to greatly simplify both the algorithm and the analysis.
$b_h^{k,r}(s,a)$ is specifically designed to enhance exploration under delayed feedback, and in particular to control the additional bias due to the delay-adapted ratio $r_h^k(s,a)$.
Finally, we add the novel terms $b_h^{k,f}(s)$ and $b_h^{k,g}(s,a)$ to ensure regret with high probability (see events $E^f$ and $E^g$ in \cref{lemma:good-event-linear-Q} in \cref{app: linear}). 
The exact role and interpretation of each of the bonus terms is further described through the main steps in the proof sketch of \cref{thm:Q-main regret-bound}.

Given the local bonuses, we define $B^k_h(s,a)$ to be the $Q$-function with respect to $b^k_h(s,a)$. 
However, due to the possibly infinite number of states, $B^k$ is infeasible to compute without additional structure. 
Instead we compute an unbiased estimate $\hat B^k$ using the simulator (see further details in \cref{app: linear}). 
Importantly, this estimate satisfies the Bellman equations in expectation and thus inherit some of the desired properties of $Q$-functions such as the validity of the value difference lemma.
We note that while $\pi^k, \hat Q^k,\hat B^k$ are defined for all states, it is sufficient to calculate them on-the-fly only over the visited states. Finally, \cref{alg:bonus linear} (which we borrow from \citet{luo2021policy}) that computes $\hat B$ is not sample efficient. However, with some additional structure we can replace it with an efficient procedure recently presented by \citet{sherman2023improved} and obtain the same regret (see \cref{remark:OLSPE}).

\begin{theorem}
    \label{thm:Q-main regret-bound}
    Running DAPO in a Linear-$Q$ adversarial MDP with $\gamma = \sqrt{\dim / K}$, $\eta = \min \{ \frac{\gamma}{10 H \dmax} , \frac{1}{H (K+D)^{3/4}} \}$ and access to a simulator guarantees, with probability $1 - \delta$, that
    $$
        \regret
        \le
            \tilde\calO ( H^{3} \dim^{5/4} K^{3/4} 
                        + H^{2} D^{3/4} ).
    $$
\end{theorem}
This is the first sub-linear regret for non-tabular MDPs with delayed feedback. 
Moreover, like the optimal bound for tabular MDP, the delay term does not depend on the dimension $\dim$. 
Importantly, our analysis is relatively simple even compared to the non-delayed case.
By that we lay solid foundations for improved regret bounds in future work, and manage to bound the regret with high probability and not just in expectation.
One significant difference between the tabular case and Linear-$Q$ is that now the estimator $\hat Q_h^k(s,a)$ might be negative (specifically, $ - H / \gamma$). 
This is not a problem in the non-delayed setting which does not have a \textsc{Drift} term, and in fact, with proper hyper-parameter tuning we get the same $\tilde\calO (H^2 \dim^{2/3} K^{2/3})$ bound of \citet{luo2021policy} without delays. However, in the presence of delays this issue induces new challenges and requires a more involved analysis and algorithmic design (and leads to worse regret).

\begin{proof}[Proof sketch of \cref{thm:Q-main regret-bound}]
    We start by decomposing the regret as in \cref{eq:main regret decomposition}. 
    To bound $\textsc{Reg}$ we can no longer apply \cref{corollary:delayed-exp-weights-regret} like we did in the tabular case, because it heavily relies on the losses being not too negative (now they might be $O(-H/\gamma)$).
    Instead, we prove a novel bound (\cref{lemma:delayed-exp-weights-regret eta Q < 1}) that bounds $\textsc{Reg}$, for sufficiently small $\eta$, by
    $$
        \frac{H}{\eta} + \eta \sum_{k,h} \bbE_{s\sim q^*_h,a \sim \pi_h^{k + d^k}} \big[ m^{k + d^k}  (\hat Q_h^k(s,a) - \hat B_h^k(s,a))^2 \big],
    $$
    where $m^k = |\{j:j + d^j = k\}|$.
    Next, follow similar steps to \cref{thm:tabular-main regret-bound}.
    Specifically, we use $ | \hat{B}_h^k(s,a) | \lesssim H^2 \sqrt{\dim} $, apply a concentration bound on $\hat Q_h^k(s,a)^2$ around its expectation and further bound the expectation. 
    This allows us to show that:
    $\textsc{Reg} \lesssim \frac{H}{\eta} + \eta H^5 \dim K + \underbrace{ \eta  H^2 \sum_{k,h}  \bbE_{s \sim q_h^*, a \sim \pi_h^{k + d^k}} \big[ m^{k + d^k} r_h^k(s,a) 
                    \| \phi_h(s,a) \|^2_{\hat\Sigma_h^{k,+}}  \big]
                    }_{(*)}$.    
    Now we once again face the issue that the expectation is taken over states generated by $q_h^*$ and actions generated by $\pi^{k+d^k}$, while $\hat\Sigma_h^{k,+}$ is constructed from trajectories generated by $\pi^k$.
    Remarkably, our technique from the tabular case extends naturally to Linear-$Q$: Since $\hat{B}^k_h(s,a)$ satisfies the Bellman equations in expectation, we can use the value difference lemma to show that: $\textsc{Bonus} \approx \sum_{k,h} \bbE_{s,a \sim q_h^*}[b_h^k(s,a)] - \sum_{k,h} \bbE_{s,a \sim q_h^k} [b_h^k(s,a)].$ 
    
    Recall that $b_h^k(s,a)$ is the sum of the $6$ local bonuses defined in \cref{eq:linear Q main bonus def}, so we can write  $\textsc{Bonus} = \textsc{Bonus}_v + \textsc{Bonus}_1 + \textsc{Bonus}_2 + \textsc{Bonus}_r + \textsc{Bonus}_f + \textsc{Bonus}_g $, where each term corresponds to its local bonus. 
    We set $\beta_v = \eta H^2$ to get that $(*) = \sum_{k,h}  \bbE_{s \sim q_h^*} [ b_h^{k,v}(s) ]$, so finally $(*) + \textsc{Bonus}_v$ is bounded by,
    \begin{align}
    \nonumber
        &
        \eta H^2 \sum_{k,h} \bbE_{s \sim q_h^k, a \sim \pi_h^{k + d^k}} \big[ m^{k + d^k} r_h^k(s,a) 
                    \| \phi_h(s,a) \|^2_{\hat\Sigma_h^{k,+}}  \big]
                    \\
         &\quad           \leq
        \eta H^2 \sum_{k,h} \bbE_{s \sim q_h^k, a \sim \pi_h^{k}} \big[
                    m^{k + d^k} \| \phi_h(s,a) \|^2_{\hat\Sigma_h^{k,+}}  \big],
                    \label{eq:linear Q main bonus}
    \end{align}
    where the last step is due to our delay-adapted ratio, demonstrating its power compared to directly changing the importance-sampling weights \cite{jin2022near}.
    It lets us adjust the expectation to be over trajectories sampled with $\pi^k$, so it is aligned with the construction of $\hat{\Sigma}_{h}^{k,+}$ via Matrix Geometric Resampling.  
    To further bound \cref{eq:linear Q main bonus}, we may now utilize standard techniques from non-delayed analysis (e.g., \citet{jin2020provably,luo2021policy}). 
    We plug in the definition of the matrix norm, then we can consider its trace and use its linearity and invariance under cyclic permutations.
    This enables us to bound the expectation in \cref{eq:linear Q main bonus} by $\trace\big(\hat{\Sigma}_{h}^{k,+} \bbE_{s,a\sim q_{h}^{k}} \big[ \phi_{h}(s,a) \phi_{h}(s,a)^{\top} \big]\big) 
    = \trace \big( \hat{\Sigma}_{h}^{k,+} \Sigma_{h}^{k} \big)$. 
    Finally, since $\hat{\Sigma}_{h}^{k,+}$ approximates $(\Sigma_{h}^{k} + \gamma I)^{-1}$, the last term is approximately bounded by $\dim$. 
    Thus, we get that $\textsc{Bonus}_v + \textsc{Reg} \lesssim \frac{H}{\eta} + \eta H^5\dim K$ because $\sum_k m^{k+d^k} \leq K$.

    For the analysis of $\textsc{Bias}_1$, we first show it is mainly bounded by two terms: the first comes from the standard estimator while the second is the additional bias due to the delay-adapted ratio $r^k_h(s,a)$.
    That is, we bound $\textsc{Bias}_1$ by:
    \begin{align}
        \label{eq:linear Q main bias1 bonus}
       & H \sqrt{\gamma \dim} \sum_{k,h} 
       \underset{s\sim q^*_h, a\sim \pi^{k + d^k}}{\bbE}
       \big[ r_h^k(s,a) \|  \phi_h(s,a) \|_{\hat\Sigma_h^{k,+}}  \big] 
       \\
       & +
       H \sqrt{ \dim} \sum_{k,h} \bbE_{s\sim q^*_h, a \sim \pi_h^{k + d^k}} \big[ ( 1-r_h^k(s,a) )  \big].
       \label{eq:linear Q main bias1 drift}
    \end{align}
    Once again, with the proper tuning $\beta_1 = H \sqrt{\gamma \dim}$, we can get \eqref{eq:linear Q main bias1 bonus} $= \sum_{k,h} \bbE_{s \sim q_h^*} [ b_h^{k,1}(s) ]$, and then combine with the corresponding $\textsc{Bonus}$ term $\textsc{Bonus}_1$, while utilizing the delay-adapted ratio. 
    This gives us: $\eqref{eq:linear Q main bias1 bonus} + \textsc{Bonus}_1 \lesssim H \sqrt{\gamma \dim} \sum_{k,h} \bbE_{s, a\sim q^k_h} \big[ \|  \phi_h(s,a) \|_{\hat\Sigma_h^{k,+}}  \big] \leq \sqrt{\gamma}H^2 \dim K $.
    
    For \cref{eq:linear Q main bias1 drift}, we first bound $\bbE_{a \sim \pi_h^{k + d^k}} \big[ ( 1-r_h^k(s,a) )  \big] \leq \| \pi_h^{k + d^k}(\cdot| s) - \pi_h^{k}(\cdot| s) \|_1$ and then follow similar arguments to the tabular case regarding the multiplicative weights update form. 
    The main difference is that now $\hat Q_h^k(s,a)$ can be negative, resulting in weaker guarantees (see more details in \cref{sec:drift-bound-linear-Q}). 
    Overall, we get: $\eqref{eq:linear Q main bias1 drift} \lesssim  \frac{\eta}{\gamma} H^3 \sqrt{\dim} (K + D)$.
    
    The analysis of $\textsc{Bias}_2$ is very different from both the tabular case and the non-delayed Linear-$Q$ \cite{luo2021policy}. 
    This is mainly for two reasons:
    First, in the tabular case, the added bias $\gamma$ makes the estimator $\hat{Q}^k_h(s,a)$ optimistic, but this is no longer the case in Linear-$Q$ since now the estimator might also be negative. 
    Second, $\textsc{Bias}_2$ contains the inner product with $\pi^*_h(\cdot|s)$ which cannot be aligned with the denominator of the delay-adapted ratio $r^k_h(s,a)$. Thus, we need a novel more involved analysis for $\textsc{Bias}_2$. 
    \begin{figure*}[ht]
        \begin{center}
        \centerline{\includegraphics[width=0.83\textwidth]{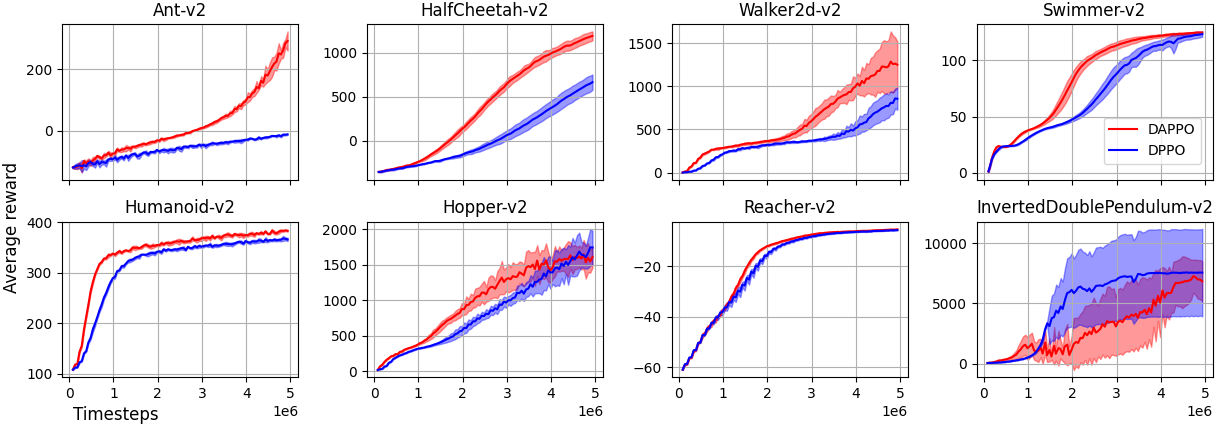}}
        \vskip -0.1in
        \caption{\textbf{Training curves: DAPPO vs DPPO.} Plots show average reward and std over 5 seeds. x-axis is number of timesteps up to 5M.}
        \label{fig:main}
        \end{center}
        \vskip -0.25in
    \end{figure*}
    \begin{figure*}[ht]
    \begin{center}
    \centerline{\includegraphics[width=0.85\textwidth]{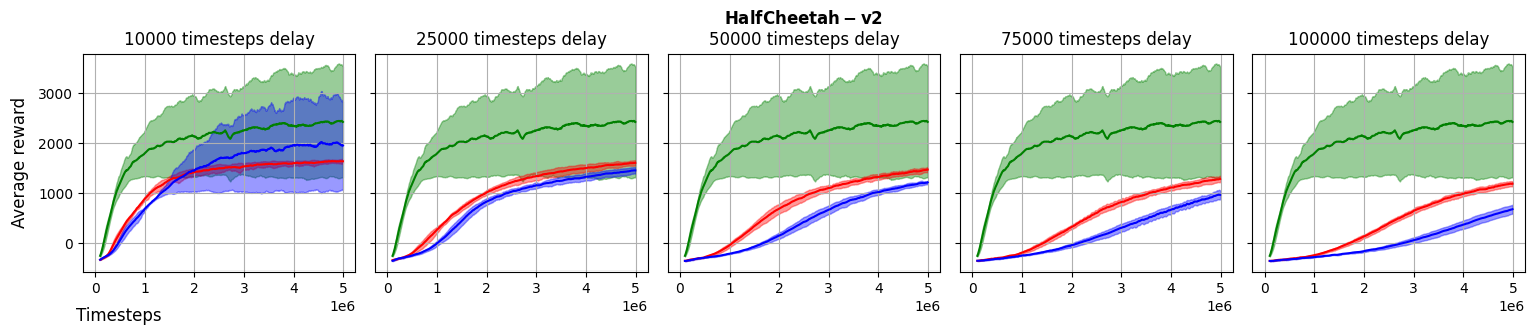}}
    \vskip -0.1in
        \caption{\textbf{Training curves with different fixed delay length:}  DAPPO vs DPPO with different delay, alongside PPO without delays. Plots show average reward and std over 5 seeds. x-axis is number of timesteps up to 5M.}
    \label{fig:different delays main}
    \end{center}
\end{figure*}
    
    We start in a similar way to $\textsc{Bias}_1$, and bound $\textsc{Bias}_2$ by:
   \begin{align}
       \label{eq:linear Q main bias2 bonus}
       & H \sqrt{\gamma \dim} \sum_{k,h} \bbE_{s, a\sim q^*_h} \big[ r_h^k(s,a) \|  \phi_h(s,a) \|_{\hat\Sigma_h^{k,+}}  \big] 
       \\
       & +
       H \sqrt{ \dim} \sum_{k,h} \bbE_{s, a\sim q^*_h} \big[ (1 - r_h^k(s,a) )  \big].
       \label{eq:linear Q main bias2 drift}
   \end{align}
   We handle \cref{eq:linear Q main bias2 bonus} like \cref{eq:linear Q main bias1 bonus}, so for $\beta_2 = H \sqrt{\gamma \dim}$ we get that $\eqref{eq:linear Q main bias2 bonus} + \textsc{Bonus}_2 \lesssim \sqrt{\gamma}H^2 \dim K $.
   Term \eqref{eq:linear Q main bias2 drift} on the other hand might be of order of $K$ due to mismatch between $\pi^*_h(a| s)$ in the expectation and $\max\{ \pi_h^{k + d^k}(a| s), \pi_h^{k}(a| s) \}$ in the denominator  of $r_h^k(s,a)$. 
   To address that, we design the novel bonus term $b_h^{k,r}(s,a)$. 
   Summing \cref{eq:linear Q main bias2 drift} with $\textsc{Bonus}_r$ essentially allows us to substitute $\pi^*_h(a| s)$ by $\pi^k_h(a| s)$, which gives: $\eqref{eq:linear Q main bias2 bonus} + \textsc{Bonus}_r \lesssim \sum_{k,h} \bbE_{s\sim q_h^*} \| \pi_h^{k + d^k}(\cdot| s) - \pi_h^{k}(\cdot| s) \|_1 \lesssim  \frac{\eta}{\gamma} H^3 \sqrt{\dim} (K + D)$.
    
   Finally, $\textsc{Drift}$ is also bounded by the $\ell_1$-distance above and further by $\frac{\eta}{\gamma} H^4 (K + D)$.
   Summing the regret from all terms and optimizing over $\eta$ and $\gamma$ completes the proof.
\end{proof}

\section{Delay-Adapted PPO and Experiments}
\label{sec:PPO}


In this section we show how our generic delay adaptation method extends to state-of-the-art deep RL methods, and demonstrate its great potential through simple experiments on popular MuJoCo environments \citep{todorov2012mujoco}.
We note that here we follow the standard convention that use rewards in deep RL rather than costs as in the rest of the paper.
Due to lack of space, here we only provide an overview of the method and main experiment.
For additional experiments and full implementation details see \cref{appendix: experiments}.

The highly successful TRPO algorithm \cite{schulman2015trust} is a deep RL policy-gradient method that builds on the same PO principles discussed in this paper.
Specifically, it follows the Policy Iteration paradigm with a ``soft" policy improvement step, where the policy is now approximated by a Deep Neural Network with parameters $\theta$.
While our update step (\cref{eq: tabular-main ewu}) is equivalent to maximizing an objective with a KL-regularization term \citep{ShaniEM20}, TRPO replaces it by a constraint which results in maximizing the objective $L^k_{TRPO}(\theta) =  \sum_{h=1}^H  \frac{\pi^{\theta}(a_h| s_h)}{\pi^{\theta^k}(a_h| s_h)} \hat{A}_h$ subject to a constraint that keeps the new and the old policies close in terms of KL-divergence.
Here, $\hat{A}_h$ is an estimate of the advantage function which replaces the $Q$-function in our formulation to further reduce variance \citep{sutton1999policy}.

While successful, TRPO's constrained optimization is computationally expensive.
Thus, PPO \citep{schulman2017proximal} removes the explicit constraint and replaces it with a sophisticated clipping technique that allows to keep strong empirical performance while only optimizing the following (non-constrained) objective:
\begin{align}
     L^k(\theta) = 
     \sum_{h=1}^H \min \l \{ g^k_h(\theta) \hat{A}_h 
     , \text{clip}_{1 \pm \epsilon}\l( g^k_h(\theta) \r) \hat{A}_h 
     \r\},
     \label{eq:PPO objective}
\end{align}
where $ g^k_h(\theta) = \frac{\pi^{\theta}(a_h| s_h)}{\pi^{\theta^k}(a_h| s_h)}$ and $\text{clip}_{1 \pm \epsilon}(x)$ clips $x$ between $1 - \epsilon$ and $1 + \epsilon$. 
This objective essentially zeros the gradient whenever the policy changes too much, and thus replaces the need for an explicit constraint (see more details in \cite{schulman2017proximal}).
Next, we adapt PPO to delayed feedback.

With delayed feedback, the trajectory that arrives at time $k$ was generated using policy $\pi^{\theta^{k - d}}$.
Thus, a naive adaptation would be to optimize \cref{eq:PPO objective} but replacing $\pi^{\theta^{k}}$ with $\pi^{\theta^{k - d}}$. 
We will refer to this algorithm as Delayed PPO (DPPO).
As this paper shows, DPPO is likely to suffer from large variance which can be reduced when multiplying the objective by the delay-adapted ratio $r^k_h(s,a)$.
The result is our novel Delay-Adapted PPO (DAPPO) which optimizes:
\begin{align*}
     L^k_{DA}(\theta) = 
     \sum_{h=1}^H \min \l \{ R^k_h(\theta) \hat{A}_h 
     , \text{clip}_{1 \pm \epsilon}\l( R^k_h(\theta) \r) \hat{A}_h 
     \r\},
     \label{eq:DAPPO objective}
\end{align*}
for $ R^k_h(\theta) = \frac{\pi^{\theta}(a_h| s_h)}{\max\{ \pi^{\theta^{k-d}}(a_h| s_h) ,\pi^{\theta^k}(a_h| s_h)\}}$.
Another alternative, which we call Non-Delayed PPO (NDPPO), is to run the original PPO and ignore the fact that feedback is delayed.
This results in a highly unstable algorithm which is likely to suffer from large bias due to the mismatch between the policy $\pi^{\theta^{k-d}}$ that generated the trajectory and the policy  $\pi^{\theta^{k}}$ which is used to re-weight the estimator.

\cref{fig:main} compares the performance of DPPO and DAPPO over $8$ MuJoCo environments. 
We use a fixed delay of $10^5$ timesteps while the total number of timestamps is $5\cdot 10^6$. 
Results are averaged over 5 runs and the shaded areas around the curves indicate standard deviation.
Note that the only difference between the algorithms is the objective, we did not tune hyper-parameters or modify network architecture.
DAPPO outperforms DPPO in at least $4$ environments and is on par with DPPO in the rest. 
The only exception is \texttt{InvertedDoublePendulum}, but it is important to note that this environment is extremely noisy.

\cref{fig:different delays main} compares the training curves of  DPPO vs DAPPO in the \textsc{Swimmer} environment (for more environments see \cref{appendix: additional experiments 2}) with different delays in $\{10000,25000,50000,75000,100000\}$, alongside the training curve of PPO without delay.
As expected, when the delay is relatively small (e.g., $10000$), there is no significant difference between learning with or without delayed feedback.
As the delay becomes larger, the performance of all algorithms drops (but at different rates).

These empirical results support our claim that handling delays via the delay-adapted ratio extends naturally beyond the tabular and Linear-$Q$ settings to practical deep function approximation.
Surprisingly, even in this simple case, delays cause significant drop in performance which demonstrates the great importance of delay-adapted algorithms. 
Our novel method makes a significant step towards practical deep RL algorithms that are robust to delayed feedback.
Finally, we note that NDPPO is omitted from the graphs because it does not converge (as expected).
Instead, it oscillates between high and low reward (see \cref{appendix: experiments} for more details).

\section{Future Work}
We leave a few open questions for future work. 
In the tabular case, it still remains unclear what is the optimal dependency under delayed feedback in terms of the horizon $H$. Another future direction is to further improve our results in the Linear Function Approximation setting and extend them to the case where a simulator is unavailable. This is in particular important in light of very recent advancement in the non-delayed setting \cite{sherman2023improved,dai2023refined} which significantly improve \citet{luo2021policy}. Finally, our experiment demonstrate the potential that delay-adaptation methods have for deep RL applications. However, a much more thorough empirical study needs to be done in order to fully understand the implications of these methods on deep RL.

\bibliography{delay_bib}

\begin{thebibliography}{100}
\providecommand{\natexlab}[1]{#1}
\providecommand{\url}[1]{\texttt{#1}}
\expandafter\ifx\csname urlstyle\endcsname\relax
  \providecommand{\doi}[1]{doi: #1}\else
  \providecommand{\doi}{doi: \begingroup \urlstyle{rm}\Url}\fi

\bibitem[Abbasi-Yadkori et~al.(2019)Abbasi-Yadkori, Bartlett, Bhatia, Lazic,
  Szepesvari, and Weisz]{abbasi2019politex}
Abbasi-Yadkori, Y., Bartlett, P., Bhatia, K., Lazic, N., Szepesvari, C., and
  Weisz, G.
\newblock Politex: Regret bounds for policy iteration using expert prediction.
\newblock In \emph{International Conference on Machine Learning}, pp.\
  3692--3702. PMLR, 2019.

\bibitem[Agarwal \& Duchi(2012)Agarwal and Duchi]{agarwal2012distributed}
Agarwal, A. and Duchi, J.~C.
\newblock Distributed delayed stochastic optimization.
\newblock In \emph{2012 IEEE 51st IEEE Conference on Decision and Control
  (CDC)}, pp.\  5451--5452. IEEE, 2012.

\bibitem[Agarwal et~al.(2020)Agarwal, Kakade, Lee, and
  Mahajan]{agarwal2020optimality}
Agarwal, A., Kakade, S.~M., Lee, J.~D., and Mahajan, G.
\newblock Optimality and approximation with policy gradient methods in markov
  decision processes.
\newblock In \emph{Conference on Learning Theory}, pp.\  64--66. PMLR, 2020.

\bibitem[Ayoub et~al.(2020)Ayoub, Jia, Szepesvari, Wang, and
  Yang]{ayoub2020model}
Ayoub, A., Jia, Z., Szepesvari, C., Wang, M., and Yang, L.
\newblock Model-based reinforcement learning with value-targeted regression.
\newblock In \emph{International Conference on Machine Learning}, pp.\
  463--474. PMLR, 2020.

\bibitem[Azar et~al.(2017)Azar, Osband, and Munos]{azar2017minimax}
Azar, M.~G., Osband, I., and Munos, R.
\newblock Minimax regret bounds for reinforcement learning.
\newblock In \emph{Proceedings of the 34th International Conference on Machine
  Learning-Volume 70}, pp.\  263--272. JMLR. org, 2017.

\bibitem[Beygelzimer et~al.(2011)Beygelzimer, Langford, Li, Reyzin, and
  Schapire]{beygelzimer2011contextual}
Beygelzimer, A., Langford, J., Li, L., Reyzin, L., and Schapire, R.
\newblock Contextual bandit algorithms with supervised learning guarantees.
\newblock In \emph{Proceedings of the Fourteenth International Conference on
  Artificial Intelligence and Statistics}, pp.\  19--26. JMLR Workshop and
  Conference Proceedings, 2011.

\bibitem[Bistritz et~al.(2019)Bistritz, Zhou, Chen, Bambos, and
  Blanchet]{bistritz2019online}
Bistritz, I., Zhou, Z., Chen, X., Bambos, N., and Blanchet, J.
\newblock Online exp3 learning in adversarial bandits with delayed feedback.
\newblock In \emph{Advances in Neural Information Processing Systems}, pp.\
  11349--11358, 2019.

\bibitem[Bistritz et~al.(2021)Bistritz, Zhou, Chen, Bambos, and
  Blanchet]{bistritz2021no}
Bistritz, I., Zhou, Z., Chen, X., Bambos, N., and Blanchet, J.
\newblock No discounted-regret learning in adversarial bandits with delays.
\newblock \emph{arXiv preprint arXiv:2103.04550}, 2021.

\bibitem[Cai et~al.(2020)Cai, Yang, Jin, and Wang]{cai2020provably}
Cai, Q., Yang, Z., Jin, C., and Wang, Z.
\newblock Provably efficient exploration in policy optimization.
\newblock In \emph{International Conference on Machine Learning}, pp.\
  1283--1294. PMLR, 2020.

\bibitem[Cesa-Bianchi et~al.(2016)Cesa-Bianchi, Gentile, Mansour, and
  Minora]{cesa2016delay}
Cesa-Bianchi, N., Gentile, C., Mansour, Y., and Minora, A.
\newblock Delay and cooperation in nonstochastic bandits.
\newblock In \emph{Conference on Learning Theory}, pp.\  605--622, 2016.

\bibitem[Cesa-Bianchi et~al.(2018)Cesa-Bianchi, Gentile, and
  Mansour]{cesa2018nonstochastic}
Cesa-Bianchi, N., Gentile, C., and Mansour, Y.
\newblock Nonstochastic bandits with composite anonymous feedback.
\newblock In \emph{Conference On Learning Theory}, pp.\  750--773, 2018.

\bibitem[Cesa-Bianchi et~al.(2019)Cesa-Bianchi, Gentile, and
  Mansour]{cesa2019delay}
Cesa-Bianchi, N., Gentile, C., and Mansour, Y.
\newblock Delay and cooperation in nonstochastic bandits.
\newblock \emph{The Journal of Machine Learning Research}, 20\penalty0
  (1):\penalty0 613--650, 2019.

\bibitem[Changuel et~al.(2012)Changuel, Sayadi, and
  Kieffer]{changuel2012online}
Changuel, N., Sayadi, B., and Kieffer, M.
\newblock Online learning for qoe-based video streaming to mobile receivers.
\newblock In \emph{2012 IEEE Globecom Workshops}, pp.\  1319--1324. IEEE, 2012.

\bibitem[Chen et~al.(2020{\natexlab{a}})Chen, Xu, Liu, Li, and
  Zhao]{chen2020delay}
Chen, B., Xu, M., Liu, Z., Li, L., and Zhao, D.
\newblock Delay-aware multi-agent reinforcement learning.
\newblock \emph{arXiv preprint arXiv:2005.05441}, 2020{\natexlab{a}}.

\bibitem[Chen \& Luo(2021)Chen and Luo]{chen2021finding}
Chen, L. and Luo, H.
\newblock Finding the stochastic shortest path with low regret: The adversarial
  cost and unknown transition case.
\newblock \emph{arXiv preprint arXiv:2102.05284}, 2021.

\bibitem[Chen et~al.(2020{\natexlab{b}})Chen, Luo, and Wei]{chen2020minimax}
Chen, L., Luo, H., and Wei, C.-Y.
\newblock Minimax regret for stochastic shortest path with adversarial costs
  and known transition.
\newblock \emph{arXiv preprint arXiv:2012.04053}, 2020{\natexlab{b}}.

\bibitem[Chen et~al.(2021)Chen, Jafarnia-Jahromi, Jain, and
  Luo]{chen2021implicit}
Chen, L., Jafarnia-Jahromi, M., Jain, R., and Luo, H.
\newblock Implicit finite-horizon approximation and efficient optimal
  algorithms for stochastic shortest path.
\newblock \emph{Advances in Neural Information Processing Systems}, 2021.

\bibitem[Chen et~al.(2022{\natexlab{a}})Chen, Jain, and Luo]{chen2022improved}
Chen, L., Jain, R., and Luo, H.
\newblock Improved no-regret algorithms for stochastic shortest path with
  linear mdp.
\newblock In \emph{International Conference on Machine Learning}, pp.\
  3204--3245. PMLR, 2022{\natexlab{a}}.

\bibitem[Chen et~al.(2022{\natexlab{b}})Chen, Luo, and
  Rosenberg]{chen2022policy}
Chen, L., Luo, H., and Rosenberg, A.
\newblock Policy optimization for stochastic shortest path.
\newblock In Loh, P. and Raginsky, M. (eds.), \emph{Conference on Learning
  Theory, 2-5 July 2022, London, {UK}}, volume 178 of \emph{Proceedings of
  Machine Learning Research}, pp.\  982--1046. {PMLR}, 2022{\natexlab{b}}.

\bibitem[Cohen et~al.(2021{\natexlab{a}})Cohen, Daniely, Drori, Koren, and
  Schain]{cohen2021asynchronous}
Cohen, A., Daniely, A., Drori, Y., Koren, T., and Schain, M.
\newblock Asynchronous stochastic optimization robust to arbitrary delays.
\newblock In Ranzato, M., Beygelzimer, A., Dauphin, Y.~N., Liang, P., and
  Vaughan, J.~W. (eds.), \emph{Advances in Neural Information Processing
  Systems 34: Annual Conference on Neural Information Processing Systems 2021,
  NeurIPS 2021, December 6-14, 2021, virtual}, pp.\  9024--9035,
  2021{\natexlab{a}}.

\bibitem[Cohen et~al.(2021{\natexlab{b}})Cohen, Efroni, Mansour, and
  Rosenberg]{cohen2021minimax}
Cohen, A., Efroni, Y., Mansour, Y., and Rosenberg, A.
\newblock Minimax regret for stochastic shortest path.
\newblock \emph{Advances in Neural Information Processing Systems}, 34,
  2021{\natexlab{b}}.

\bibitem[Dai et~al.(2022)Dai, Luo, and Chen]{dai2022follow}
Dai, Y., Luo, H., and Chen, L.
\newblock Follow-the-perturbed-leader for adversarial markov decision processes
  with bandit feedback.
\newblock \emph{arXiv preprint arXiv:2205.13451}, 2022.

\bibitem[Dai et~al.(2023)Dai, Luo, Wei, and Zimmert]{dai2023refined}
Dai, Y., Luo, H., Wei, C.-Y., and Zimmert, J.
\newblock Refined regret for adversarial mdps with linear function
  approximation.
\newblock \emph{arXiv preprint arXiv:2301.12942}, 2023.

\bibitem[Dann et~al.(2017)Dann, Lattimore, and Brunskill]{dann2017unifying}
Dann, C., Lattimore, T., and Brunskill, E.
\newblock Unifying pac and regret: Uniform pac bounds for episodic
  reinforcement learning.
\newblock In \emph{Advances in Neural Information Processing Systems}, pp.\
  5713--5723, 2017.

\bibitem[Derman et~al.(2021)Derman, Dalal, and Mannor]{derman2021acting}
Derman, E., Dalal, G., and Mannor, S.
\newblock Acting in delayed environments with non-stationary markov policies.
\newblock In \emph{9th International Conference on Learning Representations,
  {ICLR} 2021, Virtual Event, Austria, May 3-7, 2021}, 2021.

\bibitem[Dudik et~al.(2011)Dudik, Hsu, Kale, Karampatziakis, Langford, Reyzin,
  and Zhang]{dudik2011efficient}
Dudik, M., Hsu, D., Kale, S., Karampatziakis, N., Langford, J., Reyzin, L., and
  Zhang, T.
\newblock Efficient optimal learning for contextual bandits.
\newblock In \emph{Proceedings of the Twenty-Seventh Conference on Uncertainty
  in Artificial Intelligence}, pp.\  169--178, 2011.

\bibitem[Efroni et~al.(2019)Efroni, Merlis, Ghavamzadeh, and
  Mannor]{efroni2019tight}
Efroni, Y., Merlis, N., Ghavamzadeh, M., and Mannor, S.
\newblock Tight regret bounds for model-based reinforcement learning with
  greedy policies.
\newblock In Wallach, H.~M., Larochelle, H., Beygelzimer, A.,
  d'Alch{\'{e}}{-}Buc, F., Fox, E.~B., and Garnett, R. (eds.), \emph{Advances
  in Neural Information Processing Systems 32: Annual Conference on Neural
  Information Processing Systems 2019, NeurIPS 2019, 8-14 December 2019,
  Vancouver, BC, Canada}, pp.\  12203--12213, 2019.

\bibitem[Even-Dar et~al.(2009)Even-Dar, Kakade, and Mansour]{even2009online}
Even-Dar, E., Kakade, S.~M., and Mansour, Y.
\newblock Online markov decision processes.
\newblock \emph{Mathematics of Operations Research}, 34\penalty0 (3):\penalty0
  726--736, 2009.

\bibitem[Freund \& Schapire(1997)Freund and Schapire]{freund1997decision}
Freund, Y. and Schapire, R.~E.
\newblock A decision-theoretic generalization of on-line learning and an
  application to boosting.
\newblock \emph{Journal of computer and system sciences}, 55\penalty0
  (1):\penalty0 119--139, 1997.

\bibitem[Gael et~al.(2020)Gael, Vernade, Carpentier, and
  Valko]{manegueu2020stochastic}
Gael, M.~A., Vernade, C., Carpentier, A., and Valko, M.
\newblock Stochastic bandits with arm-dependent delays.
\newblock In \emph{International Conference on Machine Learning}, pp.\
  3348--3356. PMLR, 2020.

\bibitem[Gu et~al.(2017)Gu, Holly, Lillicrap, and Levine]{gu2017deep}
Gu, S., Holly, E., Lillicrap, T., and Levine, S.
\newblock Deep reinforcement learning for robotic manipulation with
  asynchronous off-policy updates.
\newblock In \emph{2017 IEEE international conference on robotics and
  automation (ICRA)}, pp.\  3389--3396. IEEE, 2017.

\bibitem[Gyorgy \& Joulani(2021)Gyorgy and Joulani]{gyorgy2020adapting}
Gyorgy, A. and Joulani, P.
\newblock Adapting to delays and data in adversarial multi-armed bandits.
\newblock In \emph{International Conference on Machine Learning}, pp.\
  3988--3997. PMLR, 2021.

\bibitem[Haarnoja et~al.(2018)Haarnoja, Zhou, Abbeel, and
  Levine]{haarnoja2018soft}
Haarnoja, T., Zhou, A., Abbeel, P., and Levine, S.
\newblock Soft actor-critic: Off-policy maximum entropy deep reinforcement
  learning with a stochastic actor.
\newblock In \emph{International Conference on Machine Learning}, pp.\
  1861--1870, 2018.

\bibitem[He et~al.(2022)He, Zhou, and Gu]{pmlr-v151-he22a}
He, J., Zhou, D., and Gu, Q.
\newblock Near-optimal policy optimization algorithms for learning adversarial
  linear mixture mdps.
\newblock In Camps-Valls, G., Ruiz, F. J.~R., and Valera, I. (eds.),
  \emph{Proceedings of The 25th International Conference on Artificial
  Intelligence and Statistics}, volume 151 of \emph{Proceedings of Machine
  Learning Research}, pp.\  4259--4280. PMLR, 28--30 Mar 2022.

\bibitem[Howson et~al.(2021)Howson, Pike-Burke, and Filippi]{howson2021delayed}
Howson, B., Pike-Burke, C., and Filippi, S.
\newblock Delayed feedback in episodic reinforcement learning.
\newblock \emph{arXiv preprint arXiv:2111.07615}, 2021.

\bibitem[Howson et~al.(2022)Howson, Pike-Burke, and Filippi]{howson2022delayed}
Howson, B., Pike-Burke, C., and Filippi, S.
\newblock Delayed feedback in generalised linear bandits revisited.
\newblock \emph{arXiv preprint arXiv:2207.10786}, 2022.

\bibitem[Ito et~al.(2020)Ito, Hatano, Sumita, Takemura, Fukunaga, Kakimura, and
  Kawarabayashi]{ito2020delay}
Ito, S., Hatano, D., Sumita, H., Takemura, K., Fukunaga, T., Kakimura, N., and
  Kawarabayashi, K.-I.
\newblock Delay and cooperation in nonstochastic linear bandits.
\newblock \emph{Advances in Neural Information Processing Systems},
  33:\penalty0 4872--4883, 2020.

\bibitem[Jaksch et~al.(2010)Jaksch, Ortner, and Auer]{jaksch2010near}
Jaksch, T., Ortner, R., and Auer, P.
\newblock Near-optimal regret bounds for reinforcement learning.
\newblock \emph{Journal of Machine Learning Research}, 11\penalty0 (4), 2010.

\bibitem[Jin et~al.(2018)Jin, Allen-Zhu, Bubeck, and Jordan]{jin2018q}
Jin, C., Allen-Zhu, Z., Bubeck, S., and Jordan, M.~I.
\newblock Is q-learning provably efficient?
\newblock In \emph{Advances in Neural Information Processing Systems}, pp.\
  4863--4873, 2018.

\bibitem[Jin et~al.(2020{\natexlab{a}})Jin, Jin, Luo, Sra, and
  Yu]{jin2019learning}
Jin, C., Jin, T., Luo, H., Sra, S., and Yu, T.
\newblock Learning adversarial markov decision processes with bandit feedback
  and unknown transition.
\newblock In \emph{International Conference on Machine Learning}, pp.\
  4860--4869. PMLR, 2020{\natexlab{a}}.

\bibitem[Jin et~al.(2020{\natexlab{b}})Jin, Yang, Wang, and
  Jordan]{jin2020provably}
Jin, C., Yang, Z., Wang, Z., and Jordan, M.~I.
\newblock Provably efficient reinforcement learning with linear function
  approximation.
\newblock In \emph{Conference on Learning Theory}, pp.\  2137--2143,
  2020{\natexlab{b}}.

\bibitem[Jin \& Luo(2020)Jin and Luo]{jin2020simultaneously}
Jin, T. and Luo, H.
\newblock Simultaneously learning stochastic and adversarial episodic mdps with
  known transition.
\newblock \emph{Advances in neural information processing systems}, 2020.

\bibitem[Jin et~al.(2021)Jin, Huang, and Luo]{jin2021best}
Jin, T., Huang, L., and Luo, H.
\newblock The best of both worlds: stochastic and adversarial episodic mdps
  with unknown transition.
\newblock \emph{Advances in Neural Information Processing Systems}, 2021.

\bibitem[Jin et~al.(2022)Jin, Lancewicki, Luo, Mansour, and
  Rosenberg]{jin2022near}
Jin, T., Lancewicki, T., Luo, H., Mansour, Y., and Rosenberg, A.
\newblock Near-optimal regret for adversarial mdp with delayed bandit feedback.
\newblock \emph{arXiv preprint arXiv:2201.13172}, 2022.

\bibitem[Kakade \& Langford(2002)Kakade and Langford]{kakade2002approximately}
Kakade, S. and Langford, J.
\newblock Approximately optimal approximate reinforcement learning.
\newblock In \emph{In Proc. 19th International Conference on Machine Learning}.
  Citeseer, 2002.

\bibitem[Kakade(2001)]{kakade2001natural}
Kakade, S.~M.
\newblock A natural policy gradient.
\newblock \emph{Advances in neural information processing systems},
  14:\penalty0 1531--1538, 2001.

\bibitem[Katsikopoulos \& Engelbrecht(2003)Katsikopoulos and
  Engelbrecht]{katsikopoulos2003markov}
Katsikopoulos, K.~V. and Engelbrecht, S.~E.
\newblock Markov decision processes with delays and asynchronous cost
  collection.
\newblock \emph{IEEE transactions on automatic control}, 48\penalty0
  (4):\penalty0 568--574, 2003.

\bibitem[Kingma \& Ba(2014)Kingma and Ba]{kingma2014adam}
Kingma, D.~P. and Ba, J.
\newblock Adam: A method for stochastic optimization.
\newblock \emph{arXiv preprint arXiv:1412.6980}, 2014.

\bibitem[Lancewicki et~al.(2021)Lancewicki, Segal, Koren, and
  Mansour]{lancewicki2021stochastic}
Lancewicki, T., Segal, S., Koren, T., and Mansour, Y.
\newblock Stochastic multi-armed bandits with unrestricted delay distributions.
\newblock In \emph{Proceedings of the 38th International Conference on Machine
  Learning, {ICML} 2021, 18-24 July 2021, Virtual Event}, pp.\  5969--5978.
  {PMLR}, 2021.

\bibitem[Lancewicki et~al.(2022{\natexlab{a}})Lancewicki, Rosenberg, and
  Mansour]{Lancewicki0M22}
Lancewicki, T., Rosenberg, A., and Mansour, Y.
\newblock Cooperative online learning in stochastic and adversarial mdps.
\newblock In Chaudhuri, K., Jegelka, S., Song, L., Szepesv{\'{a}}ri, C., Niu,
  G., and Sabato, S. (eds.), \emph{International Conference on Machine
  Learning, {ICML} 2022, 17-23 July 2022, Baltimore, Maryland, {USA}}, volume
  162 of \emph{Proceedings of Machine Learning Research}, pp.\  11918--11968.
  {PMLR}, 2022{\natexlab{a}}.

\bibitem[Lancewicki et~al.(2022{\natexlab{b}})Lancewicki, Rosenberg, and
  Mansour]{lancewicki2020learning}
Lancewicki, T., Rosenberg, A., and Mansour, Y.
\newblock Learning adversarial markov decision processes with delayed feedback.
\newblock In \emph{Proceedings of the AAAI Conference on Artificial
  Intelligence}, volume~36, pp.\  7281--7289, 2022{\natexlab{b}}.

\bibitem[Levine \& Koltun(2013)Levine and Koltun]{levine2013guided}
Levine, S. and Koltun, V.
\newblock Guided policy search.
\newblock In \emph{International conference on machine learning}, pp.\  1--9.
  PMLR, 2013.

\bibitem[Levine et~al.(2016)Levine, Finn, Darrell, and Abbeel]{levine2016end}
Levine, S., Finn, C., Darrell, T., and Abbeel, P.
\newblock End-to-end training of deep visuomotor policies.
\newblock \emph{The Journal of Machine Learning Research}, 17\penalty0
  (1):\penalty0 1334--1373, 2016.

\bibitem[Lillicrap et~al.(2015)Lillicrap, Hunt, Pritzel, Heess, Erez, Tassa,
  Silver, and Wierstra]{lillicrap2015continuous}
Lillicrap, T.~P., Hunt, J.~J., Pritzel, A., Heess, N., Erez, T., Tassa, Y.,
  Silver, D., and Wierstra, D.
\newblock Continuous control with deep reinforcement learning.
\newblock \emph{arXiv preprint arXiv:1509.02971}, 2015.

\bibitem[Liu et~al.(2014)Liu, Wang, and Liu]{liu2014impact}
Liu, S., Wang, X., and Liu, P.~X.
\newblock Impact of communication delays on secondary frequency control in an
  islanded microgrid.
\newblock \emph{IEEE Transactions on Industrial Electronics}, 62\penalty0
  (4):\penalty0 2021--2031, 2014.

\bibitem[Luo et~al.(2021)Luo, Wei, and Lee]{luo2021policy}
Luo, H., Wei, C.-Y., and Lee, C.-W.
\newblock Policy optimization in adversarial mdps: Improved exploration via
  dilated bonuses.
\newblock \emph{Advances in Neural Information Processing Systems}, 34, 2021.

\bibitem[Mahmood et~al.(2018)Mahmood, Korenkevych, Komer, and
  Bergstra]{mahmood2018setting}
Mahmood, A.~R., Korenkevych, D., Komer, B.~J., and Bergstra, J.
\newblock Setting up a reinforcement learning task with a real-world robot.
\newblock In \emph{2018 IEEE/RSJ International Conference on Intelligent Robots
  and Systems (IROS)}, pp.\  4635--4640. IEEE, 2018.

\bibitem[Masoudian et~al.(2022)Masoudian, Zimmert, and
  Seldin]{masoudian2022best}
Masoudian, S., Zimmert, J., and Seldin, Y.
\newblock A best-of-both-worlds algorithm for bandits with delayed feedback.
\newblock \emph{arXiv preprint arXiv:2206.14906}, 2022.

\bibitem[Min et~al.(2022)Min, He, Wang, and Gu]{min2022learning}
Min, Y., He, J., Wang, T., and Gu, Q.
\newblock Learning stochastic shortest path with linear function approximation.
\newblock In \emph{International Conference on Machine Learning}, pp.\
  15584--15629. PMLR, 2022.

\bibitem[Neu(2015)]{neu2015explore}
Neu, G.
\newblock Explore no more: Improved high-probability regret bounds for
  non-stochastic bandits.
\newblock \emph{Advances in Neural Information Processing Systems},
  28:\penalty0 3168--3176, 2015.

\bibitem[Neu \& Olkhovskaya(2021)Neu and Olkhovskaya]{neu2021online}
Neu, G. and Olkhovskaya, J.
\newblock Online learning in mdps with linear function approximation and bandit
  feedback.
\newblock \emph{Advances in Neural Information Processing Systems},
  34:\penalty0 10407--10417, 2021.

\bibitem[Neu et~al.(2010{\natexlab{a}})Neu, Gy{\"{o}}rgy, and
  Szepesv{\'{a}}ri]{neu2010loopfree}
Neu, G., Gy{\"{o}}rgy, A., and Szepesv{\'{a}}ri, C.
\newblock The online loop-free stochastic shortest-path problem.
\newblock In \emph{{COLT} 2010 - The 23rd Conference on Learning Theory, Haifa,
  Israel, June 27-29, 2010}, pp.\  231--243, 2010{\natexlab{a}}.

\bibitem[Neu et~al.(2010{\natexlab{b}})Neu, Gy{\"{o}}rgy, and
  Szepesv{\'{a}}ri]{neu2010ossp}
Neu, G., Gy{\"{o}}rgy, A., and Szepesv{\'{a}}ri, C.
\newblock The online loop-free stochastic shortest-path problem.
\newblock In \emph{Conference on Learning Theory {(COLT)}}, pp.\  231--243,
  2010{\natexlab{b}}.

\bibitem[Neu et~al.(2012)Neu, Gy{\"{o}}rgy, and
  Szepesv{\'{a}}ri]{neu2012unknown}
Neu, G., Gy{\"{o}}rgy, A., and Szepesv{\'{a}}ri, C.
\newblock The adversarial stochastic shortest path problem with unknown
  transition probabilities.
\newblock In \emph{Proceedings of the Fifteenth International Conference on
  Artificial Intelligence and Statistics, {(AISTATS)}}, pp.\  805--813, 2012.

\bibitem[Neu et~al.(2014)Neu, Gy{\"{o}}rgy, Szepesv{\'{a}}ri, and
  Antos]{neu2014bandit}
Neu, G., Gy{\"{o}}rgy, A., Szepesv{\'{a}}ri, C., and Antos, A.
\newblock Online {Markov Decision Processes} under bandit feedback.
\newblock \emph{{IEEE} Trans. Automat. Contr.}, 59\penalty0 (3):\penalty0
  676--691, 2014.

\bibitem[Pike-Burke et~al.(2018)Pike-Burke, Agrawal, Szepesvari, and
  Grunewalder]{pike2018bandits}
Pike-Burke, C., Agrawal, S., Szepesvari, C., and Grunewalder, S.
\newblock Bandits with delayed, aggregated anonymous feedback.
\newblock In \emph{International Conference on Machine Learning}, pp.\
  4105--4113. PMLR, 2018.

\bibitem[Quanrud \& Khashabi(2015)Quanrud and Khashabi]{quanrud2015online}
Quanrud, K. and Khashabi, D.
\newblock Online learning with adversarial delays.
\newblock \emph{Advances in neural information processing systems},
  28:\penalty0 1270--1278, 2015.

\bibitem[Raffin et~al.(2021)Raffin, Hill, Gleave, Kanervisto, Ernestus, and
  Dormann]{raffin2021stable}
Raffin, A., Hill, A., Gleave, A., Kanervisto, A., Ernestus, M., and Dormann, N.
\newblock Stable-baselines3: Reliable reinforcement learning implementations.
\newblock \emph{Journal of Machine Learning Research}, 2021.

\bibitem[Rosenberg \& Mansour(2019{\natexlab{a}})Rosenberg and
  Mansour]{rosenberg2019bandit}
Rosenberg, A. and Mansour, Y.
\newblock Online stochastic shortest path with bandit feedback and unknown
  transition function.
\newblock In \emph{Advances in Neural Information Processing Systems}, pp.\
  2209--2218, 2019{\natexlab{a}}.

\bibitem[Rosenberg \& Mansour(2019{\natexlab{b}})Rosenberg and
  Mansour]{rosenberg2019online}
Rosenberg, A. and Mansour, Y.
\newblock Online convex optimization in adversarial markov decision processes.
\newblock In \emph{International Conference on Machine Learning}, pp.\
  5478--5486. PMLR, 2019{\natexlab{b}}.

\bibitem[Rosenberg \& Mansour(2021{\natexlab{a}})Rosenberg and
  Mansour]{RosenbergM21}
Rosenberg, A. and Mansour, Y.
\newblock Oracle-efficient regret minimization in factored mdps with unknown
  structure.
\newblock In Ranzato, M., Beygelzimer, A., Dauphin, Y.~N., Liang, P., and
  Vaughan, J.~W. (eds.), \emph{Advances in Neural Information Processing
  Systems 34: Annual Conference on Neural Information Processing Systems 2021,
  NeurIPS 2021, December 6-14, 2021, virtual}, pp.\  11148--11159,
  2021{\natexlab{a}}.

\bibitem[Rosenberg \& Mansour(2021{\natexlab{b}})Rosenberg and
  Mansour]{rosenberg2021stochastic}
Rosenberg, A. and Mansour, Y.
\newblock Stochastic shortest path with adversarially changing costs.
\newblock In Zhou, Z. (ed.), \emph{Proceedings of the Thirtieth International
  Joint Conference on Artificial Intelligence, {IJCAI} 2021, Virtual Event /
  Montreal, Canada, 19-27 August 2021}, pp.\  2936--2942. ijcai.org,
  2021{\natexlab{b}}.

\bibitem[Rosenberg et~al.(2020)Rosenberg, Cohen, Mansour, and
  Kaplan]{cohen2020ssp}
Rosenberg, A., Cohen, A., Mansour, Y., and Kaplan, H.
\newblock Near-optimal regret bounds for stochastic shortest path.
\newblock In \emph{International Conference on Machine Learning}, pp.\
  8210--8219. PMLR, 2020.

\bibitem[Schuitema et~al.(2010)Schuitema, Bu{\c{s}}oniu, Babu{\v{s}}ka, and
  Jonker]{schuitema2010control}
Schuitema, E., Bu{\c{s}}oniu, L., Babu{\v{s}}ka, R., and Jonker, P.
\newblock Control delay in reinforcement learning for real-time dynamic
  systems: a memoryless approach.
\newblock In \emph{2010 IEEE/RSJ International Conference on Intelligent Robots
  and Systems}, pp.\  3226--3231. IEEE, 2010.

\bibitem[Schulman et~al.(2015)Schulman, Levine, Abbeel, Jordan, and
  Moritz]{schulman2015trust}
Schulman, J., Levine, S., Abbeel, P., Jordan, M., and Moritz, P.
\newblock Trust region policy optimization.
\newblock In \emph{International conference on machine learning}, pp.\
  1889--1897, 2015.

\bibitem[Schulman et~al.(2017)Schulman, Wolski, Dhariwal, Radford, and
  Klimov]{schulman2017proximal}
Schulman, J., Wolski, F., Dhariwal, P., Radford, A., and Klimov, O.
\newblock Proximal policy optimization algorithms.
\newblock \emph{arXiv preprint arXiv:1707.06347}, 2017.

\bibitem[Shani et~al.(2020{\natexlab{a}})Shani, Efroni, and Mannor]{ShaniEM20}
Shani, L., Efroni, Y., and Mannor, S.
\newblock Adaptive trust region policy optimization: Global convergence and
  faster rates for regularized mdps.
\newblock In \emph{The Thirty-Fourth {AAAI} Conference on Artificial
  Intelligence, {AAAI} 2020, The Thirty-Second Innovative Applications of
  Artificial Intelligence Conference, {IAAI} 2020, The Tenth {AAAI} Symposium
  on Educational Advances in Artificial Intelligence, {EAAI} 2020, New York,
  NY, USA, February 7-12, 2020}, pp.\  5668--5675. {AAAI} Press,
  2020{\natexlab{a}}.

\bibitem[Shani et~al.(2020{\natexlab{b}})Shani, Efroni, Rosenberg, and
  Mannor]{shani2020optimistic}
Shani, L., Efroni, Y., Rosenberg, A., and Mannor, S.
\newblock Optimistic policy optimization with bandit feedback.
\newblock In \emph{International Conference on Machine Learning}, pp.\
  8604--8613. PMLR, 2020{\natexlab{b}}.

\bibitem[Sherman et~al.(2023)Sherman, Koren, and Mansour]{sherman2023improved}
Sherman, U., Koren, T., and Mansour, Y.
\newblock Improved regret for efficient online reinforcement learning with
  linear function approximation.
\newblock \emph{arXiv preprint arXiv:2301.13087}, 2023.

\bibitem[Sutton \& Barto(2018)Sutton and Barto]{sutton2018reinforcement}
Sutton, R.~S. and Barto, A.~G.
\newblock \emph{Reinforcement learning: An introduction}.
\newblock MIT press, 2018.

\bibitem[Sutton et~al.(1999)Sutton, McAllester, Singh, and
  Mansour]{sutton1999policy}
Sutton, R.~S., McAllester, D., Singh, S., and Mansour, Y.
\newblock Policy gradient methods for reinforcement learning with function
  approximation.
\newblock \emph{Advances in neural information processing systems}, 12, 1999.

\bibitem[Tarbouriech et~al.(2020)Tarbouriech, Garcelon, Valko, Pirotta, and
  Lazaric]{tarbouriech2019noregret}
Tarbouriech, J., Garcelon, E., Valko, M., Pirotta, M., and Lazaric, A.
\newblock No-regret exploration in goal-oriented reinforcement learning.
\newblock In \emph{International Conference on Machine Learning}, pp.\
  9428--9437. PMLR, 2020.

\bibitem[Tarbouriech et~al.(2021)Tarbouriech, Zhou, Du, Pirotta, Valko, and
  Lazaric]{tarbouriech2021stochastic}
Tarbouriech, J., Zhou, R., Du, S.~S., Pirotta, M., Valko, M., and Lazaric, A.
\newblock Stochastic shortest path: Minimax, parameter-free and towards
  horizon-free regret.
\newblock \emph{Advances in Neural Information Processing Systems}, 34, 2021.

\bibitem[Thune et~al.(2019)Thune, Cesa-Bianchi, and
  Seldin]{thune2019nonstochastic}
Thune, T.~S., Cesa-Bianchi, N., and Seldin, Y.
\newblock Nonstochastic multiarmed bandits with unrestricted delays.
\newblock In \emph{Advances in Neural Information Processing Systems}, pp.\
  6541--6550, 2019.

\bibitem[Todorov et~al.(2012)Todorov, Erez, and Tassa]{todorov2012mujoco}
Todorov, E., Erez, T., and Tassa, Y.
\newblock Mujoco: A physics engine for model-based control.
\newblock In \emph{2012 IEEE/RSJ international conference on intelligent robots
  and systems}, pp.\  5026--5033. IEEE, 2012.

\bibitem[Tomar et~al.(2022)Tomar, Shani, Efroni, and
  Ghavamzadeh]{tomar2020mirror}
Tomar, M., Shani, L., Efroni, Y., and Ghavamzadeh, M.
\newblock Mirror descent policy optimization.
\newblock In \emph{The Tenth International Conference on Learning
  Representations, {ICLR} 2022, Virtual Event, April 25-29, 2022}, 2022.

\bibitem[Van Der~Hoeven \& Cesa-Bianchi(2022)Van Der~Hoeven and
  Cesa-Bianchi]{van2021nonstochastic}
Van Der~Hoeven, D. and Cesa-Bianchi, N.
\newblock Nonstochastic bandits and experts with arm-dependent delays.
\newblock In \emph{International Conference on Artificial Intelligence and
  Statistics}. PMLR, 2022.

\bibitem[Vernade et~al.(2017)Vernade, Capp{\'e}, and
  Perchet]{vernade2017stochastic}
Vernade, C., Capp{\'e}, O., and Perchet, V.
\newblock Stochastic bandit models for delayed conversions.
\newblock In \emph{Conference on Uncertainty in Artificial Intelligence}, 2017.

\bibitem[Vernade et~al.(2020)Vernade, Carpentier, Lattimore, Zappella, Ermis,
  and Brueckner]{vernade2020linear}
Vernade, C., Carpentier, A., Lattimore, T., Zappella, G., Ermis, B., and
  Brueckner, M.
\newblock Linear bandits with stochastic delayed feedback.
\newblock In \emph{International Conference on Machine Learning}, pp.\
  9712--9721. PMLR, 2020.

\bibitem[Vial et~al.(2022)Vial, Parulekar, Shakkottai, and
  Srikant]{vial2022regret}
Vial, D., Parulekar, A., Shakkottai, S., and Srikant, R.
\newblock Regret bounds for stochastic shortest path problems with linear
  function approximation.
\newblock In \emph{International Conference on Machine Learning}, pp.\
  22203--22233. PMLR, 2022.

\bibitem[Walsh et~al.(2009)Walsh, Nouri, Li, and Littman]{walsh2009learning}
Walsh, T.~J., Nouri, A., Li, L., and Littman, M.~L.
\newblock Learning and planning in environments with delayed feedback.
\newblock \emph{Autonomous Agents and Multi-Agent Systems}, 18\penalty0
  (1):\penalty0 83, 2009.

\bibitem[Wei et~al.(2021)Wei, Jahromi, Luo, and Jain]{wei2021learning}
Wei, C.-Y., Jahromi, M.~J., Luo, H., and Jain, R.
\newblock Learning infinite-horizon average-reward mdps with linear function
  approximation.
\newblock In \emph{International Conference on Artificial Intelligence and
  Statistics}, pp.\  3007--3015. PMLR, 2021.

\bibitem[Yang \& Wang(2019)Yang and Wang]{yang2019sample}
Yang, L. and Wang, M.
\newblock Sample-optimal parametric q-learning using linearly additive
  features.
\newblock In \emph{International Conference on Machine Learning}, pp.\
  6995--7004. PMLR, 2019.

\bibitem[Zanette et~al.(2020{\natexlab{a}})Zanette, Brandfonbrener, Brunskill,
  Pirotta, and Lazaric]{zanette2020frequentist}
Zanette, A., Brandfonbrener, D., Brunskill, E., Pirotta, M., and Lazaric, A.
\newblock Frequentist regret bounds for randomized least-squares value
  iteration.
\newblock In \emph{International Conference on Artificial Intelligence and
  Statistics}, pp.\  1954--1964, 2020{\natexlab{a}}.

\bibitem[Zanette et~al.(2020{\natexlab{b}})Zanette, Lazaric, Kochenderfer, and
  Brunskill]{zanette2020learning}
Zanette, A., Lazaric, A., Kochenderfer, M., and Brunskill, E.
\newblock Learning near optimal policies with low inherent bellman error.
\newblock In \emph{International Conference on Machine Learning}, pp.\
  10978--10989. PMLR, 2020{\natexlab{b}}.

\bibitem[Zhou \& Gu(2022)Zhou and Gu]{zhou2022computationally}
Zhou, D. and Gu, Q.
\newblock Computationally efficient horizon-free reinforcement learning for
  linear mixture mdps.
\newblock \emph{arXiv preprint arXiv:2205.11507}, 2022.

\bibitem[Zhou et~al.(2021)Zhou, Gu, and Szepesvari]{zhou2021nearly}
Zhou, D., Gu, Q., and Szepesvari, C.
\newblock Nearly minimax optimal reinforcement learning for linear mixture
  markov decision processes.
\newblock In \emph{Conference on Learning Theory}, pp.\  4532--4576. PMLR,
  2021.

\bibitem[Zhou et~al.(2019)Zhou, Xu, and Blanchet]{zhou2019learning}
Zhou, Z., Xu, R., and Blanchet, J.
\newblock Learning in generalized linear contextual bandits with stochastic
  delays.
\newblock In \emph{Advances in Neural Information Processing Systems}, pp.\
  5197--5208, 2019.

\bibitem[Zimin \& Neu(2013)Zimin and Neu]{zimin2013online}
Zimin, A. and Neu, G.
\newblock Online learning in episodic markovian decision processes by relative
  entropy policy search.
\newblock In \emph{Advances in Neural Information Processing Systems 26: 27th
  Annual Conference on Neural Information Processing Systems 2013. Proceedings
  of a meeting held December 5-8, 2013, Lake Tahoe, Nevada, United States},
  pp.\  1583--1591, 2013.

\bibitem[Zimmert \& Seldin(2020)Zimmert and Seldin]{zimmert2020optimal}
Zimmert, J. and Seldin, Y.
\newblock An optimal algorithm for adversarial bandits with arbitrary delays.
\newblock In \emph{International Conference on Artificial Intelligence and
  Statistics}, pp.\  3285--3294. PMLR, 2020.

\end{thebibliography}
\bibliographystyle{icml2023}


\newpage
\clearpage
\appendix
\onecolumn
\pagestyle{fancy}
\fancyhf{} 
\fancyfoot[C]{\vspace{20pt}\thepage}
\renewcommand{\headrulewidth}{0pt}
\renewcommand{\footrulewidth}{0pt}
\renewcommand{\partname}{}
\renewcommand{\thepart}{}
\addcontentsline{toc}{}{Appendix}
\addcontentsline{toc}{section}{Appendix}
\part{Appendix} 

\parttoc
\newpage
\section{Related Work}
\label{appendix:related-work}

In this section we provide a full review of the literature related to regret minimization in adversarial MDP with delayed feedback.
For completeness, we include topics that are not directly related to this paper.

\paragraph{Delays in RL without Regret Analysis.}
Delays were studied in the practical RL literature \citep{schuitema2010control,liu2014impact,changuel2012online,mahmood2018setting,derman2021acting}, but this is not related the topic of this paper.
In the theory literature, most previous work \citep{katsikopoulos2003markov,walsh2009learning} considered delays in the observation of the state.
That is, when the agent takes an action she is not certain what is the current state, and will only observe it in delay.
This setting is much more related to partially
observable MDPs (POMDPs) and motivated by scenarios like robotics system delays. Unfortunately, even planning
is computationally hard (exponential in the delay $d$) for delayed state observability \citep{walsh2009learning}.
The topic studied in this paper (i.e., delayed feedback) is inherently different, and is motivated by settings like recommendation systems.
Importantly, unlike delayed state observability, it is not computationally hard to handle delayed feedback. 
The challenges of delayed feedback are very different than the ones of delayed state observability, and include policy updates that occur in delay and exploration without observing feedback \citep{lancewicki2020learning}.

\paragraph{Delays in RL with Regret Analysis.}
This line of work is related to this paper the most.
\citet{howson2021delayed} studied delayed feedback in stochastic MDPs, and assume that the delays are also stochastic, i.e., sampled i.i.d from a fixed (unknown) distribution.
This is a restrictive assumption since it does not allow dependencies between costs and delays that are very common in practice.
In adversarial MDPs, delayed feedback was first studied by \citet{lancewicki2020learning}.
They proposed Policy Optimization algorithms that handle delays, but focused on the case of full-information feedback where the agent observes the entire cost function in the end of the episode instead of bandit feedback (where the agent observes only costs along its trajectory).
Full-information feedback is not a realistic assumption in most applications, and for bandit feedback they only prove sub-optimal regret of $(K+D)^{2/3}$ (ignoring dependencies in $S,A,H$).
Later, \citet{jin2022near} managed to achive a near-optimal regret bound of $\tilde \calO (H \sqrt{SAK} + (HSA)^{1/4} H\sqrt{D})$ for the case of known transition function and $\tilde \calO (H^2 S \sqrt{AK} + (HSA)^{1/4}  H\sqrt{D})$ for the case of unknown transitions.
However, their algorithm is based on the O-REPS method \citep{zimin2013online} which requires solving a computationally expensive global optimization problem and cannot be extended to function approximation.
Recently, \citet{dai2022follow} showed that delayed feedback in adversarial MDPs can also be dealt with using Follow-The-Perturbed-Leader (FTPL) algorithms.
The efficiency of FTPL algorithms is similar to Policy Optimization, but their regret bound is only $\tilde \calO (H^2 S \sqrt{AK} + \sqrt{HSA} \cdot H\sqrt{D})$.

\paragraph{Delays in multi-arm bandit (MAB).}
Delays were extensively studied in MAB and online optimization both in the stochastic setting \citep{dudik2011efficient,agarwal2012distributed,vernade2017stochastic,vernade2020linear,pike2018bandits,cesa2018nonstochastic,zhou2019learning,manegueu2020stochastic,lancewicki2021stochastic,cohen2021asynchronous,howson2022delayed}, and the adversarial setting \citep{quanrud2015online,cesa2016delay,thune2019nonstochastic,bistritz2019online,zimmert2020optimal,ito2020delay,gyorgy2020adapting,van2021nonstochastic,masoudian2022best}.
However, as discussed in \cite{lancewicki2020learning}, delays introduce new challenges in MDPs that do not appear in MAB.

\paragraph{Regret minimization in Tabular RL.}
There exists a rich literature on regret minimization in tabular MDPs.
In the stochastic case, the algorithms are mainly built on the \textit{optimism in face of uncertainty} approach \citep{jaksch2010near,azar2017minimax,dann2017unifying,jin2018q,efroni2019tight,tarbouriech2019noregret,tarbouriech2021stochastic,RosenbergM21,cohen2020ssp,cohen2021minimax,chen2021implicit}.
In the adversarial case, while a few algorithms use FTPL \citep{neu2012unknown,dai2022follow}, most algorithms are based on either the O-REPS method \citep{zimin2013online,rosenberg2019online,rosenberg2019bandit,rosenberg2021stochastic,jin2019learning,jin2021best,jin2020simultaneously,Lancewicki0M22,chen2020minimax,chen2021finding} or on Policy Optimization \citep{even2009online,neu2010loopfree,neu2010ossp,neu2014bandit,shani2020optimistic,luo2021policy,chen2022policy}.
Note that regret minimization in standard episodic MDPs is a special case of the model considered in this paper where $d^k = 0$ for every episode $k$.

\paragraph{Regret minimization in RL with Linear Function Approximation.}
In recent years the literature on regret minimization in RL has expanded to linear function approximation.
While in the stochastic case algorithms are still based on optimism \citep{jin2020provably,yang2019sample,zanette2020frequentist,zanette2020learning,ayoub2020model,zhou2021nearly,zhou2022computationally,vial2022regret,chen2022improved,min2022learning}, in the adversarial case O-REPS cannot be extended to linear function approximation without additional assumptions so algorithms are mostly based on Policy Optimization \citep{cai2020provably,abbasi2019politex,agarwal2020optimality,pmlr-v151-he22a,neu2021online,luo2021policy,wei2021learning}.

\paragraph{Policy Optimization in Deep RL.}
Policy Optimization is among the most widely used methods in deep Reinforcement Learning \citep{lillicrap2015continuous,levine2016end,gu2017deep}.
The origins of these algorithms are Policy Gradient \cite{sutton2018reinforcement}, Conservative Policy Iteration \citep{kakade2002approximately} and Natural Policy Gradient \citep{kakade2001natural}.
These have evolved into some of the state-of-the-art algorithms in RL, e.g., Trust Region Policy Optimization (TRPO) \citep{schulman2015trust}, Proximal Policy Optimization (PPO) \citep{schulman2017proximal} and Soft Actor-Critic (SAC) \citep{haarnoja2018soft}.
Recently, the connections between deep RL policy optimization algorithms and online learning regularization methods (like Follow-The-Regularized-Leader and Online-Mirror-Descent) were studied and explained \citep{ShaniEM20,tomar2020mirror}.

\begin{remark}[The Loop-Free Assumption]
\label{remark:loop-free}
    We warn the readers that some of the works mentioned in this section (mainly in the adversarial MDP literature, e.g., \citet{luo2021policy}) present a slightly different dependence in the horizon $H$.
    The reason is that they make the \textit{loop-free assumption}, i.e., they assume that the state space consists of $H$ disjoint sets $\calS = \calS_1 \cup \calS_2 \cup \dots \cup \calS_H$ such that in step $h$ the agent can only be found in states from the set $\calS_h$.
    Effectively, this means that their state space is larger than ours by a factor of $H$.
    So when they present a regret bound of $\tilde \calO(H^2 S \sqrt{A K})$, this implies a bound of $\tilde \calO(H^3 S \sqrt{A K})$ in the model presented in this paper.
    We emphasize that these differences are only due to different models, and not due to actual different regret bounds.
\end{remark}

\newpage

\section{Delay-Adapted Policy Optimization for (Tabular) Adversarial MDP with Known Transition}
\label{app:tabular known}
\begin{algorithm}
    \caption{Delay-Adapted Policy Optimization with Known Transition Function (Tabular)}  
    \label{alg:delayed-OPPO-known-p}
    \begin{algorithmic}
        \STATE \textbf{Input:} state space $\mathcal{S}$, action space $\mathcal{A}$, horizon $H$, transition function $p$, learning rate $\eta > 0$, exploration parameter $\gamma > 0$.
        
        \STATE \textbf{Initialization:} 
        Set $\pi_{h}^{1}(a \mid s) = \nicefrac{1}{A}$ for every $(s,a,h) \in \mathcal{S} \times \mathcal{A} \times [H]$.
        
        \FOR{$k=1,2,\dots,K$}
            
            \STATE Play episode $k$ with policy $\pi^k$, observe trajectory $\{ (s^k_h,a^k_h) \}_{h=1}^H$ and compute $q^k_h(s)$ for every $(s,h) \in \calS \times [H]$.
            
            \STATE {\color{gray} \# Policy Evaluation}
            
            \FOR{$j$ such that $j + d^j = k$}
            
                \STATE Observe bandit feedback $\{ c^j_h(s^j_h,a^j_h) \}_{h=1}^H$ and set $B^j_{H+1}(s,a) = 0$ for every $(s,a) \in \calS \times \calA$.
                
                \FOR{$h=H,H-1,\dots,1$}
                    
                    \FOR{$(s,a) \in \calS \times \calA$}
                    
                        \STATE Compute $r^j_h(s,a) = \frac{\pi^j_h(a \mid s)}{\max \{ \pi^j_h(a\mid s) , \pi_{h}^{k}(a \mid s) \}}$ and $L^j_h = \sum_{h'=h}^H c^j_{h'}(s^j_{h'},a^j_{h'})$.
                    
                        \STATE Compute $\hat{Q}_{h}^{j}(s,a) = r^j_h(s,a) \frac{ \mathbb{I}\{s_{h}^{j}=s,a_{h}^{j}=a\} L^j_h }{ q_{h}^{j}(s)\pi_{h}^{j}(a \mid s) +\gamma}$ and $b^j_h(s) = \sum_{a \in \calA} \frac{3 \gamma H \pi^k_h(a \mid s) r^j_h(s,a)}{q^j_h(s) \pi^j_h(a \mid s) + \gamma}$.
                        
                        \STATE Compute $B_{h}^{j}(s,a) =b_{h}^{j}(s) + \sum_{s' \in \calS,a' \in \calA} p_h(s'\mid s,a) \pi_{h+1}^{j}(a' \mid s') B^j_{h+1}(s',a')$.
                        
                    \ENDFOR
                
                \ENDFOR
            \ENDFOR
            
            \STATE {\color{gray} \# Policy Improvement}
            
            \STATE Define the policy $\pi^{k+1}$ for every $(s,a,h) \in \calS \times \calA \times [H]$ by:
            \[
                \pi^{k+1}_h(a \mid s) = \frac{\pi^k_h(a \mid s) \exp \l(-\eta \sum_{j: j+d^j = k}  ( \hat Q_{h}^{j}(s,a) - B_h^j(s,a) ) \r)} 
                {\sum_{a' \in \mathcal{A}} \pi^k_h(a' \mid s) \exp \l(-\eta \sum_{j: j+d^j = k}  ( \hat Q_{h}^{j}(s,a') - B_h^j(s,a') ) \r)}.
            \]
            
        \ENDFOR
        
    \end{algorithmic}
\end{algorithm}

\begin{theorem}
    \label{thm:regret-bound-known-dynamics}
    Set $\eta = \l( H^{2}SAK+H^{4}(K+D) \r)^{-1/2}$ and $\gamma = 2 \eta H$.
    Running \cref{alg:delayed-OPPO-known-p} in an adversarial MDP $\calM = (\calS, \calA, H, p, \{ c^k \}_{k=1}^K)$ with known transition function and delays $\{ d^k \}_{k=1}^K$ guarantees, with probability at least $1 - \delta$,
    \[
        \regret
        =
        \mathcal{O}\l( H^2 \sqrt{S A K} \log \frac{K H S A}{\delta} + H^3 \sqrt{K + D} \log \frac{K H S A}{\delta}  + H^4 \dmax \log \frac{K H S A}{\delta} \r).
    \]
\end{theorem}

\subsection{The good event}
\label{sec:good event known}
Let $\logterm = 10 \log \frac{10 KHSA}{\delta}$, $\tilde{\mathcal{H}}^{k}$ be the history of episodes $\{j:j+d^{j}<k\}$, 
and define $\bbE_k[\cdot] = \bbE\big[ \cdot \mid  \tilde{\mathcal{H}}^{k+d^k} \big]$
Define the following events: 
\begin{align*}
    E^d
    & =
    \l\{ \forall (h,s). \  \sum_{k=1}^{K}\sum_{j=1}^{K}\sum_{a}\mathbb{I}\{j\leq k+d^{k}<j+d^{j}\}\pi_{h}^{k+d^{k}}(a\mid s)(\hat{Q}_{h}^{k}(s,a)-Q_{h}^{k}(s,a))\le\frac{10 H^2 \dmax\log\frac{10 H S}{\delta}}{\gamma}  \r\}
    \\
    E^*
    & =
    \l\{ \sum_{k=1}^K \sum_{h,s,a} q^*_h(s,a) (\hat{Q}_{h}^{k}(s, a) - Q_{h}^{k}(s, a)) \le \frac{H^2  \log \frac{10 H S A}{\delta}}{\gamma} \r\}
    \\
    E^b
    & = 
    \l\{
        \sum_{k=1}^{K}\sum_{h,s,a} \frac{q_{h}^{*}(s)\pi_{h}^{k+d^{k}}(a\mid s) r^k_h(s,a) \mathbb{I}\{s_{h}^{k}=s,a_{h}^{k}=a\}}{(q_{h}^{k}(s,a) + \gamma)^{2}}
        - \sum_{k=1}^{K}\sum_{h,s,a} \frac{q_{h}^{*}(s)\pi_{h}^{k+d^{k}}(a\mid s) r^k_h(s,a)}{q_{h}^{k}(s,a) + \gamma}\leq\frac{H}{\gamma^{2}}\ln\frac{10 H}{\delta}
    \r\}
    \\
    E^f
    & = 
    \vast\{
        \sum_{k=1}^K \bbE_k \l[\sum_{h,s} q_{h}^{*}(s) \l\langle \pi_{h}^{k + d^k}(\cdot \mid s),\hat{Q}_{h}^{k}(s,\cdot) \r\rangle \r]
            - \sum_{k=1}^K \sum_{h,s} q_{h}^{*}(s) \l\langle \pi_{h}^{k + d^k}(\cdot \mid s),\hat{Q}_{h}^{k}(s,\cdot) \r\rangle 
        \\  & \qquad \qquad \qquad \qquad \qquad \qquad \qquad \qquad \qquad \qquad \qquad \qquad \qquad \qquad \qquad \leq \frac{1}{3} \sum_{k=1}^{K} \sum_{h,s} q_{h}^{*}(s)b_{h}^{k}(s) 
            + \frac{H^{2}}{\gamma} \ln \frac{10}{\delta}
    \vast\}
\end{align*}

The good event is the intersection of the above events. 
The following lemma establishes that the good event holds with high probability. 

\begin{lemma}[The Good Event]
    \label{lemma:good-event-known-dynamics}
    Let $\bbG = E^d \cap E^* \cap E^b \cap E^f$ be the good event. 
    It holds that $\Pr [ \bbG ] \geq 1-\delta$.
\end{lemma}

\begin{proof}
    We'll show that each of the events $\lnot E^d,  \lnot E^* ,  \lnot E^b , \lnot E^f$ holds with probability of at most $\delta / 4$ and so by the union bound $\Pr [ \bbG ] \geq 1-\delta$.
    
    \textbf{Event $E^d$:} Fix $s$ and $h$. For every $(h',s',a$) set:
    \begin{align*}
        z_{h'}^{k}(s',a)
        & =
        \mathbb{I} \l\{ s'=s,h'=h \r\} \sum_{j=1}^{K} \mathbb{I}\{j\leq k+d^{k}<j+d^{j}\}\pi_{h}^{k+d^{k}}(a\mid s) r^k_h(s,a) Q_{h}^{j}(s,a)
        \\
        Z_{h'}^{k}(s',a)
        & =
        \mathbb{I} \l\{ s'=s,h'=h \r\} \sum_{j=1}^{K} \mathbb{I}\{j\leq k+d^{k}<j+d^{j}\} \pi_{h}^{k+d^{k}}(a\mid s) r^k_h(s,a) M_h^k(s,a)
        \\
        M_h^k(s,a) &=  \bbI\{s_{h}^{k}=s,a_{h}^{k}=a\} \sum_{\tilde{h}=1}^{H} c_{\tilde{h}}^{k}(s_{\tilde{h}}^{k} , a_{\tilde{h}}^{k}) 
        + (1-\bbI\{s_{h}^{k}=s,a_{h}^{k}=a\}) Q_{h}^{k}(s,a).
    \end{align*}
    Note that $\bbE_k[Z_{h'}^k(s',a)] = z_{h'}^k(s',a)$ and,
    \begin{align*}
        \sum_{h',s',a} \frac{\mathbb{I} \{s_{h}^{k}=s,a_{h}^{k}=s\}Z_{h}^{k}(s,a)}{q_{h}^{k}(s,a) + \gamma} 
        & = 
        \sum_{j=1}^{K}\sum_{a}\mathbb{I}\{j\leq k+d^{k}<j+d^{j}\}\pi_{h}^{k+d^{k}}(a\mid s)\hat{Q}_{h}^{k}(s,a)
        \\
        \sum_{h',s',a} z_{h}^{k}(s,a)
        & \leq
        \sum_{j=1}^{K}\sum_{a}\mathbb{I}\{j\leq k+d^{k}<j+d^{j}\}\pi_{h}^{k+d^{k}}(a\mid s) Q_{h}^{k}(s,a),
    \end{align*}
    where in the inequality we used the fact that $r^k_h(s,a) \le 1$.
    Finally, we use \cref{lemma:concentration} with $\tilde q_h^k(s,a) = q_h^k(s,a)$ and $R \leq 2H \dmax$ since the number of $j$s such that $j\leq k+d^{k}<j+d^{j}$ is at most $2 \dmax$. Thus, the event holds for $(s,h)$ with probability $1 - \frac{\delta}{10 H S}$. By taking the union bound over all $h$ and $s$, $E^d$ holds with probability $1 - \frac{\delta}{10}$.
    
    \textbf{Event $E^*$ (Lemma C.2 of \citet{luo2021policy}):} $E^*$ holds with probability of at least $1-\delta/10$ by applying \cref{lemma:concentration}  with $\tilde q_{h'}^k(s',a) = q_{h'}^k(s'a)$, $z_h^k(s,a) = q^{*}_h(s,a) r^k_h(s,a) Q_h^k(s,a)$ and,
    \[
        Z_{h}^{k}(s,a)=q_{h}^{*}(s,a) r^k_h(s,a)
        \l( \bbI\{s_{h}^{k}=s,a_{h}^{k}=a\} \sum_{h'=1}^{H}c_{h'}^{k}(s_{h'}^{k},a_{h'}^{k}) + (1-\bbI\{s_{h}^{k}=s,a_{h}^{k}=a\})Q_{h}^{k}(s,a) \r).
    \]
    Note that $R \leq H$, $\frac{\mathbb{I} \{s_{h}^{k}=s,a_{h}^{k}=s\}Z_{h}^{k}(s,a)}{\tilde q_{h}^{k}(s,a) + \gamma} = q^{*}_h(s,a) \hat Q_h^k(s,a)$ and $\frac{q_{h}^{k}(s,a)z_{h}^{k}(s,a)}{\tilde{q}_{h}^{k}(s,a)} \leq q^{*}_h(s,a)Q_h^k(s,a)$.
    
    \textbf{Event $E^b$:} Similar to the last two events, $E^b$ holds with probability of at least $1 - \delta / 10$ by applying \cref{lemma:concentration} with $\tilde q_{h'}^k(s',a) = q_{h'}^k(s',a)$ and $z_h^k(s,a) = Z_h^k(s,a) = \frac{q_{h}^{*}(s) \pi_{h}^{k+d^{k}}(a\mid s) r^k_h(s,a)}{q_{h}^{k}(s,a) + \gamma}$.

    \textbf{Event $E^f$:}
    Let $Y_{k}= \sum_{h,s} q_{h}^{*}(s) \l\langle \pi_{h}^{k + d^k}(\cdot \mid s),\hat{Q}_{h}^{k}(s,\cdot) \r\rangle $.
    We'll use a variant of Freedman's inequality (\cref{lemma:freedman}) to bound $\sum_{k=1}^{K} \bbE_k [Y_{k}] - \sum_{k=1}^{K}Y_{k} $. Note that:
    \begin{align*}
    \bbE_k \l[ Y_{k}^2 \r] & = \bbE_k \l[ \l( \sum_{h,s,a} q_{h}^{*}(s)\pi_{h}^{k + d^k}(a \mid s)\hat{Q}_{h}^{k}(s,a) \r)^{2}  \r] \\
    & \leq
    \bbE_k \l[ \l( \sum_{h,s,a} q_{h}^{*}(s)\pi_{h}^{k + d^k}(a \mid s) \r) \l( \sum_{h,s,a} q_{h}^{*}(s)\pi_{h}^{k + d^k}(a \mid s)(\hat{Q}_{h}^{k}(s,a))^{2} \r)  \r]\\
    & \leq H \bbE_k \l[ \sum_{h,s,a} q_{h}^{*}(s)\pi_{h}^{k + d^k}(a \mid s)\frac{r^k_h(s,a)^2 (L_{h}^{k})^{2} \mathbb{I} \{s_{h}^{k}=s,a_{h}^{k}=a\}}{(q_{h}^{k}(s,a) + \gamma)^{2}}  \r]
    \\
    & \le
    H^3 \sum_{h,s,a} q_{h}^{*}(s)\pi_{h}^{k + d^k}(a \mid s) \frac{q^k_h(s,a) r^k_h(s,a)}{(q_{h}^{k}(s,a) + \gamma)^{2}}
    \\
    & \leq H^{3} \sum_{h,s,a} \frac{q_{h}^{*}(s)\pi_{h}^{k + d^k}(a \mid s) r^k_h(s,a)}{q_{h}^{k}(s,a) + \gamma},
    \end{align*}
    where the first inequality is Cauchy-Schwartz inequality.
    Also, $|Y_{k}|\leq\frac{H^{2}}{\gamma}$. Therefore by \cref{lemma:freedman} with probability $1 - \delta/10$,
    \begin{align*}
        \sum_{k=1}^{K} \bbE_k [Y_{k}] - \sum_{k=1}^{K}Y_{k} 
        & \leq
        \gamma H\sum_{k=1}^{K} \l(\sum_{h,s,a} \frac{q_{h}^{*}(s)\pi_{h}^{k + d^{k}}(a\mid s) r^k_h(s,a)}{q_{h}^{k}(s,a) + \gamma} \r) 
        + \frac{H^{2}}{\gamma} \ln \frac{10}{\delta}
        \\
        & = \frac{1}{3} \sum_{k=1}^{K} \sum_{h,s} q_{h}^{*}(s)b_{h}^{k}(s) 
        + \frac{H^{2}}{\gamma} \ln \frac{10}{\delta}. \qedhere
    \end{align*}
\end{proof}

\subsection{Proof of the main theorem}

\begin{proof}[Proof of \cref{thm:regret-bound-known-dynamics}]
By \cref{lemma:good-event-known-dynamics}, the good event holds with probability $1 - \delta$.
We now analyze the regret under the assumption that the good event holds.
We start with the following regret decomposition,
\begin{align*}
    \regret  = & \sum_{k=1}^{K} \sum_{h,s}q_{h}^{*}(s) \l\langle \pi_{h}^{k} (\cdot \mid s) - \pi_{h}^{*} (\cdot \mid s),Q_{h}^{k}(s, \cdot) \r\rangle 
    = \underbrace{
    \sum_{k=1}^{K} \sum_{h,s}q_{h}^{*}(s) \l\langle \pi_{h}^{k + d^{k}} (\cdot \mid s),Q_{h}^{k}(s, \cdot) - \hat{Q}_{h}^{k}(s, \cdot) \r\rangle }
    _{\textsc{Bias}_{1}} 
    \\
    & + \underbrace{
    \sum_{k=1}^{K} \sum_{h,s}q_{h}^{*}(s) \l\langle \pi_{h}^{*} (\cdot \mid s), \hat{Q}_{h}^{k}(s, \cdot) - Q_{h}^{k}(s, \cdot) \r\rangle 
    }_{\textsc{Bias}_{2}} 
    + \underbrace{ 
    \sum_{k=1}^{K} \sum_{h,s}q_{h}^{*}(s) \l\langle \pi_{h}^{k} (\cdot \mid s) - \pi_{h}^{*} (\cdot \mid s),B_{h}^{k}(s, \cdot) \r\rangle 
    }_{\textsc{Bonus}} 
    \\
    & + \underbrace{ 
    \sum_{k=1}^{K} \sum_{h,s}q_{h}^{*}(s) \l\langle \pi_{h}^{k + d^{k}} (\cdot \mid s) - \pi_{h}^{*} (\cdot \mid s), \hat{Q}_{h}^{k}(s, \cdot) - B_{h}^{k}(s, \cdot) \r\rangle 
    }_{\textsc{Reg}} 
    \\
    & + \underbrace{ 
    \sum_{k=1}^{K} \sum_{h,s}q_{h}^{*}(s) \l\langle \pi_{h}^{k} (\cdot \mid s) - \pi_{h}^{k + d^{k}} (\cdot \mid s),Q_{h}^{k}(s, \cdot) - B_{h}^{k}(s, \cdot) \r\rangle 
    }_{\textsc{Drift}},
\end{align*}
where the first equality is by \cref{lemma: value diff} (value difference lemma).
$\textsc{Bias}_2$ is bounded under event $E^*$ by $\mathcal{O}\l( \frac{H^2 \logterm}{\gamma} \r)$. 
The other four terms are bounded in \cref{lemma: bias1,lemma:bonus,lemma:reg,lemma:drift}. 
Overall,
\begin{align*}
	\regret & \leq\underbrace{\frac{2}{3}\sum_{k=1}^{K}\sum_{h,s}q_{h}^{*}(s)b_{h}^{k}(s)+\mathcal{O}\l( \eta H^{4}(K+D)+ \frac{\eta H^{4}\dmax \logterm}{\gamma} + \frac{H^{2}}{\gamma}\logterm \r)}_{\textsc{Bais}_{1}} + \underbrace{\mathcal{O}\l( \frac{H^{2}}{\gamma}\logterm \r)}_{\textsc{Bais}_{2}} \\
	      & \qquad+
	          \underbrace{3 \gamma H^2 SA K  - \sum_{k=1}^{K}\sum_{h,s}q_{h}^{*}(s)b_{h}^{k}(s)}_{\textsc{Bonus}}  + \underbrace{\frac{H\ln A}{\eta}+\frac{1}{3}\sum_{k=1}^{K}\sum_{h,s}q_{h}^{*}(s)b_{h}^{k}(s) + \mathcal{O}\l(\eta H^{5}K+\eta\frac{H^{3}\logterm}{\gamma^{2}}\r)}_{\textsc{Reg}}      \\
	      & \qquad+ \underbrace{\mathcal{O}\l( \eta H^{5}(K+D)+\frac{\eta H^{5}\dmax\logterm}{\gamma}\r)}_{\textsc{Drift}}                                                                                          \\
	      & \leq \mathcal{O}\l( \frac{H\ln A}{\eta} + \gamma H^2 SA K + \eta H^{5}(K+D)+ \eta\frac{H^{3} \logterm}{\gamma^{2}}+\frac{H^{2}}{\gamma}\logterm+\frac{\eta H^{5}\dmax\logterm}{\gamma}\r).
\end{align*}
For $\eta=\frac{1}{\sqrt{H^{2}SAK+H^{4}(K+D)}}$ and $\gamma=2\eta H$,
we get:
$
    \regret \le \mathcal{O}\l( H^{2}\sqrt{SAK}\logterm 
    + H^{3}\sqrt{D}\logterm
    +H^{3}\sqrt{K}\logterm
    +H^{4}\dmax\logterm \r).
$
\end{proof}

\subsection{Bound on $\textsc{Bias}_1$}
\begin{lemma}
    \label{lemma: bias1}
    Under the good event,
    $
	\textsc{Bias}_{1} \leq \frac{2}{3}\sum_{k=1}^{K}\sum_{h,s}
	                  q_{h}^{*}(s)b_{h}^{k}(s)
	                  + \mathcal{O}\l( \eta H^{4}(K+D)
	                  + \frac{ \eta H^{4}\dmax\logterm}{\gamma}
	                  + \frac{H^{2}}{\gamma}\logterm \r).                                              
    $
\end{lemma}

\begin{proof}
    Let $Y_{k}= \sum_{h,s} q_{h}^{*}(s) \l\langle \pi_{h}^{k + d^k}(\cdot \mid s),\hat{Q}_{h}^{k}(s,\cdot) \r\rangle $. It holds that
    \begin{align*}
        \textsc{Bias}_1 
        = 
        \sum_{k=1}^{K} \sum_{h,s} q_{h}^{*}(s) \l\langle \pi_{h}^{k + d^k}(\cdot \mid s),Q_{h}^{k}(s,\cdot) \r\rangle 
        - \sum_{k=1}^{K} \bbE_k [Y_{k} ] 
        + \sum_{k=1}^{K} \bbE_k [Y_{k} ] 
        - \sum_{k=1}^{K}Y_{k}.
    \end{align*}
    Under event $E^f$ it holds that
    \begin{align*}
        \sum_{k=1}^{K} \bbE_k [Y_{k} ] - \sum_{k=1}^{K}Y_{k} 
        & \leq \frac{1}{3} \sum_{k=1}^{K} \sum_{h,s} q_{h}^{*}(s)b_{h}^{k}(s) 
        + \frac{H^{2}}{\gamma} \ln \frac{10}{\delta}.
    \end{align*}
In addition,
\begin{align*} 
    \sum_{k=1}^{K} & \sum_{h,s} q_{h}^{*}(s) 
    \l\langle\pi_{h}^{k + d^{k}}(\cdot\mid s),Q_{h}^{k}(s,\cdot)\r\rangle 
    - \sum_{k=1}^{K} \bbE_k [Y_{k}]=\\
    & =
    \sum_{k=1}^{K} \sum_{h,s,a} q_{h}^{*}(s)\pi_{h}^{k + d^{k}}(a\mid s)Q_{h}^{k}(s,a) \l(1 - \frac{q_{h}^{k}(s,a) r^k_h(s,a)}{q_{h}^{k}(s,a) + \gamma} \r)
    \\
    & =
    \sum_{k=1}^{K} \sum_{h,s,a} q_{h}^{*}(s)\pi_{h}^{k + d^{k}}(a\mid s)Q_{h}^{k}(s,a) \l(1 - \frac{(q_{h}^{k}(s,a) + \gamma) r^k_h(s,a)}{q_{h}^{k}(s,a) + \gamma} \r)  + 
    \sum_{k=1}^{K} \sum_{h,s,a} q_{h}^{*}(s)\pi_{h}^{k + d^{k}}(a\mid s)Q_{h}^{k}(s,a) \frac{\gamma r^k_h(s,a)}{q_{h}^{k}(s,a) + \gamma}
    \\
    & \le
    \sum_{k=1}^{K} \sum_{h,s,a} q_{h}^{*}(s)\pi_{h}^{k + d^{k}} (a\mid s)Q_{h}^{k}(s,a) (1 - r^k_h(s,a)) + \sum_{k=1}^{K} \sum_{h,s,a} q_{h}^{*}(s) \frac{\gamma H\pi_{h}^{k + d^{k}}(a\mid s) r^k_h(s,a)}{q_{h}^{k}(s,a) + \gamma} 
    \\
    & \leq
    H \sum_{k=1}^{K} \sum_{h,s,a} q_{h}^{*}(s)\l( \max \{ \pi_{h}^{k}(a\mid s),\pi_{h}^{k + d^{k}}(a\mid s)\} - \pi_{h}^{k}(a\mid s)\r)+ \sum_{k=1}^{K} \sum_{h,s,a} q_{h}^{*}(s) \frac{\gamma H\pi_{h}^{k + d^{k}}(a\mid s) r^k_h(s,a)}{q_{h}^{k}(s,a) + \gamma} 
    \\
    & \leq 
    H \sum_{k=1}^{K} \sum_{h,s} q_{h}^{*}(s) \Vert \pi_{h}^{k + d^{k}}(\cdot\mid s) - \pi_{h}^{k}(\cdot \mid s)\Vert_1
    + \frac{1}{3} \sum_{k=1}^{K} \sum_{h,s} q_{h}^{*}(s)b_{h}^{k}(s).
\end{align*}
Finally, using \cref{lemma:l1 drift},
\begin{align*}
	\textsc{Bias}_{1} & \leq\frac{2}{3}\sum_{k=1}^{K}\sum_{h,s}q_{h}^{*}(s)b_{h}^{k}(s)+H\sum_{h,s}q_{h}^{*}(s)\sum_{k=1}^{K}
	                    \Vert\pi_{h}^{k+d^{k}}(\cdot\mid s)-\pi_{h}^{k}(\cdot\mid s)\Vert_1+\frac{H^{2}}{\gamma}\logterm 
	                    \\
	                  & \leq\frac{2}{3}\sum_{k=1}^{K}\sum_{h,s}q_{h}^{*}(s)b_{h}^{k}(s) 
	                    + H\sum_{h,s}q_{h}^{*}(s)\l(\mathcal{O}\l( \eta H^{2}(K+D)
	                    +\frac{\eta H^2\dmax\logterm}{\gamma}\r)\r)
	                    +\frac{H^{2}}{\gamma}\logterm
	                    \\
	                  & =\frac{2}{3}\sum_{k=1}^{K}\sum_{h,s}
	                  q_{h}^{*}(s)b_{h}^{k}(s)
	                  +\mathcal{O}\l( \eta H^{4}(K+D)
	                  +\frac{\eta H^{4}\dmax \logterm}{\gamma}\r)
	                  +\frac{H^{2}}{\gamma}\logterm. \qedhere                                         
\end{align*}
\end{proof}

\subsection{Bound on $\textsc{Bonus}$}
\begin{lemma}
    \label{lemma:bonus}
    It holds that
    $
        \textsc{Bonus}  \leq 3 \gamma H^2 SA K - \sum_{k,h,s}q_{h}^{*}(s)b_{h}^{k}(s).
    $
\end{lemma}
\begin{proof}
    Note that $B_{h}^{k}$ is the $Q$-function of policy $\pi^k$ with respect to the cost function $b^k$. Hence, by the value difference
    difference lemma (\cref{lemma: value diff}),
    \begin{align*}
    	\sum_{h,s}q_{h}^{k}(s)b_{h}^{k}(s) - \sum_{h,s}q_{h}^{*}(s)b_{h}^{k}(s)  =V^{\pi^{k}}_1(\sinit;b^k) - V^{\pi^{*}}(\sinit;b^k )
        =\sum_{h,s}q_{h}^{*}(s) \l\langle\pi_{h}^{k}(\cdot\mid s)-\pi_{h}^{*}(\cdot\mid s),B_{h}^{k}(s,\cdot)\r\rangle .
    \end{align*}
    Summing over $k$ we get:
    $
        \textsc{Bonus}  = \sum_{k,h,s}q_{h}^{k}(s)b_{h}^{k}(s) - \sum_{k,h,s}q_{h}^{*}(s)b_{h}^{k}(s).
    $
    For last,
    \begin{align*}
        \sum_{k=1}^{K}\sum_{h,s}q_{h}^{k}(s)b^k_h(s)
        & =
        3 \gamma H \sum_{k=1}^{K}\sum_{h,s,a}\frac{q_{h}^{k}(s) \pi^{k+d^k}_h(a \mid s) r^k_h(s,a)}{q_{h}^{k}(s)\pi^{k}_h(a \mid s) + \gamma}
        \leq
        3 \gamma H^2 SA K. 
    \end{align*}
    where the last uses the fact that $\pi^{k+d^k}_h(a\mid s)r_h^k(s,a)\leq \pi^{k}_h(a \mid s)$.
\end{proof}

\begin{remark}
    \label{remark:ratio}
    The adaptation to delay via the ratio $r_h^k(s,a)$ is simple, yet crucial.
    The main reason is the following.
    While in MAB the ratio $\pi^{k+d^k}_h(a\mid s) / \pi^{k}_h(a\mid s)$ is always bounded by a constant \citep[Lemma 11]{thune2019nonstochastic}, in MDPs it can be as large as $e^{\dmax}$. 
    In fact, even the ratio $\pi^{k+1}_h(a\mid s) / \pi^{k}_h(a\mid s)$ can be of order $e^{\dmax}$, because even if an action $a$ is chosen with probability close to $1$, the estimator $\hat{Q}_h^k(s,a)$ can be as large as $\Omega(1/\gamma)$, as long as the visitation probability to state $s$ is smaller than $1 / \gamma$. 
    This can cause radical changes in the probability to take action $a$ (and as a consequence in the probability of the rest of the actions in that state). 
    For example, assume we have two actions $a_1,a_2$ with $\pi_h^k(a_1\mid s) = 1 / (e^{-\dmax}+1)$, $\pi_h^k(a_2\mid s) = e^{-\dmax} / (e^{-\dmax} + 1)$ and $q_h^k(s) \leq \gamma$. 
    Now, assume that the feedback from $\dmax$ episodes arrive at at the end of episode $k$ for which the agent visited in the state-action pair $(s,a_1)$. 
    Further assume the cost-to-go from $(s,a_1)$ was of order $H$. 
    This would imply that $\eta \sum_{j:j+d^j = k} (\hat{Q}_h^j(s,a) - B_h^j(s,a)) \approx \dmax$. Hence, the probability to take each of the actions under $\pi^{k+1}$ would be approximately $1/2$. In particular, $\pi^{k+1}_h(a_2\mid s) / \pi^{k}_h(a_2\mid s) = \Omega((e^{-\dmax}+1)/e^{-\dmax}) = \Omega(e^{\dmax})$.
\end{remark}

\subsection{Bound on $\textsc{Reg}$}

\begin{lemma}
    \label{lemma:reg}
    For $\eta\leq\frac{\gamma}{2H}$ it holds that
    $
    	\textsc{Reg} 
    	\leq \frac{H\ln A}{\eta}
    	+ \frac{1}{3} \sum_{k,h,s} q_h^{*}(s)b_h^k(s)
    	+ \mathcal{O}\l(  \eta H^{5}K 
    	+  \eta\frac{H^{3} \logterm}{\gamma^{2}} \r).
    $
\end{lemma}

\begin{proof}
    By \cref{corollary:delayed-exp-weights-regret}, since $\max_{k,h,s,a} B_h^k(s,a) \leq 3H^2$, 
    \begin{align}
    \nonumber
    	\textsc{Reg} & \leq\frac{H\ln A}{\eta}
            	       + 2\eta\sum_{k,h,s,a}q_h^{*}(s)\pi_{h}^{k+d^k}(a\mid s)\l(\hat{Q}_{h}^{k}(s,a)-B_{h}^{k}(s,a)\r)^{2} + \mathcal{O}(\eta H^5 K)
            	       \\
        	         & \leq\frac{H\ln A}{\eta} 
            	       + 2 \eta \sum_{k,h,s,a}q_h^{*}(s)\pi_{h}^{k+d^k}(a\mid s) \hat{Q}_{h}^{k}(s,a)^{2}
            	       + 2\eta\sum_{k,h,s,a}q_h^{*}(s)\pi_{h}^{k + d^k}(a\mid s) B_{h}^{k}(s,a)^{2} + \mathcal{O}(\eta H^5 K). 
        \label{eq:lemma reg init bound}
    \end{align}
    For the middle term
    \begin{align*}
    	2\eta\sum_{k,h,s,a}q_h^{*}(s)
    	\pi_{h}^{k+d^k}(a\mid s)\hat{Q}_{h}^{k}(s,a)^{2} & \leq 2\eta\sum_{k,h,s,a}q_h^{*}(s)\pi_{h}^{k + d^k}(a\mid s)
    	                                                   \frac{H^{2} r^k_h(s,a)^2 \mathbb{I}\{s_{h}^{k}=s,a_{h}^{k}=a\}}
    	                                                   {(q_{h}^{k}(s,a) + \gamma)^{2}}
                                                        \\
                                                       & \leq 2 \eta H^{2}\sum_{k,h,s,a}
                                                          \frac{q_h^{*}(s)\pi_{h}^{k + d^k}(a\mid s) r^k_h(s,a)}
                                                          {q_{h}^{k}(s,a)+\gamma}
                                                          +  \mathcal{O}\l( \eta\frac{H^{3} \logterm}{\gamma^{2}} \r)  
                                                        \\
                                                       & \leq\frac{2 \eta}{3\gamma} H \sum_{k,h,s} q_h^{*}(s)b_h^k(s)
                                                          +  \mathcal{O}\l( \eta\frac{H^{3} \logterm}{\gamma^{2}} \r)   
                                                                                                                \\
                                                       & \leq\frac{1}{3} \sum_{k,h,s} q_h^{*}(s)b_h^k(s)
                                                          +  \mathcal{O}\l( \eta\frac{H^{3} \logterm}{\gamma^{2}} \r)   ,
    \end{align*}
    where the second inequality if by $E^b$ and the last is since $\eta\leq\frac{\gamma}{2H}$.
    For the last term in \cref{eq:lemma reg init bound} we use $b^k_h(s) \le 3H$ and therefore $B^k_h(s,a) \le 3H^2$.
    Thus: 
    $
    	\eta\sum_{k,h,s,a} q_h^{*}(s)\pi_{h}^{k+d^k} (a\mid s)B_{h}^{k}(s,a)^{2} \leq 9 \eta H^{5}K .
    $
    Overall,
    \[
    	\textsc{Reg} 
    	\leq \frac{H\ln A}{\eta}
    	+ \frac{1}{3} \sum_{k,h,s} q_h^{*}(s)b_h^k(s)
    	+ \mathcal{O}\l( \eta H^{5}K 
    	+ \eta\frac{H^{3}\logterm}{\gamma^{2}} \r) . \qedhere
    \]
\end{proof}

\subsection{Bound on $\textsc{Drift}$}

\begin{lemma}
    \label{lemma:drift}
    If event $E^d$ holds then,
    $
	\textsc{Drift} \leq 
	                \mathcal{O}\l( \eta H^{5}(K+D)
	               + \frac{\eta H^{5}\dmax \logterm}{\gamma} \r)   .                                                         
    $
\end{lemma}

\begin{proof}
    We use \cref{lemma:l1 drift} and the fact that $|Q_{h}^{k}(s,a) - B_{h}^{k}(s,a)| \le 3H^2$ to obtain
    \begin{align*}
	\textsc{Drift} & =\sum_{k=1}^{K}\sum_{h,s,a}q_{h}^{*}(s)
	               (\pi_{h}^{k}(a\mid s) - \pi_{h}^{k+d^{k}}(a\mid s))(Q_{h}^{k}(s,a) - B_{h}^{k}(s,a))
	               \\
	               & \leq\sum_{k=1}^{K} \sum_{h,s,a}q_{h}^{*}(s)
	               |\pi_{h}^{k}(a\mid s) - \pi_{h}^{k+d^{k}}(a\mid s)|\cdot|Q_{h}^{k}(s,a) - B_{h}^{k}(s,a)| 
	               \\
	               & \leq 3 H^{2} \sum_{k=1}^{K} \sum_{h,s}q_{h}^{*}(s)
	               \Vert\pi_{h}^{k}(\cdot\mid s) - \pi_{h}^{k+d^{k}}(\cdot\mid s)\Vert_1
	               \\
	               & \leq \mathcal{O}\l( \eta H^{5}(K+D) + 
	               \frac{\eta H^{5}\dmax \logterm}{\gamma} \r). \qedhere
    \end{align*}
\end{proof}

\begin{lemma}
    \label{lemma:l1 drift}
    If event $E^d$ holds then,
    $
        \sum_{k=1}^{K}\Vert\pi_{h}^{k+d^{k}}(\cdot\mid s)-\pi_{h}^{k}(\cdot\mid s)\Vert_1\leq \mathcal{O}\l(  \eta H^{2}(K+D)+ \frac{ \eta H^2\dmax\logterm}{\gamma}\r).
    $
\end{lemma}

\begin{proof}
We first bound,
\begin{align*}
	\sum_{k=1}^{K}
	\Vert\pi_{h}^{k+d^{k}}(\cdot\mid s)-\pi_{h}^{k}(\cdot\mid s)\Vert_1 & \leq\sum_{k=1}^{K}\sum_{j=k}^{k+d^{k}-1}
                                                                        \Vert\pi_{h}^{j+1}(\cdot\mid s)-\pi_{h}^{j}(\cdot\mid s)\Vert_1
                                                                        \\
                                                                        & =\sum_{k=1}^{K}\sum_{j=k}^{k+d^{k}-1}\sum_{a}
                                                                        |\pi_{h}^{j+1}(a\mid s)-\pi_{h}^{j}(a\mid s)|
\end{align*}

Now, we apply \cref{lemma:elementwise diff} for each $j$ in the summation above with $\tilde \pi(\cdot) = \pi^{j+1}_h(\cdot \mid s)$, $ \pi(\cdot) = \pi^{j}_h(\cdot \mid s)$ and $\ell(\cdot) =  \sum_{i:i+d^i = j}{( \hat Q_h^i(s,\cdot) - B_h^i(s,\cdot) )} \geq - \sum_{i:i+d^i = j} 6H^2$ and observe that,

\begin{align*}
	\sum_{k=1}^{K}
	\Vert\pi_{h}^{k+d^{k}}(\cdot\mid s)-\pi_{h}^{k}(\cdot\mid s)\Vert_1 
                                                                        & \leq\eta\sum_{k=1}^{K}\sum_{j=k}^{k+d^{k}-1}\sum_{a}\pi_{h}^{j}(a\mid s)\sum_{i:i+d^{i}=j}
                                                                        (\hat{Q}_{h}^{i}(s,a)+6H^{2})
                                                                        \\
                                                                        & \quad+\eta\sum_{k=1}^{K}\sum_{j=k}^{k+d^{k}-1}\underbrace{\sum_{a}\pi_{h}^{j+1}(a\mid s)}_{=1}
                                                                        \sum_{a'}\pi_{h}^{j}(a'\mid s)\sum_{i:i+d^{i}=j}
                                                                        (\hat{Q}_{h}^{i}(s,a')+6H^{2}) 
                                                                        \\
                                                                        & =2\eta\sum_{k=1}^{K}\sum_{j=k}^{k+d^{k}-1}\sum_{a}\pi_{h}^{j}(a\mid s)\sum_{i:i+d^{i}=j}
                                                                        (\hat{Q}_{h}^{i}(a\mid s)+6H^{2})
                                                                        \\
                                                                        & =2\eta\sum_{k=1}^{K}\sum_{j=k}^{k+d^{k}-1}\sum_{i:i+d^{i}=j}\sum_{a}\pi_{h}^{j}
                                                                        (a\mid s)\hat{Q}_{h}^{i}(a\mid s)
                                                                        +12\eta H^{2}\sum_{k=1}^{K}\sum_{i=1}^{K}
                                                                            \mathbb{I}\{k\leq i+d^{i}<k+d^{k}\}
                                                                        \\
                                                                        & \leq 2\eta\sum_{k=1}^{K}\sum_{j=k}^{k+d^{k}-1}\sum_{i:i+d^{i}=j}\sum_{a}
                                                                        \pi_{h}^{j}(a\mid s)\hat{Q}_{h}^{i}(a\mid s)+12\eta H^{2}(D+K),                                     
\end{align*}

where the last inequality is by \cref{lemma:sum-delayed-indicators}.
For the first term we use event $E^d$:
\begin{align*}
	\sum_{k=1}^{K}\sum_{j=k}^{k+d^{k}-1}\sum_{i:i+d^{i}=j}\sum_{a}
	&\pi_{h}^{j}(a\mid s)\hat{Q}_{h}^{i}(a\mid s) =\sum_{k=1}^{K}\sum_{j=k}^{k+d^{k}-1}\sum_{i:i+d^{i}=j}\sum_{a}
	                                                \pi_{h}^{i+d^{i}}(a\mid s)\hat{Q}_{h}^{i}(a\mid s)
                                                   \\
                                                   & =\sum_{k=1}^{K}\sum_{i=1}^{K}\sum_{a}
                                                   \mathbb{I}\{k\leq i+d^{i}<k+d^{k}\}\pi_{h}^{i+d^{i}}(a\mid s)\hat{Q}_{h}^{i}(a\mid s)
                                                   \\
                                                   & \leq\sum_{k=1}^{K}\sum_{i=1}^{K}\sum_{a}
                                                   \mathbb{I}\{k\leq i+d^{i}<k+d^{k}\}\pi_{h}^{i+d^{i}}(a\mid s)
                                                   Q_{h}^{i}(a\mid s) + \mathcal{O}\l(\frac{H^2\dmax\logterm}{\gamma}\r) 
                                                   \\
                                                   & \leq H \sum_{k=1}^{K} \sum_{i=1}^{K}
                                                   \mathbb{I}\{k\leq i+d^{i}<k+d^{k}\} + \mathcal{O}\l(\frac{H^2\dmax\logterm}{\gamma}\r)
                                                   \\
                                                   & \leq H(D+K)
                                                   + \mathcal{O}\l(\frac{H^2\dmax\logterm}{\gamma}\r). \qedhere
\end{align*}
\end{proof}

\newpage

\section{Delay-Adapted Policy Optimization for (Tabular) Adversarial MDP with Unknown Transition}
\label{app:tabular unknown}
\begin{algorithm}
    \caption{Delay-Adapted Policy Optimization with Unknown Transition Function (Tabular)}  
    \label{alg:delayed-OPPO-unknown-p}
    \begin{algorithmic}
        \STATE \textbf{Input:} state space $\mathcal{S}$, action space $\mathcal{A}$, horizon $H$, learning rate $\eta > 0$, exploration parameter $\gamma > 0$, confidence parameter $\delta > 0$.
        
        \STATE \textbf{Initialization:} 
        Set $\pi_{h}^{1}(a \mid s) = \nicefrac{1}{A}$ for every $(s,a,h) \in \mathcal{S} \times \mathcal{A} \times [H]$.
        
        \FOR{$k=1,2,\dots,K$}
            
            \STATE Play episode $k$ with policy $\pi^k$ and observe trajectory $\{ (s^k_h,a^k_h) \}_{h=1}^H$.
            
            \STATE Compute visit counters for every $(h,s,a,s') \in [H] \times \mathcal{S} \times \mathcal{A} \times \calS$:
            $$n^k_h(s,a,s') = \sum_{j:j+d^j < k} \bbI \{ s^j_h=s,a^j_h=a,s^j_{h+1}=s'\} \qquad ; \qquad n^k_h(s,a) = {\sum_{j:j+d^j < k} \bbI \{ s^j_h=s,a^j_h=a\} }.$$
            
            \STATE Compute empirical transition function $\bar p^k_h(s' \mid s,a) = \frac{n^k_h(s,a,s')}{\max \{ n^k_h(s,a), 1 \}}$ and confidence set $\calP^k = \{ \calP^k_h(s,a) \}_{s,a,h}$ such that $p'_h(\cdot \mid s,a) \in \calP^k_h(s,a)$ if and only if $\sum_{s'} p'_h(s'\mid s,a) = 1$ and for every $s' \in \calS$:
            \[
                |p'_h(s'\mid s,a) - \bar p^k_h(s'\mid s,a)| \le 4 \sqrt{\frac{\bar p^k_h(s'\mid s,a) \log \frac{10 H S A K}{\delta}}{n^k_h(s,a) \vee 1}} + \frac{10 
                \log \frac{10 H S A K}{\delta}}{n^k_h(s,a) \vee 1}.
            \]
            
            \STATE Compute occupancy measures $\ol q^k_h(s) =\max_{p' \in \calP^k} q^{\pi^k,p'}_h(s)$ and $\ul q^k_h(s) = \min_{p' \in \calP^k} q^{\pi^k,p'}_h(s)$.
            
            \STATE {\color{gray} \# Policy Evaluation}
            
            \FOR{$j$ such that $j + d^j = k$}
            
                \STATE Observe bandit feedback $\{ c^j_h(s^j_h,a^j_h) \}_{h=1}^H$ and set $B^j_{H+1}(s,a) = 0$ for every $(s,a) \in \calS \times \calA$.
                
                \FOR{$h=H,H-1,\dots,1$}
                    
                    \FOR{$(s,a) \in \calS \times \calA$}
                    
                        \STATE Compute $r^j_h(s,a) = \frac{\pi^j_h(a \mid s)}{\max \{ \pi^j_h(a\mid s) , \pi_{h}^{k}(a \mid s) \}}$ and $L^j_h = \sum_{h'=h}^H c^j_{h'}(s^j_{h'},a^j_{h'})$.
                    
                        \STATE Compute $\tilde{b}^j_h(s) = \sum_{a \in \calA} \frac{3 \gamma H \pi^k_h(a \mid s) r^j_h(s,a)}{\bar{q}^j_h(s) \pi^j_h(a \mid s) + \gamma}$ and $\bar b_{h}^{j}(s) = \sum_{a\in\calA}\frac{2H\pi_{h}^{k}(a \mid s)r^j_h(s,a)(\ol q^{j}_h(s,a) - \ul q^{j}_h(s,a) )}{\ol q_{h}^{j}(s)\pi^j_h(a\mid s) + \gamma}$.
                        
                        \STATE Compute $b_h^j(s) = \tilde b_h^j(s) + \bar b_h^{j}(s)$ and $\hat{Q}_{h}^{j}(s,a) = r^j_h(s,a) \frac{ \mathbb{I}\{s_{h}^{j}=s,a_{h}^{j}=a\} L^j_h }{ \bar{q}_{h}^{j}(s)\pi^j_h(a\mid s) +\gamma}$.
                        
                        \STATE Compute $B_{h}^{j}(s,a) =b_{h}^{j}(s) + \max_{p' \in \calP^j_h(s,a)} \sum_{s' \in \calS,a' \in \calA} p'_h(s'\mid s,a) \pi_{h+1}^{j}(a' \mid s') B^j_{h+1}(s',a')$.
                        
                    \ENDFOR
                
                \ENDFOR
            \ENDFOR
            
            \STATE {\color{gray} \# Policy Improvement}
            
            \STATE Define the policy $\pi^{k+1}$ for every $(s,a,h) \in \calS \times \calA \times [H]$ by:
            \[
                \pi^{k+1}_h(a \mid s) = \frac{\pi^k_h(a \mid s) \exp \l(-\eta \sum_{j: j+d^j = k}  ( \hat Q_{h}^{j}(s,a) - B_h^j(s,a) ) \r)} 
                {\sum_{a' \in \mathcal{A}} \pi^k_h(a' \mid s) \exp \l(-\eta \sum_{j: j+d^j = k}  ( \hat Q_{h}^{j}(s,a') - B_h^j(s,a') ) \r)}.
            \]
            
        \ENDFOR
        
    \end{algorithmic}
\end{algorithm}

\begin{theorem}
    \label{thm:regret-bound-unknown-dynamics}
    Set $\eta = H \l( H^2 SAK+H^{4}(K+D) \r)^{-1/2}$ and $\gamma = 2 \eta H$.
    Running \cref{alg:delayed-OPPO-unknown-p} in an adversarial MDP $\calM = (\calS, \calA, H, p, \{ c^k \}_{k=1}^K)$ with unknown transition function and delays $\{ d^k \}_{k=1}^K$ guarantees, with probability at least $1 - \delta$,
    \[
        \regret
        \le
            \mathcal{O}
            \l( H^{3}S \sqrt{AK} \log \frac{HSAK}{\delta} 
            + H^{3}\sqrt{K + D} \log \frac{HSAK}{\delta} 
            + H^4 S^2 A \dmax
            + H^4 S^3 A \log^2 \frac{HSAK}{\delta} \r).
    \]
\end{theorem}

\begin{remark}[Dependence in $\dmax$]
\label{remark:dmax-dependence}
    All our regret bounds contain additive terms that scale linearly with $\dmax$.
    While these are low-order terms when $\dmax$ is smaller than $\sqrt{D}$, they may become dominant for large maximal delay.
    The dependence in $\dmax$ can be removed altogether using the \textit{skipping} technique \citep{thune2019nonstochastic,bistritz2019online,lancewicki2020learning}, i.e., ignoring episodes that their delay is larger than some threshold $\beta$.
    In the case of known transitions (\cref{thm:regret-bound-known-dynamics} in \cref{app:tabular known}), we can set $\beta = \sqrt{D}/H$ and remove the dependence in $\dmax$ without hurting our original regret bound, i.e., we get the bound $\regret = \tilde \calO(H^2 \sqrt{SAK} + H^3 \sqrt{K+D})$.
    However, in the case of unknown transitions (\cref{thm:regret-bound-unknown-dynamics} in \cref{app:tabular unknown}), we get a slightly worse regret bound.
    Specifically, we can set $\beta = \sqrt{\frac{D}{H^2 S^2 A}}$ and obtain the regret $\regret = \tilde \calO(H^3 S \sqrt{A (K+D)})$.
    \citet{jin2022near} encounter the same issue in their regret bounds, so it remains an open problem whether the dependence in $\dmax$ can be removed without hurting the original regret bound in the unknown transitions case.
\end{remark}

\subsection{The good event}
\label{sec:good event unknown}

Let $\logterm = 10 \log \frac{10 KHSA}{\delta}$, $\tilde{\mathcal{H}}^{k}$ be the history of episodes $\{j:j+d^{j}<k\}$, and $\mathcal{H}^{k}$ be the history of episodes $\{j:j<k\}$.
Define the following events: 
\begin{align*}
    E^{p}
    & = 
    \l\{ \forall (k,s',s,a,h). \ 
    \l| p_{h}(s'\mid s,a)-\bar{p}_{h}^{k}(s'\mid s,a) \r|
    \leq 
    4\sqrt{ \frac{\bar{p}_{h}^{k}(s' \mid s,a)  \log\frac{10 HSAK}{\delta}}{\max \{ n_{h}^{k}(s,a), 1 \}}} + 10 \frac{\log\frac{10 HSAK}{\delta}}{\max \{ n_{h}^{k}(s,a), 1 \}}
    \r\}
    \\
    E^{est}
    & = 
    \l\{ \sum_{h,s,a,k}|\overline q_{h}^{k}(s,a) - \underline q_{h}^{k}(s,a)|
        \le \calO \l( \sqrt{ H^{4} S^{2} A K \log \frac{10 KHSA}{\delta}} + H^3 S^{3} A \log^2 \frac{10 KHSA}{\delta} + H^3 S^2 A d_{max} \r)
    \r\}
    \\
    E^d
    & =
    \l\{ \forall (h,s). \  \sum_{k=1}^{K}\sum_{j=1}^{K}\sum_{a}\mathbb{I}\{j\leq k+d^{k}<j+d^{j}\}\pi_{h}^{k+d^{k}}(a\mid s)(\hat{Q}_{h}^{k}(s,a)-Q_{h}^{k}(s,a))\le\frac{10 H^2 \dmax\log\frac{10 H S}{\delta}}{\gamma}  \r\}
    \\
    E^*
    & =
    \l\{ \sum_{k=1}^K \sum_{h,s,a} q^*_h(s,a) (\hat{Q}_{h}^{k}(s, a) - Q_{h}^{k}(s, a)) \le \frac{H^2  \log \frac{10 H S A}{\delta}}{\gamma} \r\}
    \\
    E^b
    & = 
    \l\{
        \sum_{k=1}^{K}\sum_{h,s,a} \frac{q_{h}^{*}(s)\pi_{h}^{k+d^{k}}(a\mid s) r^k_h(s,a) \mathbb{I}\{s_{h}^{k}=s,a_{h}^{k}=a\}}{(\ol q_{h}^{k}(s,a) + \gamma)^{2}}
        - \sum_{k=1}^{K}\sum_{h,s,a} \frac{ q_{h}^{*}(s)\pi_{h}^{k+d^{k}}(a\mid s) r^k_h(s,a)}{\ol q_{h}^{k}(s,a) + \gamma}\leq\frac{H}{\gamma^{2}}\ln\frac{10 H}{\delta}
    \r\}
    \\
    E^f
    & = 
    \vast\{
        \sum_{k=1}^K \bbE_k \l[\sum_{h,s} q_{h}^{*}(s) \l\langle \pi_{h}^{k + d^k}(\cdot \mid s) , \hat{Q}_{h}^{k}(s,\cdot) \r\rangle \r]
            - \sum_{k=1}^K \sum_{h,s} q_{h}^{*}(s) \l\langle \pi_{h}^{k + d^k}(\cdot \mid s),\hat{Q}_{h}^{k}(s,\cdot) \r\rangle 
        \\  & \qquad \qquad \qquad \qquad \qquad \qquad \qquad \qquad \qquad \qquad \qquad \qquad \qquad \qquad \qquad \leq \frac{1}{3} \sum_{k=1}^{K} \sum_{h,s} q_{h}^{*}(s) \tilde b_{h}^{k}(s) 
            + \frac{H^{2}}{\gamma} \ln \frac{10}{\delta}
    \vast\}
\end{align*}

The good event is the intersection of the above events. 
The following lemma establishes that the good event holds with high probability. 

\begin{lemma}[The Good Event]
    \label{lemma:good-event-unknown-dynamics}
    Let $\bbG = E^p \cap E^{est} \cap E^d \cap E^* \cap E^b \cap E^f$ be the good event. 
    It holds that $\Pr [ \bbG ] \geq 1-\delta$.
    Moreover, under the good event, it holds that $p \in \calP^k$ and $\ul q^k_h(s,a) \le q^k_h(s,a) \le \ol q^k_h(s,a)$ for every $(k,h,s,a) \in [K] \times [H] \times \calS \times \calA$.
\end{lemma}

\begin{proof}
    We'll show that each of the events $\lnot E^p,\lnot E^d,  \lnot E^* , \lnot E^b, \lnot E^f$ holds with probability of at most $\delta/5$ and so by the union bound $\Pr [ \bbG ] \geq 1-\delta$.
    
    \textbf{Event $E^p$:} Holds with probability $1 - \delta/10$ by standard Bernstein inequality (see, e.g., Lemma 2 in \citet{jin2019learning}).
    As a consequence of event $E^p$, $p\in \calP^k$ for all $k$. In particular, $\ol q^k_h(s,a) \geq q^k_h(s,a)$ for all $k,h,s$ and $a$.
    
    \textbf{Event $E^{est}$:} Holds with probability $1 - \delta/10$ by \citet[Lemma D.12]{jin2022near} (see also \citep[Lemma 4]{jin2019learning}) which is a standard techniques (adapted to delays) of summing the the confidence radius on the trajectory.
    
    \textbf{Event $E^d$:} We show the proof under event $E^p$ which occurs with probability $1 - \delta/10$. Fix $s$ and $h$. Similar to the proof of \cref{lemma:good-event-known-dynamics}, we apply \cref{lemma:concentration}.
    For every $(h',s',a$), set $\tilde q_{h'}^k(s',a) = \ol q_{h'}^k(s',a)$ and
    \begin{align*}
        z_{h'}^{k}(s',a)
        & =
        \mathbb{I} \l\{ s'=s,h'=h \r\} \sum_{j=1}^{K} \mathbb{I}\{j\leq k+d^{k}<j+d^{j}\}\pi_{h}^{k+d^{k}}(a\mid s) r^k_h(s,a) Q_{h}^{j}(s,a)
        \\
        Z_{h'}^{k}(s',a)
        & =
        \mathbb{I} \l\{ s'=s,h'=h \r\} \sum_{j=1}^{K} \mathbb{I}\{j\leq k+d^{k}<j+d^{j}\} \pi_{h}^{k+d^{k}}(a\mid s) r^k_h(s,a) M_h^k(s,a)
        \\
        M_h^k(s,a) &=  \bbI\{s_{h}^{k}=s,a_{h}^{k}=a\} \sum_{\tilde{h}=1}^{H} c_{\tilde{h}}^{k}(s_{\tilde{h}}^{k} , a_{\tilde{h}}^{k}) 
        + (1-\bbI\{s_{h}^{k}=s,a_{h}^{k}=a\}) Q_{h}^{k}(s,a).
    \end{align*}
    Note that $\bbE_k [Z_{h'}^k(s',a)] = z_{h'}^k(s',a)$ and
    \begin{align*}
        \sum_{h',s',a} \frac{\mathbb{I} \{s_{h}^{k}=s,a_{h}^{k}=s\}Z_{h}^{k}(s,a)}{\tilde q_{h}^{k}(s,a) + \gamma} 
        & = 
        \sum_{j=1}^{K}\sum_{a}\mathbb{I}\{j\leq k+d^{k}<j+d^{j}\}\pi_{h}^{k+d^{k}}(a\mid s)\hat{Q}_{h}^{k}(s,a)
        \\
        \sum_{h',s',a} \frac{q_{h}^{k}(s,a)z_{h}^{k}(s,a)}{\tilde{q}_{h}^{k}(s,a)} 
        & \leq
        \sum_{j=1}^{K}\sum_{a}\mathbb{I}\{j\leq k+d^{k}<j+d^{j}\}\pi_{h}^{k+d^{k}}(a\mid s) Q_{h}^{k}(s,a),
    \end{align*}
    where in the inequality we used the fact that $\tilde q_h^k(s,a) \leq q_h^k(s,a)$ under the event $E^p$.
    Finally, we use \cref{lemma:concentration} with $R \leq 2H \dmax$ as in the proof of \cref{lemma:good-event-known-dynamics}. Thus, the event holds for $(s,h)$ with probability $1 - \frac{\delta}{10 H S}$. By taking the union bound over all $h$ and $s$, $E^d$ holds with probability $1 - \frac{\delta}{10}$.
    
    \textbf{Event $E^*$ (Lemma C.2 of \citet{luo2021policy}):} We show the proof under event $E^p$ which occurs with probability $1 - \delta/10$. Again, $E^*$ holds with probability of at least $1-\delta/10$ by applying \cref{lemma:concentration}  with $\tilde q_{h'}^k(s',a) = \ol q_{h'}^k(s',a)$, $z_h^k(s,a) = q^{*}_h(s,a) r^k_h(s,a) Q_h^k(s,a)$ and,
    \[
        Z_{h}^{k}(s,a)=q_{h}^{*}(s,a) r^k_h(s,a)
        \l( \bbI\{s_{h}^{k}=s,a_{h}^{k}=a\} \sum_{h'=1}^{H}c_{h'}^{k}(s_{h'}^{k},a_{h'}^{k}) + (1-\bbI\{s_{h}^{k}=s,a_{h}^{k}=a\})Q_{h}^{k}(s,a) \r).
    \]
    Note that $R \leq H$, $\frac{\mathbb{I} \{s_{h}^{k}=s,a_{h}^{k}=s\} Z_{h}^{k}(s,a)}{\tilde q_{h}^{k}(s,a) + \gamma} = q^{*}_h(s,a) \hat Q_h^k(s,a)$ and $\frac{q_{h}^{k}(s,a)z_{h}^{k}(s,a)}{\tilde{q}_{h}^{k}(s,a)} \leq q^{*}_h(s,a)Q_h^k(s,a)$ where similar to before, the inequality holds under event $E^p$.
    
    \textbf{Event $E^b$:} We show the proof under event $E^p$ which occurs with probability $1 - \delta/10$. Similar to before, $E^b$ holds with probability of at least $1 - \delta / 10$ by applying \cref{lemma:concentration} with $\tilde q_{h'}^k(s',a) = \ol q_{h'}^k(s',a)$ and $z_h^k(s,a) = Z_h^k(s,a) = \frac{q_{h}^{*}(s) \pi_{h}^{k+d^{k}}(a\mid s) r^k_h(s,a)}{\ol q_{h}^{k}(s,a)  + \gamma}$.

    \textbf{Event $E^f$:} We show the proof under event $E^p$ which occurs with probability $1 - \delta/10$.
    Let $Y_{k}= \sum_{h,s} q_{h}^{*}(s) \l\langle \pi_{h}^{k + d^k}(\cdot \mid s),\hat{Q}_{h}^{k}(s,\cdot) \r\rangle $.
    Similar to the proof of \cref{lemma:good-event-known-dynamics}, we use a variant of Freedman's inequality (\cref{lemma:freedman}) to bound $\sum_{k=1}^{K} \bbE_k[Y_{k}   ] - \sum_{k=1}^{K}Y_{k} $.
    \begin{align*}
    \bbE_k \l[ Y_{k}^2  \r] & = \bbE_k \l[ \l( \sum_{h,s,a} q_{h}^{*}(s)\pi_{h}^{k + d^k}(a \mid s)\hat{Q}_{h}^{k}(s,a) \r)^{2}  \r] \\
    & \leq\bbE_k \l[ \l( \sum_{h,s,a} q_{h}^{*}(s)\pi_{h}^{k + d^k}(a \mid s) \r) \l( \sum_{h,s,a} q_{h}^{*}(s)\pi_{h}^{k + d^k}(a \mid s)(\hat{Q}_{h}^{k}(s,a))^{2} \r)  \r]\\
    & \leq H\bbE_k \l[ \sum_{h,s,a} q_{h}^{*}(s)\pi_{h}^{k + d^k}(a \mid s)\frac{r^k_h(s,a)^2(L_{h}^{k})^{2} \mathbb{I} \{s_{h}^{k}=s,a_{h}^{k}=a\}}{(\ol q_{h}^{k}(s,a) + \gamma)^{2}} \r]
    \\
    & \le
    H^3 \sum_{h,s,a} q_{h}^{*}(s)\pi_{h}^{k + d^k}(a \mid s) \frac{r^k_h(s,a) q^k_h(s,a)}{(\ol q_{h}^{k}(s,a) + \gamma)^{2}}
    \leq H^{3} \sum_{h,s,a} \frac{q_{h}^{*}(s)\pi_{h}^{k + d^k}(a \mid s) r^k_h(s,a)}{\ol q_{h}^{k}(s,a) + \gamma},
    \end{align*}
    where the first inequality is Cauchy-Schwartz inequality and the last inequality holds under event $E^p$.
    Also, $|Y_{k}|\leq\frac{H^{2}}{\gamma}$. Therefore by \cref{lemma:freedman} with probability $1 - \delta/10$,
    \begin{align*}
        \sum_{k=1}^{K}  \bbE_k [Y_{k}  ] - \sum_{k=1}^{K}Y_{k} 
        \leq
        \gamma H\sum_{k=1}^{K} \l(\sum_{h,s,a} \frac{q_{h}^{*}(s)\pi_{h}^{k + d^{k}}(a\mid s) r^k_h(s,a)}{\ol q_{h}^{k}(s,a) + \gamma} \r) 
        + \frac{H^{2}}{\gamma} \ln \frac{10}{\delta}
        = \frac{1}{3} \sum_{k=1}^{K} \sum_{h,s} q_{h}^{*}(s)\tilde b_{h}^{k}(s) 
        + \frac{H^{2}}{\gamma} \ln \frac{10}{\delta}. \qedhere
    \end{align*}
\end{proof}

\subsection{Proof of the main theorem}

\begin{proof}[Proof of \cref{thm:regret-bound-unknown-dynamics}]
By \cref{lemma:good-event-unknown-dynamics}, the good event holds with probability of at least $1 - \delta$.
As in the previous section, we analyze the regret under the assumption that the good event holds.
We start with the following regret decomposition,
\begin{align*}
    \regret  = & \sum_{k=1}^{K} \sum_{h,s}q_{h}^{*}(s) \l\langle \pi_{h}^{k} (\cdot \mid s) - \pi_{h}^{*} (\cdot \mid s),Q_{h}^{k}(s, \cdot) \r\rangle 
    = \underbrace{
    \sum_{k=1}^{K} \sum_{h,s}q_{h}^{*}(s) \l\langle \pi_{h}^{k + d^{k}} (\cdot \mid s),Q_{h}^{k}(s, \cdot) - \hat{Q}_{h}^{k}(s, \cdot) \r\rangle }
    _{\textsc{Bias}_{1}} 
    \\
    & + \underbrace{
    \sum_{k=1}^{K} \sum_{h,s}q_{h}^{*}(s) \l\langle \pi_{h}^{*} (\cdot \mid s), \hat{Q}_{h}^{k}(s, \cdot) - Q_{h}^{k}(s, \cdot) \r\rangle 
    }_{\textsc{Bias}_{2}} 
    + \underbrace{ 
    \sum_{k=1}^{K} \sum_{h,s}q_{h}^{*}(s) \l\langle \pi_{h}^{k} (\cdot \mid s) - \pi_{h}^{*} (\cdot \mid s),B_{h}^{k}(s, \cdot) \r\rangle 
    }_{\textsc{Bonus}} 
    \\
    & + \underbrace{ 
    \sum_{k=1}^{K} \sum_{h,s}q_{h}^{*}(s) \l\langle \pi_{h}^{k + d^{k}} (\cdot \mid s) - \pi_{h}^{*} (\cdot \mid s), \hat{Q}_{h}^{k}(s, \cdot) - B_{h}^{k}(s, \cdot) \r\rangle 
    }_{\textsc{Reg}} 
    + \underbrace{ 
    \sum_{k=1}^{K} \sum_{h,s}q_{h}^{*}(s) \l\langle \pi_{h}^{k} (\cdot \mid s) - \pi_{h}^{k + d^{k}} (\cdot \mid s),Q_{h}^{k}(s, \cdot) - B_{h}^{k}(s, \cdot) \r\rangle 
    }_{\textsc{Drift}},
\end{align*}
where the first is by \cref{lemma: value diff}.
$\textsc{Bias}_2$ is bounded under event $E^*$ by $\mathcal{O}\l( \frac{H^2 \logterm}{\gamma} \r)$, and $\textsc{Drift} \leq \mathcal{O}\l( \eta H^{5}(K+D)
	               +\frac{\eta H^{5}\dmax}{\gamma} \logterm \r)  $ by \cref{lemma:drift}. 
The other three term are bounded in \cref{lemma: bias1-unknown,lemma:bounus unknown,lemma:reg-unknown}.
Overall,
\begin{align*}
	\regret & \leq\underbrace{
            	\frac{2}{3}\sum_{k=1}^{K}\sum_{h,s}q_{h}^{*}(s) \tilde b_{h}^{k}(s)
            	+\sum_{k=1}^{K}\sum_{h,s}q_{h}^{*}(s) \bar b_{h}^{k}(s)
            	+\mathcal{O}\l( \eta H^{4}(K+D)
	                  + \frac{ \eta H^{4}\dmax\logterm}{\gamma}
	                  + \frac{H^{2}}{\gamma}\logterm \r)
            	}_{\textsc{Bais}_{1}} 
            	\\
    	      & \qquad+\underbrace{\mathcal{O}\l( \frac{H^{2}}{\gamma}\logterm \r)
                }_{\textsc{Bais}_{2}}
	            +\underbrace{
	            \mathcal{O}\l( \gamma H^{2}SAK
            + H^3 S \sqrt{AK} \logterm
            + H^4 S^3 A \logterm^2 + H^4 S^2 A \dmax \r)
	            - \sum_{k=1}^{K}\sum_{h,s}q_{h}^{*}(s)b_{h}^{k}(s)
	            }_{\textsc{Bonus}}
	            \\
              & \qquad+\underbrace{
                \frac{H\ln A}{\eta}
    	+ \frac{1}{3} \sum_{k,h,s} q_h^{*}(s)\tilde b_h^k(s)
    	+ \mathcal{O}\l( \eta H^{5}K 
    	+ \eta\frac{H^{3} \logterm}{\gamma^{2}} \r)
                }_{\textsc{Reg}}
              \\
	          & \qquad + \underbrace{
                \mathcal{O}\l( \eta H^{5}(K+D)
	               +\frac{\eta H^{5}\dmax}{\gamma} \logterm \r)
                }_{\textsc{Drift}}
                \\
	      & \leq 
	            \mathcal{O}\l( \frac{H\ln A}{\eta} 
	            + \gamma H^2 SA K 
	            + \eta H^{5}(K+D)
	            + H^3 S \sqrt{AK}  \logterm
	            +\eta\frac{H^{3}\logterm}{\gamma^{2}}
	            +\frac{H^{2}}{\gamma}\logterm
	            +\frac{\eta H^{5}\dmax \logterm}{\gamma}
	            + H^4 S^2 A \dmax + H^4 S^3 A \logterm^2 \r).
\end{align*}
For $\eta=\frac{1}{\sqrt{H^{2}SAK+H^{4}(K+D)}}$ and $\gamma=2\eta H$ we get:
\[
	\regret \leq \mathcal{O}\l( H^{3} S \sqrt{AK}\logterm+H^{3}\sqrt{K+D}\logterm
	 + H^4 S^2 A \dmax + H^4 S^3 A \logterm^2 \r). \qedhere
\]
\end{proof}

\subsection{Bound on $\textsc{Bias}_1$}
\begin{lemma}
    \label{lemma: bias1-unknown}
    Under the good event,
    \[
	\textsc{Bias}_{1} \leq \frac{2}{3}\sum_{k=1}^{K}\sum_{h,s}
	                  q_{h}^{*}(s) \tilde b_{h}^{k}(s)
	                  + \sum_{k=1}^{K} \sum_{h,s}q_{h}^{*}(s) \bar b_{h}^{k}(s) + \mathcal{O}\l( \eta H^{4}(K+D)
	                  + \frac{ \eta H^{4}\dmax\logterm}{\gamma}
	                  + \frac{H^{2}}{\gamma}\logterm \r).                                              
    \]
\end{lemma}

\begin{proof}
    Let $Y_{k}= \sum_{h,s} q_{h}^{*}(s) \l\langle \pi_{h}^{k + d^k}(\cdot \mid s),\hat{Q}_{h}^{k}(s,\cdot) \r\rangle $. It holds that
    \begin{align*}
        \textsc{Bias}_1 
        = 
        \sum_{k=1}^{K} \sum_{h,s} q_{h}^{*}(s) \l\langle \pi_{h}^{k + d^k}(\cdot \mid s),Q_{h}^{k}(s,\cdot) \r\rangle 
        - \sum_{k=1}^{K} \bbE_k [Y_{k} ] 
        + \sum_{k=1}^{K} \bbE_k [Y_{k} ] 
        - \sum_{k=1}^{K}Y_{k}.
    \end{align*}
Under event $E^f$ it holds that
$
    \sum_{k=1}^{K} \bbE_k [Y_{k} ] - \sum_{k=1}^{K}Y_{k} 
    \leq \frac{1}{3} \sum_{k=1}^{K} \sum_{h,s} q_{h}^{*}(s) \tilde b_{h}^{k}(s) 
    + \frac{H^{2}}{\gamma} \ln \frac{10}{\delta}.
$
In addition,
\begin{align*}
	\sum_{k=1}^{K} & \sum_{h,s}q_{h}^{*}(s) \l\langle\pi_{h}^{k + d^{k}}(\cdot\mid s),Q_{h}^{k}(s,\cdot) \r\rangle - \sum_{k=1}^{K} \bbE_k [Y_{k} ]=
                                        \\
                                        & =\sum_{k=1}^{K} \sum_{h,s,a}q_{h}^{*}(s) \pi_{h}^{k + d^{k}}(a\mid s)Q_{h}^{k}(s,a) 
                                        \l(1 - \frac{q_{h}^{k}(s,a) r^k_h(s,a)}
                                        {\ol q_{h}^{k}(s,a) + \gamma} \r)
                                        \\
                                        & =\sum_{k=1}^{K} \sum_{h,s,a}q_{h}^{*}(s) \pi_{h}^{k + d^{k}}(a\mid s)Q_{h}^{k}(s,a) 
                                        \l(1 - \frac{(q_{h}^{k}(s,a) + \gamma + \ol q_{h}^{k}(s,a) - \ul q_{h}^{k}(s,a)) r^k_h(s,a)}
                                        {\ol q_{h}^{k}(s,a) + \gamma} \r) 
                                        \\
                                        & \qquad + \sum_{k=1}^{K} \sum_{h,s,a}q_{h}^{*}(s) \pi_{h}^{k + d^{k}}(a\mid s)Q_{h}^{k}(s,a) \frac{\gamma r^k_h(s,a)}
                                        {\ol q_{h}^{k}(s,a) + \gamma} + \sum_{k=1}^{K} \sum_{h,s,a}q_{h}^{*}(s) \pi_{h}^{k + d^{k}}(a\mid s)Q_{h}^{k}(s,a) \frac{ r^k_h(s,a) (\ol q_{h}^{k}(s,a) - \ul q_{h}^{k}(s,a))}
                                        {\ol q_{h}^{k}(s,a) + \gamma}
                                        \\
                                        & =\underbrace{
                                            \sum_{k=1}^{K} \sum_{h,s,a}q_{h}^{*}(s) \pi_{h}^{k + d^{k}}(a\mid s)Q_{h}^{k}(s,a) 
                                            \l(\frac{(\ol q_{h}^{k}(s,a) + \gamma) (1 - r^k_h(s,a))}
                                            {\ol q_{h}^{k}(s,a) + \gamma} \r)
                                            }_{(i)} 
                                        \\
                                        &  \qquad + \underbrace{
                                            \sum_{k=1}^{K} \sum_{h,s,a}q_{h}^{*}(s) \pi_{h}^{k + d^{k}}(a\mid s)Q_{h}^{k}(s,a) 
                                            \l(\frac{r^k_h(s,a) (q_{h}^{k}(s,a) - \ul q_{h}^{k}(s))}
                                            {\ol q_{h}^{k}(s,a) + \gamma} \r)
                                            }_{(ii)}
                                        \\
                                        &  \qquad+ \underbrace{
                                            \sum_{k=1}^{K} \sum_{h,s,a}q_{h}^{*}(s) \pi_{h}^{k + d^{k}}(a\mid s)Q_{h}^{k}(s,a) \frac{\gamma r^k_h(s,a)}
                                        {\ol q_{h}^{k}(s,a) + \gamma}
                                            }_{(iii)}
                                        \\
                                        &  \qquad+ \underbrace{
                                            \sum_{k=1}^{K} \sum_{h,s,a}q_{h}^{*}(s) \pi_{h}^{k + d^{k}}(a\mid s)Q_{h}^{k}(s,a) \frac{ r^k_h(s,a) (\ol q_{h}^{k}(s,a) - \ul q_{h}^{k}(s,a))}
                                        {\ol q_{h}^{k}(s,a) + \gamma}
                                            }_{(iv)}.
\end{align*}
Terms $(i)$ and $(iii)$ are bounded similarly to the known dynamics:
case:
\begin{align*}
	(i) & \leq H\sum_{k=1}^{K} \sum_{h,s,a}q_{h}^{*}(s) \l(\max\{\pi_{h}^{k}(a\mid s),\pi_{h}^{k + d^{k}}(a\mid s) \} - \pi_{h}^{k}(a\mid s) \r) \leq H\sum_{k=1}^{K} \sum_{h,s}q_{h}^{*}(s) \Vert\pi_{h}^{k + d^{k}}(\cdot\mid s) - \pi_{h}^{k}(\cdot\mid s) \Vert_1
	\\
	(iii) & \leq\sum_{k=1}^{K} \sum_{h,s,a}q_{h}^{*}(s) 
	        \frac{\pi_{h}^{k + d^{k}}(a\mid s) r^k_h(s,a) \gamma H}{\ol q_{h}^{k}(s,a) + \gamma} 
	        =\frac{1}{3} \sum_{k=1}^{K} \sum_{h,s}q_{h}^{*}(s) \tilde{b}_{h}^{k}(s).
\end{align*}
Finally,
\begin{align*}
	(ii) + (iv) \leq \sum_{k=1}^{K} \sum_{h,s,a}q_{h}^{*}(s) 
	        \frac{2H \pi_{h}^{k + d^{k}}(a\mid s) r^k_h(s,a)(\ol q_{h}^{k}(s,a) - \ul q_{h}^{k}(s,a))}{\ol q_{h}^{k}(s,a) + \gamma}= \sum_{k=1}^{K} \sum_{h,s}q_{h}^{*}(s) \bar{b}_{h}^{k}(s).                        
\end{align*}
In total, using \cref{lemma:l1 drift},
\begin{align*}
	\textsc{Bias}_{1} & \leq\frac{2}{3} \sum_{k=1}^{K} \sum_{h,s}q_{h}^{*}(s) \tilde b_{h}^{k}(s)  + \sum_{k=1}^{K} \sum_{h,s}q_{h}^{*}(s) \bar b_{h}^{k}(s)  +  H\sum_{h,s} \sum_{k=1}^{K}
	                    q_{h}^{*}(s) \Vert\pi_{h}^{k + d^{k}}(\cdot\mid s) - \pi_{h}^{k}(\cdot\mid s) \Vert_1 + \frac{H^{2}}{\gamma} \logterm 
	                    \\
	                  & \leq\frac{2}{3} \sum_{k=1}^{K} \sum_{h,s}q_{h}^{*}(s) \tilde b_{h}^{k}(s)  
	                  + \sum_{k=1}^{K} \sum_{h,s}q_{h}^{*}(s) \bar b_{h}^{k}(s)  
	                  + \mathcal{O}\l( \eta H^{4}(K+D)
	                  + \frac{ \eta H^{4}\dmax\logterm}{\gamma}
	                  + \frac{H^{2}}{\gamma}\logterm \r). \qedhere                    
\end{align*}
\end{proof}

\subsection{Bound on $\textsc{Bonus}$}

\begin{lemma}
    \label{lemma:bounus unknown}
    It holds that
    $$
        \textsc{Bonus} \le
        \mathcal{O}\l( \gamma H^{2}SAK
            + H^3 S \sqrt{AK} \logterm
            + H^4 S^3 A \logterm^2 + H^4 S^2 A \dmax \r)
            - \sum_{k=1}^K \sum_{h,s} q_{h}^{*}(s) b_{h}^{k}(s).
    $$
\end{lemma}

\begin{proof}
    Let $\hat p^k$ be the transition function chosen by the algorithm when calculating $B^k$. It holds that
    \begin{align*}
        \sum_{h,s} & q_{h}^{*}(s)\l\langle \pi_{h}^{k}(\cdot\mid s)-\pi_{h}^{*}(\cdot\mid s), B_{h}^{k}(s,\cdot)\r\rangle
        =
                    \\
                    & =
                    \sum_{h,s,a} q_{h}^{*}(s) \pi_{h}^{k}(a \mid s) B_{h}^{k}(s, a) - \sum_{h,s,a} q_{h}^{*}(s) \pi_{h}^{*}(a \mid s) B_{h}^{k}(s, a)
                    \\
                    & =
                    \sum_{h,s,a} q_{h}^{*}(s) \pi_{h}^{k}(a \mid s) B_{h}^{k}(s, a) 
                    - 
                    \sum_{h,s,a} q_{h}^{*}(s) \pi_{h}^{*}(a \mid s) \l( b_{h}^{k}(s) +  \sum_{s',a'}  \hat p^k_h(s'\mid s,a) \pi^k_{h+1}(a' \mid s') B^k_{h+1}(s',a') \r)
                    \\
                    & \leq
                    \sum_{h,s,a} q_{h}^{*}(s) \pi_{h}^{k}(a \mid s) B_{h}^{k}(s, a) 
                    - 
                    \sum_{h,s,a} q_{h}^{*}(s) \pi_{h}^{*}(a \mid s) \l( b_{h}^{k}(s) + \sum_{s',a'}  p_h(s'\mid s,a) \pi^k_{h+1}(a' \mid s') B^k_{h+1}(s',a') \r)
                    \\
                    & =
                    \sum_{h,s,a} q_{h}^{*}(s) \pi_{h}^{k}(a \mid s) B_{h}^{k}(s, a) - \sum_{h,s} q_{h}^{*}(s) b_{h}^{k}(s) 
                    -
                    \sum_{h,s,a} q^*_{h+1}(s) \pi^k_{h+1} (a \mid s) B^k_{h+1} (s,a)
                    \\
                    & =
                    \sum_{s,a} q^*_1(s) \pi^k_1(a \mid s) B^k_1(s,a) - \sum_{h,s} q_{h}^{*}(s) b_{h}^{k}(s) 
                    \\
                    & =
                    \sum_{s,a} q^k_1(s) \pi^k_1(a \mid s) B^k_1(s,a) - \sum_{h,s} q_{h}^{*}(s) b_{h}^{k}(s) 
                    \\                    
                    & \leq
                    \sum_{s,a} q^{\pi^k,\hat p^k}_1(s) \pi^k_1(a \mid s) B^k_1(s,a) - \sum_{h,s} q_{h}^{*}(s) b_{h}^{k}(s) 
                    \\
                    & =
                    \sum_{h,s} q_{h}^{\pi^k,\hat p^k}(s) b_{h}^{k}(s) - \sum_{h,s} q_{h}^{*}(s) b_{h}^{k}(s)
                    \\
                    & \leq
                    \sum_{h,s} \ol q^{k}_h(s) b^k_h(s) - \sum_{h,s} q_{h}^{*}(s) b_{h}^{k}(s),
    \end{align*}
where the first two inequalities are by definition of $\hat p^k$ and the fact that $p\in \calP^k$ under event $E^p$; and the last inequality is since, by definition, $\ol q^k_h(s)$ maximize the probability to visit $s$ at time $h$ among all the  occupancy measures within $\calP^k$.  
Summing over $k$ and bounding the first term above as follows completes the proof:
\begin{align*}
	\sum_{k=1}^{K}\sum_{h,s} 
	\ol q_{h}^{k}(s)b_{h}^{k}(s) & =\sum_{k=1}^{K}\sum_{h,s,a}\frac{3\gamma H \ol q_{h}^{k}(s)\pi_{h}^{k+d^{k}}(a \mid s) r^k_h(s,a)}{\ol                                 q_{h}^{k}(s,a) +\gamma}
                                    + \sum_{k=1}^{K}\sum_{h,s,a}
                                    \frac{2 H \ol q_{h}^{k}(s)\pi_{h}^{k+d^{k}}(a\mid s) r^k_h(s,a) (\ol q_{h}^{k}(s,a)-\ul q_{h}^{k}(s,a))}
                                    {\ol q_{h}^{k}(s,a)+\gamma} 
                                 \\
                                 & \leq 3\gamma H^{2}SAK
                                    +2H\sum_{k=1}^{K}\sum_{h,s,a} (\ol q_{h}^{k}(s,a)-\ul q_{h}^{k}(s,a))   \\     
                                 & \le \mathcal{O}\l( \gamma H^{2}SAK
                                    +H^3 S \sqrt{AK} \logterm + H^4 S^3 A \logterm^2  + H^4 S^2 A \dmax \r),
\end{align*}
where the last is by event $E^{est}$.
\end{proof}

\subsection{Bound on $\textsc{Reg}$}

\begin{lemma}
    \label{lemma:reg-unknown}
    For $\eta\leq\frac{\gamma}{2H}$ it holds that
    $
    	\textsc{Reg} 
    	\leq \frac{H\ln A}{\eta}
    	+ \frac{1}{3} \sum_{k,h,s} q_h^{*}(s)\tilde b_h^k(s)
    	+ \mathcal{O}\l( \eta H^{5}K 
    	+ \eta\frac{H^{3} \logterm}{\gamma^{2}} \r).
    $
\end{lemma}

\begin{proof}
    By \cref{corollary:delayed-exp-weights-regret}, since $\max_{k,h,s,a} B_h^k(s,a) \leq 5 H^2$, 
    \begin{align}
    \nonumber
    	\textsc{Reg} & \leq\frac{H\ln A}{\eta}
            	       + 2\eta\sum_{k,h,s,a}q_h^{*}(s)\pi_{h}^{k+d^k}(a\mid s)\l(\hat{Q}_{h}^{k}(s,a)-B_{h}^{k}(s,a)\r)^{2} + \mathcal{O}(\eta H^5 K)
            	       \\
        	         & \leq\frac{H\ln A}{\eta} 
            	       + 2\eta\sum_{k,h,s,a}q_h^{*}(s)\pi_{h}^{k+d^k}(a\mid s) \hat{Q}_{h}^{k}(s,a)^{2}
            	       + 2\eta\sum_{k,h,s,a}q_h^{*}(s)\pi_{h}^{k + d^k}(a\mid s) B_{h}^{k}(s,a)^{2} + \mathcal{O}(\eta H^5 K). 
        \label{eq:lemma reg unknown init bound}
    \end{align}
    For the middle term
    \begin{align*}
    	2\eta\sum_{k,h,s,a}q_h^{*}(s)
    	\pi_{h}^{k+d^k}(a\mid s)\hat{Q}_{h}^{k}(s,a)^{2} &
        	                                \leq 2\eta\sum_{k,h,s,a}q_h^{*}(s)\pi_{h}^{k + d^k}(a\mid s)
                                           \frac{r^k_h(s,a)^2 H^{2}\mathbb{I}\{s_{h}^{k}=s,a_{h}^{k}=a\}}
                                           {(\ol q_{h}^{k}(s,a) + \gamma)^{2}}
                                       \\
                                       & \leq 2\eta H^{2}\sum_{k,h,s,a}
                                                          \frac{q_h^{*}(s)\pi_{h}^{k + d^k}(a\mid s)r^k_h(s,a)}
                                                          {\ol q_{h}^{k}(s,a)+\gamma}
                                                          +  \mathcal{O}\l(\eta\frac{H^{3} \logterm}{\gamma^{2}} \r)  
                                        \\
                                       & \leq\frac{2\eta}{3\gamma} H \sum_{k,h,s} q_h^{*}(s)  \tilde b_h^k(s)
                                          + \mathcal{O}\l(\eta\frac{H^{3} \logterm}{\gamma^{2}} \r)
                                            \\
                                            & \leq\frac{1}{3} \sum_{k,h,s} q_h^{*}(s)  \tilde b_h^k(s)
                                            + \mathcal{O}\l(\eta\frac{H^{3} \logterm}{\gamma^{2}} \r),
    \end{align*}
    where the second inequality is by $E^b$ and the last is since $\eta\leq\frac{\gamma}{2H}$.
    For the last term in \cref{eq:lemma reg unknown init bound} we use $b^k_h(s) \le 5 H$ and therefore $B^k_h(s,a) \le 5 H^2$.
    Thus,
    \begin{align*}
    	\eta\sum_{k,h,s,a} q_h^{*}(s)\pi_{h}^{k+d^k} (a\mid s)B_{h}^{k}(s,a)^{2} & \leq 25 \eta H^{5}K .
    \end{align*}
    Overall,
    \[
    	\textsc{Reg} 
    	\leq \frac{H\ln A}{\eta}
    	+ \frac{1}{3} \sum_{k,h,s} q_h^{*}(s)\tilde b_h^k(s)
    	+ \mathcal{O}\l( \eta H^{5}K 
    	+ \eta\frac{H^{3} \logterm}{\gamma^{2}} \r). \qedhere
    \]
\end{proof}

\newpage

\section{Delay-Adapted Policy Optimization for Adversarial MDP with Linear $Q$-function}
\label{app: linear}
\begin{algorithm}
    \caption{Delay-Adapted Policy Optimization with Linear $Q$-function}  
    \label{alg:delayed-OPPO-linear}
    \begin{algorithmic}
        \STATE \textbf{Input:} feature dimension $\dim$, action space $\calA$, horizon $H$, feature map $\{ \phi_h:\calS\times\calA \to \bbR^{\dim} \}_{h=1}^H$, simulator of the environment, horizon $H$, learning rate $\eta > 0$, exploration parameter $\gamma > 0$, confidence parameter $\delta > 0$.
        
        \STATE \textbf{Initialization:} Set approximation parameter $\epsilon = (H \dim K)^{-1}$, bonus parameters 
        $\beta_1 = H\sqrt{\gamma \dim}$, 
        $\beta_2 = H\sqrt{\gamma \dim}$, 
        $\beta_{r} = 2 H\sqrt{\dim}$, 
        $\beta_v = 4\eta H^2$,
        $\beta_f = \gamma H$,
        $\beta_g = \gamma H$
        and Geometric Resampling parameters $M = \lceil \frac{24 }{\gamma^2 \epsilon^2}\ln \frac{10  H^2 K \dim }{\delta} \rceil , N = \lceil \frac{2}{\gamma} \log \frac{1}{\epsilon \gamma} \rceil$.
        
        \STATE  
        Define $\pi_{h}^{1}(a \mid s) = \nicefrac{1}{A}$ for every $(s,a,h) \in \mathcal{S} \times \mathcal{A} \times [H]$.
        
        \FOR{$k=1,2,\dots,K$}
            
            \STATE Play episode $k$ with policy $\pi^k$ and observe trajectory $\{ (s^k_h,a^k_h) \}_{h=1}^H$.
            
            \STATE {\color{gray} \# Policy Evaluation}
            
            \FOR{$j$ such that $j+d^j = k$}
                \STATE Observe bandit feedback $\{ c^j_h(s^j_h,a^j_h) \}_{h=1}^H$.
                
                \STATE Collect $MN$ trajectories $\trajset$ using the simulator and $\pi^j$.
                
                \STATE Compute estimated inverse covariance matrix $\{ \hat{\Sigma}_{h}^{j,+} \}_{h=1}^H = \textsc{GeometricResampling}(\trajset, M, N, \gamma)$ using the Matrix Geometric Resampling procedure \cite{neu2021online} (which samples trajectories of the policy using the simulator).
                
                \STATE Compute Monte-Carlo estimates $L^j_h = \sum_{h'=h}^H c^j_{h'}(s^j_{h'},a^j_{h'})$ for every $h \in [H]$.

                \STATE Compute estimated $Q$-function weights $\hat{\theta}_{h}^{j} = \hat{\Sigma}_{h}^{j,+}\phi_{h}(s_h^j,a_h^j)L_{h}^j$ for every $h \in [H]$.
                    
                \STATE Define delay-adapted ratio $r^j_h(s,a) = \frac{\pi^j_h(a\mid s)}{\max \{\pi^j_h(a\mid s),\pi^{k}_h (a\mid s)\}}$ for every $(s,a,h) \in \calS \times \calA \times [H]$.
                    
                \STATE Define delay-adapted estimated $Q$-function $\hat Q_h^j(s,a) = r^j_h(s,a)\phi_h(s,a)^{\top}\hat{\theta}_{h}^{j}$ for every $(s,a,h) \in \calS \times \calA \times [H]$.
                
                \STATE Define bonus $b^j_h(s,a) = b^{j,1}_h(s) + b^{j,2}_h(s,a) + b_h^{j,v}(s) + b^{j,r}_h(s,a) + b_h^{j,f}(s) + b^{j,g}_h(s,a)$ for every $(s,a,h)$, where:
                \begin{align*}
                    b_h^{j,1}(s) 
                    = 
                    \beta_1 \sum_{a} r_{h}^{j}(s,a) \pi_{h}^{j + d^{j}}(a \mid s) \Vert \phi_h(s,a) \Vert_{\hat{\Sigma}_{h}^{j,+}}
                    \qquad & ; \qquad
                    b_h^{j,2}(s,a) 
                    = 
                    \beta_2 r_{h}^{j}(s,a)\Vert\phi_h(s,a)\Vert_{\hat{\Sigma}_{h}^{j,+}}
                    \\
                    b_h^{j,v}(s)
                    = 
                    \beta_v m^{j+d^{j}} \sum_{a} r_{h}^{j}(s,a) \pi^{j+d^{j}}(a\mid s)  \norm{\phi_h(s,a)}_{\hat{\Sigma}_{h}^{j,+}}^{2}
                    \qquad & ; \qquad
                    b_h^{j,r}(s,a) 
                    = 
                    \beta_r (1-r_{h}^{j}(s,a))
                    \\
                    b_h^{j,f}(s) 
                    = 
                    \beta_f \sum_{a} r_{h}^{j}(s,a) \pi_{h}^{j + d^{j}}(a \mid s) \Vert \phi_h(s,a) \Vert_{\hat{\Sigma}_{h}^{j,+}}^2
                    \qquad & ; \qquad
                    b_h^{j,g}(s,a) 
                    = 
                    \beta_g r_{h}^{j}(s,a)\Vert\phi_h(s,a)\Vert_{\hat{\Sigma}_{h}^{j,+}}^2.
                \end{align*}
                
            \ENDFOR
            
            \STATE {\color{gray} \# Policy Improvement}
            
                \STATE Define the policy $\pi^{k+1}$ for every $(s,a,h) \in \calS \times \calA \times [H]$ by:
                \[
                    \pi^{k+1}_h(a \mid s) = \frac{\pi^k_h(a \mid s) \exp \l(-\eta \sum_{j: j+d^j = k}  ( \hat Q_{h}^{j}(s,a) - \hat B_h^j(s,a) ) \r)} 
                    {\sum_{a' \in \mathcal{A}} \pi^k_h(a' \mid s) \exp \l(-\eta \sum_{j: j+d^j = k}  ( \hat Q_{h}^{j}(s,a') - \hat B_h^j(s,a') ) \r)},
                \]
                where $\hat{B}_{h}^{j}(s,a)$ estimates $B_h^j(s,a)$ using the bonus procedure in \cref{alg:bonus linear}.
        \ENDFOR
        
    \end{algorithmic}
\end{algorithm}

\begin{algorithm}
    \caption{Bonus Procedure (Algorithm 3 in \citet{luo2021policy} adapted to non-dilated bonuses)}  
    \label{alg:bonus linear}
    \begin{algorithmic}
        \STATE \textbf{Input:} episode $j$, horizon $h$ state $s$, action $a$, local bonus function $b^j$, simulator of the environment.
        \IF{$\hat{B}_h^j(s,a)$ was computed before}
             \STATE Return the previously computed value of $\hat{B}_h^j(s,a)$.
        \ENDIF
        \STATE Compute $\pi^j_h(\cdot\mid s)$ and $\pi^{j+d^j}_h(\cdot\mid s)$ (which involves recursive calls to Bonus Procedure to compute $\hat B^{j'}$ for $j': j'+d^{j'} < j+d^j$).
        \STATE Sample $s'\sim p_h(\cdot \mid s, a)$ using the simulator.
        \STATE Compute $\pi_{h+1}^j(\cdot \mid s')$ and sample $a' \sim \pi_{h+1}^j(\cdot \mid s')$.
        \STATE Return $b_h^j(s,a) + \hat{B}_{h+1}^j(s',a')$.
    \end{algorithmic}
\end{algorithm}

\begin{theorem}
    \label{thm:regret-bound-linear-Q}
    Set $\gamma = \sqrt{\dim / K}$ and $\eta = \min \{ \frac{\gamma}{10 H \dmax} , \frac{1}{H (K+D)^{3/4}} \}$ and skip episodes with delay larger than $\beta = D^{1/4}$.
    Running \cref{alg:delayed-OPPO-linear} in an adversarial MDP $\calM = (\calS, \calA, H, p, \{ c^k \}_{k=1}^K)$ with linear $Q$-function, access to a simulator of the environment, a known features map $\{ \phi:\calS\times\calA \to \bbR^{\dim} \}_{h=1}^H$ and delays $\{ d^k \}_{k=1}^K$ guarantees, with probability at least $1 - \delta$,
    \[
        \regret
        \le
            \calO \l( H^{3} \dim^{5/4} K^{3/4} \log \frac{K H A}{\delta}
                        + H^{2} D^{3/4} \log \frac{K H A}{\delta} + H^5 \dim \log \frac{K H A}{\delta} \r).
    \]
\end{theorem}

\begin{remark}
    \label{remark:OLSPE}
    As noted in the main text, \cref{alg:bonus linear} is not sample efficient, and may require $(K A H)^{O (H)}$ calls to the simulator. However, under the stronger assumption that the MDP is linear (e.g., see assumption 2.1 in \citet{sherman2023improved}), we can replace this procedure with the OLSPE procedure of \citet{sherman2023improved} which would make our algorithm fully efficient while achieving the same regret guarantee of \cref{thm:regret-bound-linear-Q}.
\end{remark}

\subsection{The good event}
\label{sec:good event linear Q}

Let $\logterm = 10 \log \frac{10 KHA}{\delta}$,  $\tilde H^{k}$ be the history of all episodes $\{j: j+d^j < k\}$ and define $\bbE_k[\cdot] = \bbE\big[\cdot \mid \tilde H^{k+d^k} \big]$.
Let $\norm{\cdot}_{op}$ be the operator norm. That is, give a matrix $A\in \bbR^{\dim \times \dim}$, $\norm{A}_{op} := \inf\{c\geq 0:
\norm{Ax} \leq c\norm{x} \forall x\in \bbR^\dim \}$. Define the following events:
\begin{align*}
    E^\Sigma
    & = 
    \vast\{ 
        \forall k\in [K],h\in[H]:
        \norm{\hat{\Sigma}_{h}^{k,+}-(\gamma I+\Sigma_{h}^{k})^{-1}}_{op} \leq 2 \epsilon \text{ and }
        \norm{\hat{\Sigma}_{h}^{k,+}\Sigma_{h}^{k}}_{op}
        \leq
        1 + 2\epsilon
    \vast\}
    \\
    E^b
    & = 
    \vast\{ \forall h\in[H]:
        \sum_{k=1}^{K} \sum_{a} \bbE_{s \sim q_{h}^*} \l[ \pi^{k+d^{k}}(a\mid s) m^{k+d^{k}} \l( \hat{Q}_{h}^{k}(s,a) \r)^{2} \r]
        \\ & \qquad \qquad \qquad \qquad \qquad \qquad 
        \leq
        2 \sum_{k=1}^{K} \sum_{a} \bbE_{s \sim q_{h}^*} \l[  \pi^{k+d^{k}} (a\mid s) m^{k+d^{k}} \bbE_{k} \l[ \l(\hat{Q}_{h}^{k}(s,a)\r)^{2} \r]\r] + \frac{4 H^2 \dmax\log\frac{10H}{\delta}}{\gamma^{2}} 
    \vast\}
    \\
    E^f
    & = 
    \vast\{
        \sum_{k=1}^K \bbE_k \l[\sum_{h} \bbE_{s\sim q_{h}^{*}} \l\langle \pi_{h}^{k + d^k}(\cdot \mid s),\hat{Q}_{h}^{k}(s,\cdot) \r\rangle \r]
            - \sum_{k=1}^K \sum_{h} \bbE_{s\sim q_{h}^{*}} \l\langle \pi_{h}^{k + d^k}(\cdot \mid s),\hat{Q}_{h}^{k}(s,\cdot) \r\rangle 
        \\  & \qquad \qquad \qquad \qquad \qquad \qquad \qquad \qquad \qquad \qquad \qquad \qquad 
        \leq \sum_{k=1}^{K} \sum_{h} \bbE_{s\sim q_{h}^{*}} [b_{h}^{k,f}(s)] 
            + \frac{H^{2} \log \frac{10}{\delta}}{\gamma}  + {\cal O}(\gamma H^2 K \epsilon) 
    \vast\}
    \\
    E^g
    & = 
    \vast\{
    \sum_{k=1}^K \sum_{h} \bbE_{s\sim q_{h}^{*}} \l\langle \pi_{h}^{*}(\cdot \mid s),\hat{Q}_{h}^{k}(s,\cdot) \r\rangle 
    -
        \sum_{k=1}^K \bbE_k \l[\sum_{h} \bbE_{s\sim q_{h}^{*}} \l\langle \pi_{h}^{*}(\cdot \mid s),\hat{Q}_{h}^{k}(s,\cdot) \r\rangle \r]
        \\  & \qquad \qquad \qquad \qquad \qquad \qquad \qquad \qquad \qquad \qquad \qquad \qquad 
        \leq \sum_{k=1}^{K} \sum_{h} \bbE_{ s,a \sim q_{h}^{*}} [b_{h}^{k,g}(s,a)] 
            + \frac{H^{2} \log \frac{10}{\delta}}{\gamma}  + {\cal O}(\gamma H^2 K \epsilon)
    \vast\}.
    \\
    E^B
    & = 
    \vast\{
        \sum_{k=1}^{K}\sum_{h=1}^{H}\bbE_{s\sim q_{h}^{*}} \l[ \l\langle \pi_{h}^{k}(\cdot\mid s) - \pi_{h}^{*}(\cdot\mid s),\hat{B}_{h}^{k}(s,\cdot) \r\rangle  \r]
        \\
        & \qquad \qquad \qquad \qquad \qquad \qquad \qquad 
        \leq 
        \sum_{k=1}^{K}\sum_{h=1}^{H}\bbE_{k} \l[\sum_{h}\bbE_{s\sim q_{h}^{*}} \l[ \l\langle \pi_{h}^{k}(\cdot\mid s) - \pi_{h}^{*}(\cdot\mid s),\hat{B}_{h}^{k}(s,\cdot) \r\rangle  \r] \r]
        + 
        \calO ( H^{3}\sqrt{\dim K}\ln\frac{10}{\delta})
    \vast\}.
\end{align*}

The good event is the intersection of the above events. 
The following lemma establishes that the good event holds with high probability. 

\begin{lemma}[The Good Event]
    \label{lemma:good-event-linear-Q}
    Let $\bbG = E^\Sigma \cap E^b \cap E^f \cap E^g\cap E^B$ be the good event. 
    It holds that $\Pr [ \bbG ] \geq 1-\delta$.
\end{lemma}

\begin{proof}

    We'll show that each of the events $\lnot E^\Sigma, \lnot E^b ,\lnot E^f ,\lnot E^g,\lnot E^B$ occures with probability $\leq \delta/5$. By taking the union bound we'll get that $\Pr [ \bbG ] \geq 1-\delta$.

    \textbf{Event $E^\Sigma$:} We use \textsc{GeometricResampling} with $M \geq \frac{24 }{\gamma^2 \epsilon^2}\ln \frac{10  H^2 K \dim }{\delta}$ and $N \geq \frac{2}{\gamma} \ln \frac{1}{\epsilon \gamma}$. Thus, by \cref{lemma: GR h.p}, the event holds for each $h$ and $k$ separately with probability of at least $1-\frac{\delta}{10 H K}$. By taking the union bound over $[H]$ and $[K]$ we get that $\Pr[E^\Sigma] \geq 1 - \delta/10$.

    \textbf{Event $E^b$:}
    Fix $h$. By \cref{lemma: linear Q < 1/gamma} $|\hat{Q}_{h}^{k}(s,a)| \leq \frac{H}{\gamma}$ and thus 
    $\sum_{a} \bbE_{s \sim q_{h}^*} \l[ \pi^{k+d^{k}}(a\mid s) m^{k+d^{k}} \l( \hat{Q}_{h}^{k}(s,a) \r)^{2} \r] \leq \frac{H^2 \dmax}{\gamma^2}$.
     Also note that $\pi^{k+d^k}$ is determined by the history $\tilde H^{k+d^k}$. Thus, by \cref{lemma:cons-freedman} the event holds with probability $1- \delta/(10H)$. The proof is finished by a union bound over $h$.

    \textbf{Event $E^f$:}
    Let $Y_{k}= \sum_{h} \bbE_{s\sim q_{h}^{*}} \l\langle \pi_{h}^{k + d^k}(\cdot \mid s),\hat{Q}_{h}^{k}(s,\cdot) \r\rangle $.
    Similarly to the tabular case, we use a variant of Freedman's inequality (\cref{lemma:freedman}) to bound $\sum_{k=1}^{K} X_k := \sum_{k=1}^{K}  \bbE_k[Y_{k}   ] - \sum_{k=1}^{K}Y_{k} $.
    \begin{align*}
    	\bbE_k \l[ X_{k}^2 \r] \leq \bbE_k \l[ Y_{k}^2 \r] & =\bbE_{k}\l[\l(\sum_{h,a}\bbE_{s\sim q_{h}^{*}}\l[\pi_{h}^{k+d^{k}}(a\mid s)\hat{Q}_{h}^{k}(s,a)\r]\r)^{2}\r]                                                                                       \\
    	                                                   & =\bbE_{k}\l[\l(\sum_{h}\bbE_{s\sim q_{h}^{*},a\sim\pi_{h}^{k+d^{k}}(\cdot \mid s)}\l[\hat{Q}_{h}^{k}(s,a)\r]\r)^{2}\r]                                                                                    \\
    	                                                   & \leq H\bbE_{k}\l[\sum_{h}\l(\bbE_{s\sim q_{h}^{*},a\sim\pi_{h}^{k+d^{k}}(\cdot \mid s)}\l[\hat{Q}_{h}^{k}(s,a)\r]\r)^{2}\r]                                                                               \\
    	                                                   & \leq H\bbE_{k}\l[\sum_{h}\bbE_{s\sim q_{h}^{*},a\sim\pi_{h}^{k+d^{k}}(\cdot \mid s)}\l[\l(\hat{Q}_{h}^{k}(s,a)\r)^{2}\r]\r]                                                      \\
    	                                                   & = H\bbE_{k}\l[\sum_{h}\bbE_{s\sim q_{h}^{*}}\l[\sum_{a}\pi_{h}^{k+d^{k}}(a\mid s)\l(\hat{Q}_{h}^{k}(s,a)\r)^{2}\r] \r]                                               \\
    	                                                   & = H\sum_{h}\bbE_{s\sim q_{h}^{*}}\l[\sum_{a}\pi_{h}^{k+d^{k}}(a\mid s)\bbE_{k}\l[ \l(\hat{Q}_{h}^{k}(s,a)\r)^{2}\r]\r] 
    	                                                                  \\
    	                                                   & \leq H^{3}\sum_{h}\bbE_{s\sim q_{h}^{*}}\l[\sum_{a}\pi_{h}^{k+d^{k}}(a\mid s)r_{h}^{k}(s,a)\norm{\phi_{h}(s,a)}_{\hat{\Sigma}_{h}^{k,+}}^{2}\r]+{\cal O}(H^{4}\epsilon) 
    \end{align*}
    
    where the first inequality is since $\Vert x\Vert_{1}^{2} \leq n\Vert x\Vert_{2}^{2}$ for any $x\in\bbR^{n}$, the second inequality is by Jensen's inequality, and the last inequality is by \cref{lemma:linear Q bound E[Q2]}.
    Also, $|Y_{k}|\leq\frac{H^{2}}{\gamma}$. Therefore by \cref{lemma:freedman} with probability $1 - \delta/10$,
    \begin{align*}
        \sum_{k=1}^{K}  \bbE_k [Y_{k}  ] - \sum_{k=1}^{K}Y_{k} 
        & \leq
        \gamma H\sum_{k=1}^{K} \sum_{h} \bbE_{s\sim q_{h}^{*}}\l[\sum_{a}\pi_{h}^{k+d^{k}}(a\mid s)r_{h}^{k}(s,a)\norm{\phi_{h}(s,a)}_{\hat{\Sigma}_{h}^{k,+}}^{2}\r]  
        + \frac{H^{2} \log \frac{10}{\delta}}{\gamma}
        + {\cal O}(\gamma H K \epsilon) 
        \\
        & = \sum_{k=1}^{K} \sum_{h} \bbE_{s\sim q_{h}^{*}}\l[ b_{h}^{k,f}(s) \r]
        + \frac{H^{2} \log \frac{10}{\delta}}{\gamma}
        + {\cal O}(\gamma H^2 K \epsilon) 
        .
    \end{align*}
    
   \textbf{Event $E^g$:}
  The proof is similar to event $E^f$.

\textbf{Event $E^B$:} 
    Define, 
    $$
        X_k := \sum_{h=1}^{H}\bbE_{s\sim q_{h}^{*}} \l[ \l\langle \pi_{h}^{k}(\cdot\mid s) - \pi_{h}^{*}(\cdot\mid s),\hat{B}_{h}^{k}(s,\cdot) \r\rangle  \r]
        -\sum_{h=1}^{H}\bbE_{k} \l[\sum_{h}\bbE_{s\sim q_{h}^{*}} \l[ \l\langle \pi_{h}^{k}(\cdot\mid s) - \pi_{h}^{*}(\cdot\mid s),\hat{B}_{h}^{k}(s,\cdot) \r\rangle  \r] \r]
    $$
    and note that $|X_k| \leq \calO (H^3 \sqrt{\dim})$. By Azuma–Hoeffding inequality (\cref{lemma:Azuma_Hoeffding}) we get that the event holds with probability $1 - \delta/10$.
\end{proof}

\begin{lemma}
\label{lemma: linear Q < 1/gamma}
    For any $(k,h,s,a) \in [K] \times [H] \times \calS \times \calA$, it holds that $|\hat{Q}_h^k(s,a)| \leq \frac{H}{\gamma}$.
\end{lemma}

\begin{proof}
By Cauchy-Schwartz, \cref{eq:PSDs invA invA} in \cref{lemma:PSDs} and \cref{ass:linear-Q}:
$$
	|\hat{Q}_{h}^{k}(s,a)| 
	=
	|r^k_h(s,a) \phi_{h}^{\top}(s,a)\hat{\Sigma}_{h}^{k,+}\phi(s_{h}^{k},a_{h}^{k})L_{h}^{k}| 
	\le  
	H \norm{\phi(s,a)}_2 \norm{\hat{\Sigma}_{h}^{k,+}\phi(s_{h}^{k},a_{h}^{k})}_2
	\le 
	H \norm{\hat{\Sigma}_{h}^{k,+}}_{op} \norm{\phi(s_{h}^{k},a_{h}^{k})}_2
	\le
	\frac{H}{\gamma}.
$$
\end{proof}

\begin{lemma}[Lemma D.1 in \citet{luo2021policy}]
    \label{lemma: GR h.p}
    Fix a policy $\pi$ with a covariance matrix $\Sigma_h = \bbE_{s,a \sim q_h^\pi} [\phi(s,a) \phi(s,a)^\top]$. Let $\hat \Sigma^+_h$ be the output of $\textsc{GeometricResampling}(\mathcal{T},M,N,\gamma)$ with $M \geq \frac{24 }{\gamma^2 \epsilon^2}\ln \frac{\dim H}{\delta}$ and $N \geq \frac{2}{\gamma} \ln \frac{1}{\epsilon \gamma}$ and $\mathcal{T}$ are $MN$ trajectories collected with $\pi$. Then, $\norm{\hat \Sigma^+_h}_{op} \le \frac{1}{\gamma}$, and with probability $1 - \delta$,
    \begin{align*}
        \norm{\hat{\Sigma}_{h}^{+} - (\gamma I+\Sigma_{h})^{-1}}_{op}
        \leq
        2\epsilon
        \qquad ; \qquad
        \norm{\hat{\Sigma}_{h}^{+} \Sigma_{h}}_{op}
        \leq
        1 + 2\epsilon.
    \end{align*}
\end{lemma}

\subsection{Proof of the main theorem}

\begin{proof}[Proof of \cref{thm:regret-bound-linear-Q}]
By \cref{lemma:good-event-linear-Q}, the good event holds with probability of at least $1 - \delta$.
As in the previous section, we analyze the regret under the assumption that the good event holds.
We start with the following regret decomposition,
\begin{align*}
     \regret  = & \sum_{k=1}^{K}  \bbE_{s\sim q^{\pi^*}_h} \l[\l\langle \pi_{h}^{k} (\cdot \mid s) - \pi_{h}^{*} (\cdot \mid s),Q_{h}^{k}(s, \cdot) \r\rangle \r]
    = \underbrace{
    \sum_{k=1}^{K} \sum_{h}\bbE_{s\sim q^{\pi^*}_h} \l[ \l\langle \pi_{h}^{k + d^{k}} (\cdot \mid s),Q_{h}^{k}(s, \cdot) - \hat{Q}_{h}^{k}(s, \cdot) \r\rangle \r]}
    _{\textsc{Bias}_{1}} 
    \\
    & + \underbrace{
    \sum_{k=1}^{K} \sum_{h}\bbE_{s\sim q^{\pi^*}_h} \l[ \l\langle \pi_{h}^{*} (\cdot \mid s), \hat{Q}_{h}^{k}(s, \cdot) - Q_{h}^{k}(s, \cdot) \r\rangle \r]
    }_{\textsc{Bias}_{2}} 
    + \underbrace{ 
    \sum_{k=1}^{K} \sum_{h}\bbE_{s\sim q^{\pi^*}_h} \l[ \l\langle \pi_{h}^{k} (\cdot \mid s) - \pi_{h}^{*} (\cdot \mid s),\hat B_{h}^{k}(s, \cdot) \r\rangle \r]
    }_{\textsc{Bonus}} 
    \\
    & + \underbrace{ 
    \sum_{k=1}^{K} \sum_{h}\bbE_{s\sim q^{\pi^*}_h} \l[ \l\langle \pi_{h}^{k + d^{k}} (\cdot \mid s) - \pi_{h}^{*} (\cdot \mid s), \hat{Q}_{h}^{k}(s, \cdot) - 
    \hat B_{h}^{k}(s, \cdot) \r\rangle \r] 
    }_{\textsc{Reg}} 
    \\
    & + \underbrace{ 
    \sum_{k=1}^{K} \sum_{h}\bbE_{s\sim q^{\pi^*}_h} \l[ \l\langle \pi_{h}^{k} (\cdot \mid s) - \pi_{h}^{k + d^{k}} (\cdot \mid s),Q_{h}^{k}(s, \cdot) - \hat B_{h}^{k}(s, \cdot) \r\rangle \r] 
    }_{\textsc{Drift}},
\end{align*}
where the first is by \cref{lemma: value diff}.
The five terms above are bounded in \cref{lemma: linear Q bias1,lemma:linear Q bias2,lemma: linear Q bonus,lemma: linear Q reg,lemma: linear Q drift}. For $\eta <  \frac{\gamma}{10 H \dmax}$ and $\epsilon < (H\dim  K)^{-1}$, 
\begin{align*}
     \regret  & \leq 
    \underbrace{
            \sum_{k=1}^K \sum_{h} \bbE_{s \sim q_{h}^{\pi^{*}}} \l[ b_{h}^{k,1}(s) + b_{h}^{k,f}(s) \r] 
        + \calO\l( \frac{\eta}{\gamma} H^3 (K + D) \sqrt{\dim} 
        + \frac{H^2}{\gamma}\logterm
        +  H \sqrt{\gamma \dim}K \epsilon + \gamma H^2 K \epsilon \r)
    }_{\textsc{Bias}_{1}} 
    \\
    &\qquad + \underbrace{
            \sum_{k=1}^K \sum_{h} \bbE_{s , a \sim q_{h}^{\pi^{*}}} \l[ {b}_{h}^{k,2}(s,a) + b_h^{k,r}(s,a) + {b}_{h}^{k,g}(s,a) \r]  
        + \calO \l( \frac{H^2}{\gamma}\logterm
        +  H \sqrt{\gamma \dim}K \epsilon + \gamma H^2 K \epsilon \r)
    }_{\textsc{Bias}_{2}} 
     \\
    &\qquad + \underbrace{ 
    \sum_{k=1}^K\sum_{h} \bbE_{s\sim q_{h}^{*}}\l[b_{h}^{k,v}(s) \r]
        +\calO \l(
            \frac{H\logterm}{\eta} 
        + \eta H^{5} \dim K
        + \frac{H^2 \logterm
	                    }{\gamma} + \eta H^3 \epsilon \r)
    }_{\textsc{Reg}} 
    + \underbrace{ 
    \calO\l(
    \frac{\eta}{\gamma} H^{4} \sqrt{\dim}  (K+D) 
    \r)
    }_{\textsc{Drift}}.
    \\
    &\qquad + \underbrace{ 
            \calO \l( \sqrt{\gamma}H^2  \dim K
        + \eta H^2 \dim K
        + \frac{\eta }{\gamma} H^3 \sqrt{ \dim} (K+D) 
        + H^{3} \sqrt{\dim K} \logterm  \r)
        - \sum_{k=1}^K \sum_{h=1}^H \bbE_{s,a \sim q_{h}^{\pi^{*}}}\l[ b_{h}^{k}(s,a)\r]
    }_{\textsc{Bonus}} 
    \\
    & \qquad\le \calO \l( \frac{H \logterm}{\eta} 
                + \eta H^5 \dim K
                + \frac{\eta}{\gamma} H^4 \sqrt{\dim} (K + D)
                + \sqrt{\gamma} H^2 \dim K
                + H^3 \sqrt{\dim K} \logterm \r).
\end{align*}
For $\gamma = H \sqrt{\frac{\dim}{K}}$ and $\eta = \min\l\{ \frac{1}{H (K+D)^{3/4}}, \frac{\gamma}{10 H \dmax} \r\}$, we get: $$
    \regret    \le \calO \l( H^{3} \dim^{5/4} K^{3/4} \logterm
                        + H^{2} D^{3/4} \logterm
                        + H^4 \dim K^{1/4} \logterm
                        + H^3 \sqrt{\dim K} \logterm
                        + H^2 \dmax \sqrt{K / n}  \logterm \r).
$$
Finally, by skipping rounds larger than $\beta = D^{1/4}$, we get:
$
    \regret    \le \calO ( H^{3} \dim^{5/4} K^{3/4} \logterm
                        + H^{2} D^{3/4} \logterm + H^5 \dim \logterm ).$
\end{proof}

\subsection{Bound on $\textsc{Bias}_1$}

\begin{lemma}
    \label{lemma: linear Q bias1}
    Under the good event,
    \begin{align*}
        \textsc{Bias}_1 
        \leq
        \sum_{k=1}^K \sum_{h} \bbE_{s \sim q_{h}^{\pi^{*}}} \l[ b_{h}^{k,1}(s) + b_{h}^{k,f}(s) \r] 
        +  \calO \l( \frac{\eta}{\gamma} H^3 \sqrt{\dim}  (K + D) 
        + \frac{H^2}{\gamma}\logterm
        + H \sqrt{\gamma \dim}K \epsilon + \gamma H^2 K \epsilon \r).
    \end{align*}
\end{lemma}

\begin{proof}
We first decompose $\textsc{Bias}_1$ as,
\begin{align*}
	\textsc{Bias}_{1} & =\sum_{k=1}^{K}\sum_{h}\bbE_{s\sim q_{h}^{*}}  \l[  \l\langle \pi_{h}^{k+d^{k}}(\cdot\mid s),Q_{h}^{k}(s,\cdot)-\hat{Q}_{h}^{k}(s,\cdot) \r\rangle  \r]                                  \\
	                  & =\underbrace{
	                  \sum_{k=1}^{K}\sum_{h}\bbE_{s\sim q_{h}^{*}}  \l[  \l\langle \pi_{h}^{k+d^{k}}(\cdot\mid s),\bbE_{k}  \l[Q_{h}^{k}(s,\cdot)-\hat{Q}_{h}^{k}(s,\cdot) \r] \r\rangle  \r] }_{(A)}             \\
	                  & \qquad+ \underbrace{
	                  \sum_{k=1}^{K}\sum_{h}\bbE_{s\sim q_{h}^{*}}  \l[  \l\langle \pi_{h}^{k+d^{k}}(\cdot\mid s),\bbE_{k}  \l[\hat{Q}_{h}^{k}(s,\cdot) \r]-\hat{Q}_{h}^{k}(s,\cdot) \r\rangle  \r] }_{(B)} ,
\end{align*}
where the $(B)\leq \sum_{k,h} \bbE_{s\sim q_{h}^{*}} [b_{h}^{k,f}(s)] 
            + \frac{H^{2}}{\gamma} \logterm + {\cal O}(\gamma H^2 K \epsilon)$ by event $E^f$.
For $(A)$ we use \cref{lemma:linear Q bias equality},
\begin{align*}
	(A) & \leq\sum_{k=1}^{K} \sum_{h,a} \bbE_{s \sim q_{h}^{\pi^{*}}} \l[ \gamma\pi_{h}^{k+d^{k}}(a\mid s) 
	                    \phi_h(s,a)^{\top} (\gamma I + \Sigma_{h}^{k})^{-1} \theta_{h}^{k} \r]
	                  \\
	                  & \qquad + \sum_{k=1}^{K} \sum_{h,a} \bbE_{s \sim q_{h}^{\pi^{*}}} \l[ \pi_{h}^{k+d^{k}}(a\mid s)(1-r_{h}^{k}(s,a)) \phi_h(s,a)^{\top} (\gamma I + \Sigma_{h}^{k})^{-1} \Sigma_{h}^{k} \theta_{h}^{k} \r] + \calO(\epsilon H^{2} K)
	                  \\
	                  & = \underbrace{\sum_{k=1}^{K}
	                  \sum_{h,a} \bbE_{s \sim q_{h}^{\pi^{*}}} \l[ \gamma r_{h}^{k}(s,a)\pi_{h}^{k+d^{k}}(a\mid s) \phi_h(s,a)^{\top} (\gamma I + \Sigma_{h}^{k})^{-1} \theta_{h}^{k} \r]
	                  }_{(i)}
	                  \\
	                  & \qquad +  \underbrace{\sum_{k=1}^{K}
	                  \sum_{h,a} \bbE_{s \sim q_{h}^{\pi^{*}}} \l[ \gamma\pi_{h}^{k+d^{k}}(a\mid s)(1-r_{h}^{k}(s,a)) \phi_h(s,a)^{\top} (\gamma I + \Sigma_{h}^{k})^{-1} \theta_{h}^{k} \r]
	                  }_{(ii)}
	                  \\
	                  & \qquad + \underbrace{\sum_{k=1}^{K} 
	                  \sum_{h,a} \bbE_{s \sim q_{h}^{\pi^{*}}} \l[ \pi_{h}^{k+d^{k}}(a\mid s)(1-r_{h}^{k}(s,a)) \phi_h(s,a)^{\top} (\gamma I + \Sigma_{h}^{k})^{-1} \Sigma_{h}^{k} \theta_{h}^{k} \r]
	                  }_{(iii)} + \calO(\epsilon H^{2} K).
\end{align*}
Using Cauchy-Schwarz inequality,
\begin{align*}
	(i) & \leq\sum_{h,a,k} \bbE_{s \sim q_{h}^{\pi^{*}}} \l[ \gamma r_{h}^{k}(s,a)\pi_{h}^{k+d^{k}}(a\mid s)\Vert\phi_h(s,a)\Vert_{ (\gamma I + \Sigma_{h}^{k})^{-1}} \Vert\theta_{h}^{k} \Vert_{ (\gamma I + \Sigma_{h}^{k})^{-1}} \r] 
	    \\
	    & \leq H \sqrt{\gamma \dim} \sum_{h,a,k} \bbE_{s \sim q_{h}^{\pi^{*}}} \l[r_{h}^{k}(s,a)\pi_{h}^{k+d^{k}}(a\mid s)\Vert\phi_h(s,a)\Vert_{ (\gamma I + \Sigma_{h}^{k})^{-1}} \r]
	    \\
	    & \leq H \sqrt{\gamma \dim} \sum_{h,k} \bbE_{s \sim q_{h}^{\pi^{*}}} \l[ \sum_{a}r_{h}^{k}(s,a)\pi_{h}^{k+d^{k}}(a\mid s)\Vert\phi_h(s,a)\Vert_{\hat{\Sigma}_{h}^{k,+}} \r] + \calO(\epsilon H K \sqrt{\gamma \dim})
	    \\
	    & = \sum_{h,k} \bbE_{s \sim q_{h}^{\pi^{*}}} \l[  b_{h}^{k,1}(s)\r] + \calO(\epsilon H K \sqrt{\gamma \dim})                                        
	    ,                                             
\end{align*}
where the second inequality is since $\Vert\theta_{h}^{k} \Vert_{ (\gamma I + \Sigma_{h}^{k})^{-1}}=\sqrt{\theta_{h}^{k\top} (\gamma I + \Sigma_{h}^{k})^{-1} \theta_{h}^{k}} \leq\sqrt{\frac{\norm{\theta_{h}^{k}}_2^2}{\gamma}} \leq H \sqrt{\frac{\dim}{\gamma}}$ by \cref{eq:PSDs invA} of \cref{lemma:PSDs} and \cref{ass:linear-Q}, and the third is by Cauchy–Schwarz and event $E^\Sigma$ in a similar way to \cref{eq:linear Q op-norm use}. 
For term $(ii)$, note that
\begin{align}
    \label{eq: cdot leq 1/gamma}
    |\phi_h(s,a)^{\top}(\gamma I+\Sigma_{h}^{k})^{-1}\theta_{h}^{k}|
    \leq
    \Vert\phi_h(s,a) \Vert_2 \cdot\Vert(\gamma I + \Sigma_{h}^{k})^{-1}\theta_{h}^{k}\Vert_2
    \leq
    \frac{1}{\gamma}\Vert \theta_{h}^{k} \Vert_2
    \leq
    \frac{H\sqrt{\dim}}{\gamma} ,
\end{align}
where the first inequality is by Cauchy–Schwarz, and the second inequality is by \cref{eq:PSDs invA invA} in \cref{lemma:PSDs}.
Therefore,
\begin{align*}
	(ii) & \leq H\sqrt{\dim}\sum_{h,a,k} \bbE_{s \sim q_{h}^{\pi^{*}}}          \l[ \pi_{h}^{k+d^{k}}(a\mid s)(1-r_{h}^{k}(s,a))\r]
	     \\
	     & =H\sqrt{\dim}\sum_{h,a,k} \bbE_{s \sim q_{h}^{\pi^{*}}} \l[ \pi_{h}^{k+d^{k}}(a\mid s)\frac{\max\{\pi_{h}^{k+d^{k}}(a\mid s),\pi_{h}^{k}(a\mid s)\}-\pi_{h}^{k}(a\mid s)}{\max\{\pi_{h}^{k+d^{k}}(a\mid s),\pi_{h}^{k}(a\mid s)\}} \r] 
	     \\
	     & \leq H\sqrt{\dim}\sum_{h,a,k} \bbE_{s \sim q_{h}^{\pi^{*}}} \l[ \max\{\pi_{h}^{k+d^{k}}(a\mid s),\pi_{h}^{k}(a\mid s)\}-\pi_{h}^{k}(a\mid s)\r]
	     \\
	     & \leq H\sqrt{\dim} \sum_{h,k} \bbE_{s \sim q_{h}^{\pi^{*}}} \l[ \Vert\pi_{h}^{k+d^{k}}(\cdot\mid s)-\pi_{h}^{k}(\cdot\mid s)\Vert_{1} \r]
 	     \\
	     & \leq  \calO \l(\frac{\eta}{\gamma} H^3 \sqrt{\dim} (K + D) \r),                                                  
\end{align*}
where the lsat is by \cref{lemma:linear Q l1 diff}. Similarly, using \cref{eq:PSDs invA A} in \cref{lemma:PSDs},
\begin{align}
    \label{eq: cdot leq 1}
    |\phi_h(s,a)^{\top}(\gamma I+\Sigma_{h}^{k})^{-1}\Sigma_{h}^{k}\theta_{h}^{k}|
    \leq
    \Vert\phi_h(s,a) \Vert_2 \cdot\Vert(\gamma I+\Sigma_{h}^{k})^{-1} \Sigma_{h}^{k}\theta_{h}^{k} \Vert_2
    \leq \Vert \theta_{h}^{k} \Vert_2
    \leq 
    H\sqrt{\dim}.
\end{align}
Thus, $(iii)$ is bounded in the same way as $(ii)$.
\end{proof}

\begin{lemma}
    \label{lemma:linear Q bias equality}
    Under the good event, for every $(k,h,s,a) \in [K] \times [H] \times \calS \times \calA$, it holds that
        \begin{align*}
    	 \l| \bbE_{k}\l[Q_{h}^{k}(s,a) - \hat{Q}_{h}^{k}(s,a) \r] 
    	-
    	\gamma\phi_h(s,a)^{\top}(\gamma I + \Sigma_{h}^{k})^{-1} \theta_{h}^{k}+(1 - r_{h}^{k}(s,a)) \phi_h(s,a)^{\top}(\gamma I + \Sigma_{h}^{k})^{-1} \Sigma_{h}^{k} \theta_{h}^{k} \r| \le \calO(\epsilon H).
    \end{align*}
\end{lemma}

\begin{proof}
    By the definition of $Q^k_h(s,a)$, $\hat Q^k_h(s,a)$ and $\hat \theta^k_h$,
    \begin{align*}
    	\bbE_{k}\l[Q_{h}^{k}(s,a) - \hat{Q}_{h}^{k}(s,a) \r] & =\phi_h(s,a)^{\top} \l( \theta_{h}^{k} - r_{h}^{k}(s,a) \bbE_{k}\l[\hat{\theta}_{h}^{k} \r] \r) 
    	                                                \\
    	                                              & =\phi_h(s,a)^{\top} \l( \theta_{h}^{k} - r_{h}^{k}(s,a) \bbE_{k}\l[\hat{\Sigma}_{h}^{k,+} \r]\bbE_{k}\l[\phi_{h}(s_{h}^{k},a_{h}^{k})L_{h}^{j} \r] \r)
    	                                                \\
    	                                              & =\phi_h(s,a)^{\top} \l( \theta_{h}^{k} - r_{h}^{k}(s,a)(\gamma     
    	                                                I + \Sigma_{h}^{k})^{-1} \bbE_{k}\l[\phi_{h}(s_{h}^{k},a_{h}^{k})L_{h}^{j} \r] \r)  + \calO(\epsilon H)
    	                                                \\
    	                                              & =\phi_h(s,a)^{\top} \l( \theta_{h}^{k} - r_{h}^{k}(s,a)(\gamma I + \Sigma_{h}^{k})^{-1} \bbE_{k}\l[\phi_{h}(s_{h}^{k},a_{h}^{k}) \bbE\l[L_{h}^{j} \mid s_{h}^{k},a_{h}^{k} \r] \r] \r)  + \calO(\epsilon H)
    	                                                \\
    	                                              & =\phi_h(s,a)^{\top} \l( \theta_{h}^{k} - r_{h}^{k}(s,a)(\gamma I + \Sigma_{h}^{k})^{-1} \bbE_{k}\l[\phi_{h}(s_{h}^{k},a_{h}^{k})Q_{h}^{j}(s_{h}^{k},a_{h}^{k}) \r] \r) + \calO(\epsilon H)
    	                                                \\
    	                                              & =\phi_h(s,a)^{\top} \l( \theta_{h}^{k} - r_{h}^{k}(s,a)(\gamma I + \Sigma_{h}^{k})^{-1} \bbE_{k}\l[\phi_{h}(s_{h}^{k},a_{h}^{k}) \phi_{h}(s_{h}^{k},a_{h}^{k})^{\top} \theta_{h}^{k} \r] \r) + \calO(\epsilon H)
    	                                                \\
    	                                              & =\phi_h(s,a)^{\top} \l( \theta_{h}^{k} - r_{h}^{k}(s,a)(\gamma I + \Sigma_{h}^{k})^{-1} \Sigma_{h}^{k} \theta_{h}^{k} \r)  + \calO(\epsilon H)
    	                                                \\
    	                                              & =\phi_h(s,a)^{\top} \l( (\gamma I + \Sigma_{h}^{k})^{-1}(\gamma I + \Sigma_{h}^{k}) \theta_{h}^{k} - r_{h}^{k}(s,a)(\gamma I + \Sigma_{h}^{k})^{-1} \Sigma_{h}^{k} \theta_{h}^{k} \r)  + \calO(\epsilon H) 
    	                                                \\
    	                                              & =\gamma\phi_h(s,a)^{\top}(\gamma I + \Sigma_{h}^{k})^{-1} \theta_{h}^{k}+(1 - r_{h}^{k}(s,a)) \phi_h(s,a)^{\top}(\gamma I + \Sigma_{h}^{k})^{-1} \Sigma_{h}^{k} \theta_{h}^{k} + \calO(\epsilon H)
    \end{align*}
    where the third equality is since,
    \begin{align}
    \nonumber
    	  \nonumber \phi_{h}(s,a)^{\top}r_{h}^{k}(s,a)(\bbE_{k}\l[\hat{\Sigma}_{h}^{k,+}\r] & -(\gamma I+\Sigma_{h}^{k})^{-1})\bbE_{k}\l[\phi_{h}(s_{h}^{k},a_{h}^{k})L_{h}^{j}\r] \le \\
    	  \nonumber& \leq\norm{\phi_{h}(s,a)}_2 \norm{(\bbE_{k}\l[\hat{\Sigma}_{h}^{k,+}\r]-(\gamma I+\Sigma_{h}^{k})^{-1})\bbE_{k}\l[\phi_{h}(s_{h}^{k},a_{h}^{k})L_{h}^{j}\r]}_2 \\
    	  & \leq H\norm{(\hat{\Sigma}_{h}^{k,+}-(\gamma I+\Sigma_{h}^{k})^{-1})}_{op}\bbE_{k}\l[\norm{\phi_{h}(s_{h}^{k},a_{h}^{k})}_2 \r] \leq 2 H \epsilon                                                                              ,          
    	  \label{eq:linear Q op-norm use}
    \end{align}
    where the first inequality is by Cauchy–Schwarz, the second inequality is by \cref{ass:linear-Q} and the last is by event $E^\Sigma$ and \cref{ass:linear-Q}.
    In the same way we  have,
    \[
    	\gamma\phi_h(s,a)^{\top}(\gamma I + \Sigma_{h}^{k})^{-1} \theta_{h}^{k}+(1 - r_{h}^{k}(s,a)) \phi_h(s,a)^{\top}(\gamma I + \Sigma_{h}^{k})^{-1} \Sigma_{h}^{k} \theta_{h}^{k}
    	-
    	\bbE_{k}\l[Q_{h}^{k}(s,a) - \hat{Q}_{h}^{k}(s,a) \r] \le \calO(\epsilon H). \qedhere
    \]
\end{proof}

\subsection{Bound on $\textsc{Bias}_2$}
\begin{lemma}
\label{lemma:linear Q bias2}
    Under the good event,
    \begin{align*}
        \textsc{Bias}_2 
        \leq
        \sum_{k=1}^K \sum_{h} \bbE_{s , a \sim q_{h}^{\pi^{*}}} \l[ {b}_{h}^{k,2}(s,a) + b_h^{k,r}(s,a) + {b}_{h}^{k,g}(s,a) \r]  
        + \calO \l( \frac{H^2}{\gamma}\logterm
        + H \sqrt{\gamma \dim}K \epsilon + \gamma H^2 K \epsilon \r).
        \end{align*}
\end{lemma}

\begin{proof}
Similarly to $\textsc{Bias}_{1}$, we first decompose,
\begin{align*}
	\textsc{Bias}_{2} & =\sum_{k=1}^{K}\sum_{h}\bbE_{s\sim q_{h}^{*}}  \l[  \l\langle \pi_{h}^{*}(\cdot\mid s), \hat{Q}_{h}^{k}(s,\cdot) - Q_{h}^{k}(s,\cdot) \r\rangle  \r]                                  \\
	                  & =\underbrace{
	                  \sum_{k=1}^{K}\sum_{h}\bbE_{s\sim q_{h}^{*}}  \l[  \l\langle \pi_{h}^{*}(\cdot\mid s),\bbE_{k}  \l[\hat{Q}_{h}^{k}(s,\cdot) - Q_{h}^{k}(s,\cdot) \r] \r\rangle  \r] }_{(A)}           + \underbrace{
	                  \sum_{k=1}^{K}\sum_{h}\bbE_{s\sim q_{h}^{*}}  \l[  \l\langle \pi_{h}^{*}(\cdot\mid s), \hat{Q}_{h}^{k}(s,\cdot) -  \bbE_{k}  \l[\hat{Q}_{h}^{k}(s,\cdot) \r] \r\rangle  \r] }_{(B)},
\end{align*}
where $(B)\leq \sum_{k,h} \bbE_{s , a \sim q_{h}^{*}} [b_{h}^{k,g}(s , a)] 
            + \frac{H^{2}}{\gamma} \logterm + {\cal O}(\gamma H^2 K \epsilon)$ by event $E^g$, and
for $(A)$ we use again \cref{lemma:linear Q bias equality},
\begin{align*}
	(A) & \leq\sum_{k=1}^{K} \sum_{h,a} \bbE_{s \sim q_{h}^{\pi^{*}}} 
	                        \l[ \gamma\pi_{h}^{*}(a\mid s)| \phi_h(s,a)^{\top} (\gamma I + \Sigma_{h}^{k})^{-1} \theta_{h}^{k}| \r]
	                  \\
	                  & \qquad + \sum_{k=1}^{K} \sum_{h,a} \bbE_{s \sim q_{h}^{\pi^{*}}} \l[ \pi_{h}^{*}(a\mid s)(r_{h}^{k}(s,a) - 1) 
	                            \phi_h(s,a)^{\top} (\gamma I + \Sigma_{h}^{k})^{-1} \Sigma_{h}^{k} \theta_{h}^{k} \r] + \calO (\epsilon H^{2} K)
	                  \\
	                  & =\sum_{k=1}^{K} \underbrace{\sum_{h,a} \bbE_{s \sim q_{h}^{\pi^{*}}} \l[ \gamma r_{h}^{k}(s,a) 
	                        \pi_{h}^{*}(a\mid s)| \phi_h(s,a)^{\top} (\gamma I + \Sigma_{h}^{k})^{-1} \theta_{h}^{k}| \r]}_{(i)}
	                  \\
	                  & \qquad + \sum_{k=1}^{K} \underbrace{\sum_{h,a} \bbE_{s \sim q_{h}^{\pi^{*}}} 
	                            \l[ \gamma\pi_{h}^{*}(a\mid s)(1 - r_{h}^{k}(s,a))| \phi_h(s,a)^{\top} (\gamma I + \Sigma_{h}^{k})^{-1} \theta_{h}^{k}| \r]}_{(ii)}
	                  \\
	                  & \qquad + \sum_{k=1}^{K} \underbrace{\sum_{h,a} \bbE_{s \sim q_{h}^{\pi^{*}}} 
	                            \l[ \pi_{h}^{*}(a\mid s)(r_{h}^{k}(s,a) - 1) \phi_h(s,a)^{\top} (\gamma I + \Sigma_{h}^{k})^{-1} \Sigma_{h}^{k} \theta_{h}^{k} \r]}_{(iii)} + \calO (\epsilon H^{2} K) 
\end{align*}
Similar to $(i)$ in \cref{lemma: linear Q bias1}
\begin{align*}
	(i) & \leq\sum_{h,a} \bbE_{s \sim q_{h}^{\pi^{*}}} \l[ \gamma r_{h}^{k}(s,a) \pi_{h}^{*}(a\mid s) 
	            \Vert \phi_h(s,a) \Vert_{ (\gamma I + \Sigma_{h}^{k})^{-1}} \Vert \theta_{h}^{k} \Vert_{ (\gamma I + \Sigma_{h}^{k})^{-1}} \r] 
        \\
	    & \leq H\sqrt{\gamma \dim} \sum_{h,a} \bbE_{s \sim q_{h}^{\pi^{*}}} 
	           \l[r_{h}^{k}(s,a) \pi_{h}^{*}(a\mid s) \Vert \phi_h(s,a) \Vert_{ (\gamma I + \Sigma_{h}^{k})^{-1}} \r]
        \\
	    & \leq H\sqrt{\gamma \dim} \sum_{h} \bbE_{s \sim q_{h}^{\pi^{*}}} \l[ \sum_{a} 
	            \pi_{h}^{*}(a\mid s)r_{h}^{k}(s,a) \Vert \phi_h(s,a) \Vert_{\hat{\Sigma}_{h}^{k,+}} \r] + \calO (\epsilon H\sqrt{\gamma \dim})
	    \\
	    & =\sum_{h} \bbE_{s,a \sim q_{h}^{\pi^{*}}} \l[ {b}_{h}^{k,2}(s,a) \r] 
                + \calO (\epsilon H \sqrt{\gamma \dim}).        
\end{align*}
Using \cref{eq: cdot leq 1/gamma} and \cref{eq: cdot leq 1},
\begin{align*}
	(ii) & \leq H\sqrt{\dim} \sum_{h,a} \bbE_{s \sim q_{h}^{\pi^{*}}} \l[ \pi_{h}^{*}(a\mid s)(1 - r_{h}^{k}(s,a)) \r]
	     =\frac{1}{2} \sum_h \bbE_{s,a \sim q_{h}^{\pi^{*}}} \l[  b_h^{k,r}(s,a) \r]
	     \\
	(iii) & \leq H\sqrt{\dim} \sum_{h,a} \bbE_{s \sim q_{h}^{\pi^{*}}} \l[ \pi_{h}^{*}(a\mid s)(1 - r_{h}^{k}(s,a)) \r]
	   =\frac{1}{2} \sum_h \bbE_{s,a \sim q_{h}^{\pi^{*}}} \l[ b_h^{k,r}(s,a) \r]. \qedhere
\end{align*}
\end{proof}

\subsection{Bound on $\textsc{Reg}$}
\begin{lemma}
    \label{lemma: linear Q reg}
Under the good event, for $\eta < \frac{\gamma}{10 H \dmax}$,
\[
	\textsc{Reg} \leq   \sum_{k=1}^K\sum_{h} \bbE_{s\sim q_{h}^{*}}\l[b_{h}^{k,v}(s) \r]
	                    + \calO \l( \eta H^{5} \dim K
	                    + \frac{H \log A}{\eta} 
	                    + \frac{H^2 \logterm}{\gamma} + \eta H^3 \epsilon \r). 
\]

\end{lemma}

\begin{proof}
Note that $\eta (\hat Q^k_h(s,a) - \hat B^k_h(s,a)) \ge -1$. Thus, using lemma \cref{lemma:delayed-exp-weights-regret eta Q < 1} for each $(h,s)$,
\begin{align}
    \nonumber
	\sum_{k=1}^{K} & \l\langle \pi_{h}^{k+d^{k}}(\cdot\mid s) 
	- \pi_{h}^{*}(\cdot\mid s) ,\hat{Q}_{h}^{k}(s,\cdot) - \hat{B}_{h}^{k}(s,\cdot)\r\rangle \le
	\\
	\nonumber
	& \leq \frac{\log A}{\eta} + \eta\sum_{k=1}^{K}  \sum_{a}\pi^{k+d^{k}}(a\mid s)m^{k+d^{k}} \l( \hat{Q}_{h}^{k}(s,a) - \hat{B}_{h}^{k}(s,a)\r)^{2} 
	\\
    & \leq \frac{\log A}{\eta} + 2\eta\sum_{k=1}^{K} \sum_{a}\pi^{k+d^{k}}(a\mid s)m^{k+d^{k}} \l( \hat{Q}_{h}^{k}(s,a)\r)^{2} + 2\eta\sum_{k=1}^{K} 
                    \sum_{a}\pi^{k+d^{k}}(a\mid s)m^{k+d^{k}} \l( \hat{B}_{h}^{k}(s,a)\r)^{2}.                             \label{eq:linear Q reg first}
\end{align}
Note that $b_{h}^{k}(s,a)\leq \calO(H\sqrt{\dim})$. Thus 
$\hat{B}_{h}^{k}(s,a) \leq \calO (H^2 \sqrt{\dim})$, and the last sum can be bounded by,
\begin{align}
	\sum_{k,a}\pi^{k+d^{k}}(a\mid s)m^{k+d^{k}} \l( \hat{B}_{h}^{k}(s,a)\r)^{2} 
	& \le \calO ( H^{4} \dim) \sum_{k,a}\pi^{k+d^{k}}(a\mid s)m^{k+d^{k}} = \calO (H^{4} \dim) \sum_{k=1}^{K}m^{k+d^{k}} \leq \calO (H^{4} \dim K).      
	                                                                                      \label{eq:linear Q reg B}
\end{align}
Taking the expectation with respect to $s\sim q^*_h$ on the first sum in \eqref{eq:linear Q reg first}, by event $E^b$,
\begin{align}
\nonumber
        & 2\eta \sum_{k=1}^{K} \sum_{a} \bbE_{s \sim q_{h}^*} \l[ \pi^{k+d^{k}}(a\mid s) m^{k+d^{k}} \l( \hat{Q}_{h}^{k}(s,a) \r)^{2} \r] \le
        \\
        & \qquad
        \leq
        4 \eta \sum_{k=1}^{K} \sum_{a} \bbE_{s \sim q_{h}^*} \l[  \pi^{k+d^{k}} (a\mid s) m^{k+d^{k}} \bbE_{k} \l[ \l(\hat{Q}_{h}^{k}(s,a)\r)^{2} \r]\r] + \frac{8\eta H^2 \dmax\logterm}{\gamma^{2}} 
        \nonumber
        \\
        & \qquad
        \leq
        4 \eta \sum_{k=1}^{K} \sum_{a} \bbE_{s \sim q_{h}^*} \l[  \pi^{k+d^{k}} (a\mid s) m^{k+d^{k}} \bbE_{k} \l[ \l(\hat{Q}_{h}^{k}(s,a)\r)^{2} \r]\r] + \frac{H \logterm}{\gamma},
        \label{eq:linear Q reg good}
\end{align}
where the last is since $\eta < \frac{\gamma}{10 H \dmax}$.
Finally, by \cref{lemma:linear Q bound E[Q2]},
\begin{align}
	4\eta \sum_{a}\pi^{k+d^{k}}(a\mid s)m^{k+d^{k}} \bbE_{k}\l[\l(\hat{Q}_{h}^{k}(s,a)\r)^{2}\r] & \leq 4 \eta H^{2} \sum_{a}\pi^{k+d^{k}}(a\mid s)m^{k+d^{k}}r_{h}^{k}(s,a)
                                                                                                        \norm{\phi_h(s,a)}_{\hat{\Sigma}_{h}^{k,+}}^{2} 
                                                                                                            + \calO(\eta H^{2} \epsilon), 
                                                                                                \nonumber \\
                                                                                                & =b_{h}^{k,v}(s) + \calO(\eta H^{2} \epsilon). \label{eq:linear Q reg bonus}         
\end{align}
Combining the above with \cref{eq:linear Q reg first,eq:linear Q reg B,eq:linear Q reg bonus,eq:linear Q reg good}, summing over $h$ and taking the expectation completes the proof.
\end{proof}

\begin{lemma}
    \label{lemma:linear Q bound E[Q2]}    
    Under event $E^\Sigma$, for every $(k,h,s,a) \in [K] \times [H] \times \calS \times \calA$, it holds that
    \[
        \bbE_{k}\l[\l(\hat{Q}_{h}^{k}(s,a)\r)^{2}\r]
        \leq 
        H^2 r_h^k(s,a) \norm{ \phi_h(s,a) }^2_{\hat \Sigma_h^{k,+}} + \calO(H^2 \epsilon).
    \]
\end{lemma}
\begin{proof}
By the definition of $\hat Q^k_h(s,a)$ and $\hat \theta^k_h$,
\begin{align}
    \nonumber
	\bbE_{k}\l[\l(\hat{Q}_{h}^{k}(s,a)\r)^{2}\r] & =\bbE_{k}\l[\l(r_{h}^{k}(s,a)\r)^{2}\phi_h(s,a)^{\top}\hat{\theta}_{h}^{k}\hat{\theta}_{h}^{k\top}\phi_h(s,a)\r]
	                                            \\
	                                            \nonumber
	                                             & \leq H^{2} \bbE_{k}\l[\l(r_{h}^{k}(s,a)\r)^{2}\phi_h(s,a)^{\top}\hat{\Sigma}_{h}^{k,+}\phi_{h}(s_{h}^{k},a_{h}^{k})
	                                                    \phi_{h}(s_{h}^{k},a_{h}^{k})^{\top}\hat{\Sigma}_{h}^{k,+}\phi_h(s,a)\r] 
	                                                    \\
	                                                    \nonumber
	                                                    & = H^{2} \l(r_{h}^{k}(s,a)\r)^{2}\phi_h(s,a)^{\top}\hat{\Sigma}_{h}^{k,+} \bbE_{k}\l[ \phi_{h}(s_{h}^{k},a_{h}^{k})
	                                                    \phi_{h}(s_{h}^{k},a_{h}^{k})^{\top} \r]  \hat{\Sigma}_{h}^{k,+}\phi_h(s,a)
	                                             \\
	                                             \nonumber
	                                             & =H^{2} \l( r_{h}^{k}(s,a)\r)^{2}\phi_h(s,a)^{\top}\hat{\Sigma}_{h}^{k,+}\Sigma_{h}^{k}\hat{\Sigma}_{h}^{k,+}\phi_h(s,a)
	                                             \\
	                                             & \le H^{2}r_{h}^{k}(s,a)\phi_h(s,a)^{\top}\hat{\Sigma}_{h}^{k,+}\Sigma_{h}^{k}\hat{\Sigma}_{h}^{k,+}\phi_h(s,a).
	                                             \label{eq:linear Q Q2}
\end{align}
We rewrite,
\begin{align}
    \nonumber
	\phi_h(s,a)^{\top}\hat{\Sigma}_{h}^{k,+}\Sigma_{h}^{k}\hat{\Sigma}_{h}^{k,+}\phi_h(s,a) & =\underbrace{\phi_h(s,a)^{\top}\hat{\Sigma}_{h}^{k,+}\Sigma_{h}^{k} \l( \hat{\Sigma}_{h}^{k,+} - (\gamma I + 
	                                                                                        \Sigma_{h}^{k})^{-1}\r)\phi_h(s,a)}_{(i)}
	                                                                                    \\
	                                                                                    \nonumber
	                                                                                    & \qquad + \underbrace{\phi_h(s,a)^{\top} \l( \hat{\Sigma}_{h}^{k,+} - (\gamma I + \Sigma_{h}^{k})^{-1}\r)\Sigma_{h}^{k}(\gamma I + 
	                                                                                            \Sigma_{h}^{k})^{-1}\phi_h(s,a)}_{(ii)} 
                                                                                        \\
	                                                                                    & \qquad + \underbrace{\phi_h(s,a)^{\top}(\gamma I + \Sigma_{h}^{k})^{-1}\Sigma_{h}^{k}(\gamma I + \Sigma_{h}^{k})^{-1}\phi_h(s,a)}_{(iii)}. 
	                                                                                    \label{eq:linear Q Q2 bilinear}
\end{align}
We now bound each of the above as follows: 
\begin{align}
        \tag{Cauchy–Schwarz}
	(i) & \leq \norm{\phi_h(s,a)\hat{\Sigma}_{h}^{k,+}\Sigma_{h}^{k}}_2 \norm{\l(\hat{\Sigma}_{h}^{k,+} - (\gamma I + \Sigma_{h}^{k})^{-1}\r)\phi_h(s,a)}_2
	    \nonumber\\         \nonumber
	    & \leq \norm{\phi_h(s,a)}_2 \norm{\hat{\Sigma}_{h}^{k,+}\Sigma_{h}^{k}}_{op} \norm{\l(\hat{\Sigma}_{h}^{k,+} - (\gamma I + \Sigma_{h}^{k})^{-1}\r)}_{op} \norm{\phi_h(s,a)}_2 
	    \\         \nonumber
	    \tag{$\norm{\phi_h(s,a)}\leq 1$}
	    & \leq \norm{\hat{\Sigma}_{h}^{k,+}\Sigma_{h}^{k}}_{op} \norm{\l(\hat{\Sigma}_{h}^{k,+} - (\gamma I + \Sigma_{h}^{k})^{-1}\r)}_{op}
	    \\
	    & \leq(1 + \epsilon)2\epsilon    ,    
	    \label{eq:linear Q Q2 i}
\end{align}
where the last inequality is by event $E^\Sigma$.
Similarly,
\begin{align}
	(ii) & \leq \norm{\phi_h(s,a)}_2 \norm{\l(\hat{\Sigma}_{h}^{k,+} - (\gamma I + \Sigma_{h}^{k})^{-1}\r)\Sigma_{h}^{k}(\gamma I + \Sigma_{h}^{k})^{-1}\phi_h(s,a)}_2 
	    \nonumber
	     \\         \nonumber
	     & \leq \norm{\l(\hat{\Sigma}_{h}^{k,+} - (\gamma I + \Sigma_{h}^{k})^{-1}\r)}_{op} \norm{\Sigma_{h}^{k}(\gamma I + \Sigma_{h}^{k})^{-1}\phi_h(s,a)}_2
	     \\         \nonumber
	     & \leq 2\epsilon\norm{\Sigma_{h}^{k}(\gamma I + \Sigma_{h}^{k})^{-1}\phi_h(s,a)}_2
	     \\
	     & \leq 2\epsilon\norm{\phi_h(s,a)}_2 \leq 2\epsilon     
	     \label{eq:linear Q Q2 ii},
\end{align}
where the third inequality is by event $E^\Sigma$ and the forth inequality uses \cref{eq:PSDs invA A} in \cref{lemma:PSDs}. Finally, using \cref{eq:PSDs invA A invA} in \cref{lemma:PSDs},
\begin{align}
	(iii) & \leq \phi_h(s,a)^{\top}(\gamma I + \Sigma_{h}^{k})^{-1}\phi_h(s,a)
            \nonumber\\         \nonumber
	      & =\phi_h(s,a)^{\top} \l( (\gamma I + \Sigma_{h}^{k})^{-1} - \hat{\Sigma}_{h}^{k,+}\r)\phi_h(s,a) + \norm{\phi_h(s,a)}_{\hat{\Sigma}_{h}^{k,+}}^{2}
	      \\         \nonumber
	      & \leq \norm{\phi_h(s,a)}_2 \norm{\l((\gamma I + \Sigma_{h}^{k})^{-1} - \hat{\Sigma}_{h}^{k,+}\r)\phi_h(s,a)}_2 + \norm{\phi_h(s,a)}_{\hat{\Sigma}_{h}^{k,+}}^{2} 
	      \\         \nonumber
	      & \leq \norm{\l((\gamma I + \Sigma_{h}^{k})^{-1} - \hat{\Sigma}_{h}^{k,+}\r)}_{op} + \norm{\phi_h(s,a)}_{\hat{\Sigma}_{h}^{k,+}}^{2}                     
	      \\
	      & \leq 2\epsilon + \norm{\phi_h(s,a)}_{\hat{\Sigma}_{h}^{k,+}}^{2}.      
	      \label{eq:linear Q Q2 iii}
\end{align}
Combining \cref{eq:linear Q Q2,eq:linear Q Q2 bilinear,eq:linear Q Q2 i,eq:linear Q Q2 ii,eq:linear Q Q2 iii} completes the proof.
\end{proof}

\subsection{Bound on $\textsc{Drift}$}
\label{sec:drift-bound-linear-Q}

\begin{lemma}
    \label{lemma: linear Q drift}
    If $\gamma \leq \frac{1}{H\sqrt{\dim}}$ and $\eta \leq \frac{\gamma \sqrt{\dim}}{4H\dmax}$ then,
    $
    	\textsc{Drift} 
    	\le
	    \calO \l(\frac{\eta}{\gamma}H^{4}\sqrt{\dim}(K+D)\r).
	 $
\end{lemma}

\begin{proof}
    Note that if $\eta < \frac{\gamma \sqrt{\dim}}{4H\dmax}$ then $b_h^k(s,a)\le \calO(H\sqrt{\dim})$ and thus $B_h^k(s,a)\le \calO (H^2 \sqrt{\dim})$. Now, using \cref{lemma:linear Q l1 diff},
    \begin{align*}
    	\textsc{Drift} 
    	               & = \sum_{k=1}^{K}\sum_{h}\bbE_{s\sim q_{h}^{\pi^{*}}}\l[\l\langle\pi_{h}^{k}(\cdot\mid s) - \pi_{h}^{k+d^{k}}(\cdot\mid s),Q_{h}^{k}(s,\cdot) - \hat B_{h}^{k}(s,\cdot)\r\rangle\r] \\
    	               & \leq \calO (H^2\sqrt{\dim}) \sum_{k=1}^{K}\sum_{h}\bbE_{s\sim q_{h}^{\pi^{*}}} \l[\Vert \pi_{h}^{k}(\cdot\mid s)-\pi_{h}^{k+d^{k}}(\cdot\mid s) \Vert_{1}\r]                                                 \leq \calO \l( \frac{\eta}{\gamma}H^{4} \sqrt{\dim} (K+D) \r). \qedhere                            
    \end{align*}
\end{proof}

\begin{lemma}
\label{lemma:linear Q l1 diff}
    If $\gamma \leq \frac{1}{H\sqrt{\dim}}$ and $\eta \leq \frac{\gamma \sqrt{\dim}}{4H\dmax}$ then for each $(h,s)$:
    $
    \sum_{k=1}^K \Vert \pi_{h}^{k+d^{k}}(\cdot\mid s)-\pi_{h}^{k}(\cdot\mid s) \Vert_{1}
    \le
    \calO (\frac{\eta}{\gamma}H(K + D)).
    $
\end{lemma}

\begin{proof}
We apply lemma \cref{lemma:elementwise diff} for each $(k,h,s,a)$ with $\ell(a) = \sum_{j:k\leq j+d^{j} < k+d^{k}}(\hat{Q}_{h}^{j}(s,a)-\hat{B}_{h}^{j}(s,a))$
and $M=\sum_{j:k\leq j+d^{j} < k+d^{k}}\frac{10H}{\gamma}$ since $|\hat Q_h^k(s,a)| \leq \frac{H}{\gamma}$ by \cref{lemma: linear Q < 1/gamma}, and  $ b_h^k(s,a) \leq 6 H \sqrt{\dim} $ whenever 
$\eta \leq \frac{\gamma \sqrt{\dim}}{4H\dmax}$ which implies that $\hat B_h^k(s,a) \leq 6 H^2 \sqrt{\dim} \leq \frac{6H}{\gamma}$.
We get that, 
\begin{align*}
	\Vert \pi_{h}^{k+d^{k}}(\cdot\mid s)-\pi_{h}^{k}(\cdot\mid s) \Vert_{1} & =\sum_{a}|\pi_{h}^{k+d^{k}}(a\mid s)-\pi_{h}^{k}(a\mid s)|                                                                                                                    \\
	                                                                   & \leq\eta\sum_{a}\pi_{h}^{k+d^{k}}(a\mid s)\sum_{a'}\pi_{h}^{k}(a'\mid s)\sum_{j:k\leq j+d^{j} < k+d^{k}}\l(|\hat{Q}_{h}^{j}(s,a')|+|\hat{B}_{h}^{j}(s,a')|+\frac{10H}{\gamma}\r) \\
	                                                                   & \qquad+\eta\sum_{a}\pi_{h}^{k}(a\mid s)\sum_{j:k\leq j+d^{j} < k+d^{k}}\l(|\hat{Q}_{h}^{j}(s,a)|+|\hat{B}_{h}^{j}(s,a)|+\frac{10H}{\gamma}\r)                                    \\
	                                                                   & \leq\eta\sum_{a}\pi_{h}^{k+d^{k}}(a\mid s)\sum_{a'}\pi_{h}^{k}(a'\mid s)\sum_{j:k\leq j+d^{j} < k+d^{k}}\frac{20 H}{\gamma} +\eta\sum_{a}\pi_{h}^{k}(a\mid s)\sum_{j:k\leq j+d^{j} < k+d^{k}}\frac{20 H}{\gamma}                                                                                        \\
	                                                                   & =\eta\sum_{j:k\leq j+d^{j} < k+d^{k}}\frac{40 H}{\gamma}                                                                      .                                                   
\end{align*}
Summing the above over $k$ and applying \cref{lemma:sum-delayed-indicators} completes the
proof.
\end{proof}

\subsection{Bound on $\textsc{Bonus}$}
\begin{lemma}
\label{lemma: linear Q bonus}
Under the good event, 
    \[
        \textsc{Bonus} 
        \le
        - \sum_{k=1}^K \sum_{h=1}^H \bbE_{s,a \sim q_{h}^{\pi^{*}}}\l[ b_{h}^{k}(s,a)\r]
        + \calO \l( \sqrt{\gamma}H^2 \dim K 
        + \eta H^2 \dim K
        + \frac{\eta }{\gamma} H^3 \sqrt{ \dim} (K+D)
        + H^{3} \sqrt{\dim K} \logterm \r).
    \]
\end{lemma}

\begin{proof}
    Under event $E^B$,
    \begin{align*}
    	\sum_{h,k} \bbE_{s\sim q_{h}^{\pi^{*}}}\l[\l\langle\pi_{h}^{k}(\cdot\mid s) 
    	- \pi_{h}^{*}(\cdot\mid s),\hat{B}_{h}^{k}(s,\cdot)\r\rangle\r] & 
    	\leq 
    	\sum_{h,k}  \bbE_{k} \l[ \bbE_{s\sim q_{h}^{\pi^{*}}} \l[\l\langle\pi_{h}^{k}(\cdot\mid s) - \pi_{h}^{*}(\cdot\mid s) , \hat B_{h}^{k}(s,\cdot)\r\rangle\r] \r]
    	+ \calO (H^{3} \sqrt{\dim K} \logterm) \\
    	                                                                & =\sum_{h,k} \bbE_{s,a\sim q_{h}^{\pi^{k}}}\l[b_{h}^{k}(s,a)\r]- \sum_{h,k} \bbE_{s,a\sim q_{h}^{\pi^{*}}}\l[b_{h}^{k}(s,a)\r]+ \calO (H^{3} \sqrt{\dim K} \logterm),
    \end{align*}
    where the equality is by \cref{lemma: value difference expectation}. Recall that $b^k_h(s,a) = b^{k,1}_h(s) + b^{k,2}_h(s,a) + b_h^{k,v}(s) + b^{k,r}_h(s,a) + b_h^{k,f}(s) + b^{k,g}_h(s,a)$. Now, the expectation $\bbE_{s,a \sim q_{h}^{\pi^{k}}}[\cdot]$ of each of the bonus terms  is bounded in \cref{lemma:linear Q bonus 1,lemma:linear Q bonus 2,lemma:linear Q bonus r,lemma:linear Q bonus v,lemma:linear Q bonus f,lemma:linear Q bonus g}.
\end{proof}

\begin{lemma}
\label{lemma: value difference expectation}
    For any $k$, let $b_{h}^{k}(s,a)$ be a loss function determined
    by the history $\tilde{H}^{k+d^k}$, and let $\hat B_{h}^{k}(s,a)$ be a randomized
    bonus function such that, for every $(h,s,a) \in [H] \times \calS \times \calA$,
    \[
    	\bbE_{k}[\hat{B}_{h}^{k}(s,a)]=b_{h}^{k}(s,a)+\bbE_{s'\sim P_{h}(\cdot\mid s,a)}\bbE_{a'\sim\pi_{h+1}^{k}(\cdot\mid s')}\bbE_{k}[\hat{B}_{h+1}^{k}(s',a')]
    \]
    Then,
    $
    	\bbE_{k}\l[\sum_{h}\bbE_{s\sim q_{h}^{*}}\l[\l\langle \pi_{h}^{k}(\cdot\mid s)-\pi_{h}^{*}(\cdot\mid s),\hat{B}_{h}^{k}(s,\cdot)\r\rangle \r]\r] =\sum_{h}\bbE_{s,a\sim q_{h}^{\pi^{k}}}[b_{h}^{k}(s,a)]-\sum_{h}\bbE_{s,a\sim q_{h}^{*}}[b_{h}^{k}(s,a)]. 
    $
\end{lemma}

\begin{proof}

Fix some $h\in[H]$,
\begin{align*}
	  \bbE_{k}\Bigl[[\bbE_{s\sim q_{h}^{*}}\bigl[\langle & \pi_{h}^{k}(\cdot\mid s)  -\pi_{h}^{*}(\cdot\mid s),\hat{B}_{h}^{k}(s,\cdot)\rangle \bigr]\Bigr]
	  =\bbE_{k} \l[ \bbE_{s\sim q_{h}^{*}}\l[\sum_{a}\pi_{h}^{k}(a\mid s)\hat{B}_{h}^{k}(s,a)\r] \r] -\bbE_{s,a\sim q_{h}^{*}} \l[ \bbE_{k}[\hat{B}_{h}^{k}(s,a)]\r]
	  \\
	  & =\bbE_{k} \l[ \bbE_{s\sim q_{h}^{*}}\l[\sum_{a}\pi_{h}^{k}(a\mid s)\hat{B}_{h}^{k}(s,a)\r]  \r]                                                                                                  -\bbE_{s,a\sim q_{h}^{*}}\l[b_{h}^{k}(s,a)+\bbE_{s'\sim P_{h}(\cdot\mid s,a)}\bbE_{a'\sim\pi_{h+1}^{k}(\cdot\mid s')}\bbE_{k}[\hat{B}_{h+1}^{k}(s',a')]\r]                                       \\
	  & =\bbE_{k}\bbE_{s\sim q_{h}^{*}}\l[\sum_{a}\pi_{h}^{k}(a\mid s)\hat{B}_{h}^{k}(s,a)\r]-\bbE_{s'\sim q_{h+1}^{*}}\l[\bbE_{a'\sim\pi_{h+1}^{k}(\cdot\mid s')}\bbE_{k}[\hat{B}_{h+1}^{k}(s',a')]\r]  -\bbE_{s,a\sim q_{h}^{*}}\l[b_{h}^{k}(s,a)\r]                                                                                                                                                    \\
	  & =\bbE_{k}\bbE_{s\sim q_{h}^{*}}\l[\sum_{a}\pi_{h}^{k}(a\mid s)\hat{B}_{h}^{k}(s,a)\r]-\bbE_{k}\bbE_{s\sim q_{h+1}^{*}}\l[\sum_{a}\pi_{h+1}^{k}(a\mid s)\hat{B}_{h+1}^{k}(s,a)\r]           -\bbE_{s,a\sim q_{h}^{*}}\l[b_{h}^{k}(s,a)\r]    
\end{align*}
Summing over $h$ we get,
\begin{align*}
	\sum_h \bbE_{k}\l[\bbE_{s\sim q_{h}^{*}}\l[\l\langle \pi_{h}^{k}(\cdot\mid s)-\pi_{h}^{*}(\cdot\mid s),\hat{B}_{h}^{k}(s,\cdot)\r\rangle \r]\r] & =\bbE_{k}\bbE_{s\sim q_{1}^{*}}\l[\sum_{a}\pi_{1}^{k}(a\mid s)\hat{B}_{1}^{k}(s,a)\r]-\sum_{h}\bbE_{s,a\sim q_{h}^{*}}[b_{h}^{k}(s,a)] \\
	                                                                                                                                                              & =\bbE_{k}\bbE_{s,a\sim q_{1}^{\pi^{k}}}\bbE_{k}\l[\hat{B}_{1}^{k}(s,a)\r]-\sum_{h}\bbE_{s,a\sim q_{h}^{*}}[b_{h}^{k}(s,a)].            
\end{align*}
For last,
\begin{align*}
	\bbE_{s,a\sim q_{1}^{\pi^{k}}}\bbE_{k}\l[\hat{B}_{1}^{k}(s,a)\r] & =\bbE_{s,a\sim q_{1}^{\pi^{k}}}\l[b_{1}^{k}(s,a)+\bbE_{s'\sim P_{1}(\cdot\mid s,a)}\bbE_{a'\sim\pi_{2}^{k}(\cdot\mid s')}\bbE_{k}[\hat{B}_{2}^{k}(s',a')]\r]                                 \\
	                                                                        & =\bbE_{s,a\sim q_{1}^{\pi^{k}}}\l[b_{1}^{k}(s,a)\r]+\bbE_{s,a\sim q_{1}^{\pi^{k}}}\bbE_{s'\sim P_{1}(\cdot\mid s,a)}\bbE_{a'\sim\pi_{2}^{k}(\cdot\mid s')}\bbE_{k}[\hat{B}_{h+1}^{k}(s',a')] \\
	                                                                        & =\bbE_{s,a\sim q_{1}^{\pi^{k}}}\l[b_{1}^{k}(s,a)\r]+\bbE_{s,a\sim q_{2}^{\pi^{k}}}\bbE_{k}\l[\hat{B}_{2}^{k}(s,a)\r] = \dots = \sum_{h=1}^{H}\bbE_{s,a\sim q_{h}^{\pi^{k}}}\l[b_{h}^{k}(s,a)\r]. \qedhere
\end{align*}
\end{proof}

\begin{lemma}
\label{lemma:linear Q bonus 1}
    Under the good event,
    $
        \sum_{k=1}^K\sum_{h=1}^H \bbE_{s\sim q_{h}^{k}}\l[b_{h}^{k,1}(s)\r] 
        \leq 
        \calO \l( \sqrt{\gamma (1 + \epsilon)} H^2\dim  K \r).
    $
\end{lemma}

\begin{proof}
    For any $k$ and $h$,
    \begin{align*}
    	\bbE_{s\sim q_{h}^{k}}\l[b_{h}^{k,1}(s)\r] & =\beta_{1}\sum_{a}\bbE_{s\sim q_{h}^{k}}\l[r_{h}^{k}(s,a)\pi_{h}^{k+d^{k}}(a\mid s)\Vert\phi_h(s,a)\Vert_{\hat{\Sigma}_{h}^{k,+}}\r] \leq\beta_{1}\sum_{a}\bbE_{s\sim q_{h}^{k}}\l[\pi_{h}^{k}(a\mid s)\Vert\phi_h(s,a)\Vert_{\hat{\Sigma}_{h}^{k,+}}\r]                  \\
    	                                           & =\beta_{1}\bbE_{(s,a)\sim q_{h}^{k}}\l[\sqrt{\phi_h(s,a)^{\top}\hat{\Sigma}_{h}^{k,+}\phi_h(s,a)}\r]        \leq\beta_{1}\sqrt{\bbE_{(s,a)\sim q_{h}^{k}}\l[\phi_h(s,a)^{\top}\hat{\Sigma}_{h}^{k,+}\phi_h(s,a)\r]}                                \\
    	                                           & =\beta_{1}\sqrt{\bbE_{(s,a)\sim q_{h}^{k}}\l[\trace\l(\phi_h(s,a)^{\top}\hat{\Sigma}_{h}^{k,+}\phi_h(s,a)\r)\r]} =\beta_{1}\sqrt{\bbE_{(s,a)\sim q_{h}^{k}}\l[\trace\l(\hat{\Sigma}_{h}^{k,+}\phi_h(s,a)\phi_h(s,a)^{\top}\r)\r]}                       \\
    	                                           & =\beta_{1}\sqrt{\trace\l(\hat{\Sigma}_{h}^{k,+}\bbE_{(s,a)\sim q_{h}^{k}}\l[\phi_h(s,a)\phi_h(s,a)^{\top}\r]\r)}  =\beta_{1}\sqrt{\trace\l(\hat{\Sigma}_{h}^{k,+}\Sigma_{h}^{k}\r)}        ,                                                          
    \end{align*}
    where the second inequality is Jensen's inequality, and in the last three equalities we use the cyclic property of trace, the linearity of trace and the definition of $\Sigma_h^k$. Finally,
    \begin{align*}
    	\trace\l(\hat{\Sigma}_{h}^{k,+}\Sigma_{h}^{k}\r) & =\sum_{i=1}^{\dim}e_{i}^{\top}\hat{\Sigma}_{h}^{k,+}\Sigma_{h}^{k}e_{i} \leq\sum_{i=1}^{\dim}\norm{e_{i}}_2 \norm{\hat{\Sigma}_{h}^{k,+}\Sigma_{h}^{k}e_{i}}_2 \leq\sum_{i=1}^{\dim}\norm{\hat{\Sigma}_{h}^{k,+}\Sigma_{h}^{k}}_{op}  \leq\sum_{i=1}^{\dim}(1+2\epsilon)=\dim(1+2\epsilon), 
    \end{align*}
where the first inequality is Cauchy–Schwarz and the last inequality is by $E^\Sigma$.
\end{proof}

\begin{lemma}
\label{lemma:linear Q bonus 2}
    Under the good event,
    $
        \sum_{k=1}^K\sum_{h=1}^H \bbE_{s,a\sim q_{h}^{k}}\l[b_{h}^{k,2}(s,a)\r]
        \leq \calO \l( \sqrt{\gamma (1 + \epsilon)} H^2 \dim K \r).
    $
\end{lemma}

\begin{proof}
    For any $h$ and $k$,
    \begin{align*}
    	\bbE_{s,a\sim q_{h}^{k}}\l[b_{h}^{k,2}(s,a)\r] =\beta_{2}\bbE_{s,a\sim q_{h}^{k}}\l[r_{h}^{k}(s,a)\Vert\phi_h(s,a)\Vert_{\hat{\Sigma}_{h}^{k,+}}\r]
    	                                               \leq \beta_{2}\bbE_{s,a\sim q_{h}^{k}}\l[\Vert\phi_h(s,a)\Vert_{\hat{\Sigma}_{h}^{k,+}}\r]
    	                                                    \leq \beta_{2} \sqrt{\dim(1+2\epsilon)} ,
    \end{align*}
    where the last is as in the proof of \cref{lemma:linear Q bonus 1}.
\end{proof}

\begin{lemma}
\label{lemma:linear Q bonus r}
    Under the good event,
    $
        \sum_{k=1}^K\sum_{h=1}^H \bbE_{s,a\sim q_{h}^{k}}\l[b_{h}^{k,r}(s,a)\r]
        \le
        \calO \l( \frac{\eta }{\gamma} H^3 \sqrt{ \dim} (K+D) \r).
    $
\end{lemma}

\begin{proof}
For any $s$ and $h$,
    \begin{align*}
    	\bbE_{s,a\sim q_{h}^{k}}\l[b_{h}^{k,r}(s,a)\r] & =\beta_{r}\bbE_{s,a\sim q_{h}^{k}}\l[1-r_{h}^{k}(s,a)\r]                                                                                                                                 =\beta_{r}\bbE_{s\sim q_{h}^{k}}\l[\sum_{a}\pi_{h}^{k}(a\mid s)(1-r_{h}^{k}(s,a))\r]                                                                                                                           \\
    	                                               & =\beta_{r}\bbE_{s\sim q_{h}^{k}}\l[\sum_{a}\pi_{h}^{k}(a\mid s)\frac{\max\{\pi_{h}^{k}(a\mid s),\pi_{h}^{k+d^{k}}(a\mid s)\}-\pi_{h}^{k}(a\mid s)}{\max\{\pi_{h}^{k}(a\mid s),\pi_{h}^{k+d^{k}}(a\mid s)\}}\r] \\
    	                                               & \leq\beta_{r}\bbE_{s\sim q_{h}^{k}}\l[\sum_{a}\max\{\pi_{h}^{k}(a\mid s),\pi_{h}^{k+d^{k}}(a\mid s)\}-\pi_{h}^{k}(a\mid s)\r]                                                                                  \\
    	                                               & \leq\beta_{r}\bbE_{s\sim q_{h}^{k}}\l[\Vert \pi_{h}^{k+d^{k}}(\cdot\mid s)-\pi_{h}^{k}(\cdot\mid s) \Vert_{1}\r].
    \end{align*}
    
    Finally, taking the sum and applying  \cref{lemma:linear Q l1 diff} completes the proof.
\end{proof}

\begin{lemma}
\label{lemma:linear Q bonus v}
    Under the good event,
    $
        \sum_{k=1}^K\sum_{h=1}^H \bbE_{s,a\sim q_{h}^{k}}\l[b_{h}^{k,v}(s)\r]
        \leq
        \calO \l( \eta (1+2\epsilon) H^2 \dim K \r).
    $
\end{lemma}

\begin{proof}
For any $s$ and $h$,
\begin{align*}
	\bbE_{s \sim q_{h}^{k}}\l[b_{h}^{k,v}(s)\r] & =\beta_{v}m^{k+d^{k}}\bbE_{s\sim q_{h}^{k}}\l[\sum_{a}\pi^{k+d^{k}}(a\mid s)r_{h}^{k}(s,a)\norm{\phi_h(s,a)}_{\hat{\Sigma}_{h}^{k,+}}^{2}\r] \\
	                                             & \leq\beta_{v}m^{k+d^{k}}\bbE_{s\sim q_{h}^{k}}\l[\sum_{a}\pi^{k}(a\mid s)\norm{\phi_h(s,a)}_{\hat{\Sigma}_{h}^{k,+}}^{2}\r]                  \\
	                                             & =\beta_{v}m^{k+d^{k}}\bbE_{s,a\sim q_{h}^{k}}\l[\norm{\phi_h(s,a)}_{\hat{\Sigma}_{h}^{k,+}}^{2}\r]  \leq\beta_{v}m^{k+d^{k}}\dim(1+2\epsilon).                                                                                                 
\end{align*}
Summing over $h$ and $k$ and noting that $\sum_k m^{k+d^k} \leq K$ completes the proof.
\end{proof}

\begin{lemma}
\label{lemma:linear Q bonus f}
    Under the good event,
    $
        \sum_{k=1}^K\sum_{h=1}^H \bbE_{s,a\sim q_{h}^{k}}\l[b_{h}^{k,f}(s)\r]
        \leq
        \calO \l( \gamma (1 + 2\epsilon) H \dim K \r).
    $
\end{lemma}

\begin{proof}
Similarly to \cref{lemma:linear Q bonus v},
$
	\bbE_{s \sim q_{h}^{k}}\l[b_{h}^{k,f}(s)\r] 
	                               \leq\beta_{f} \dim(1 + 2\epsilon)                                                                                                  
$.
\end{proof}

\begin{lemma}
\label{lemma:linear Q bonus g}
    Under the good event,
    $
        \sum_{k=1}^K\sum_{h=1}^H \bbE_{s,a\sim q_{h}^{k}}\l[b_{h}^{k,g}(s,a)\r]
        \leq
        \calO \l( \gamma (1 + 2\epsilon)   H^2 \dim K   \r).
    $
\end{lemma}

\begin{proof}
Again, similarly to \cref{lemma:linear Q bonus v},
$
	\bbE_{s,a \sim q_{h}^{k}}\l[b_{h}^{k,g}(s,a)\r] 
	                                \leq\beta_{g} \dim(1 + 2\epsilon).
$
\end{proof}

\newpage
\section{Auxiliary Lemmas}
\label{app:auxil}

\begin{lemma}[Value Difference Lemma \cite{even2009online}]
\label{lemma: value diff}
    For any two policies $\pi$ and $\pi^*$,
    \[
        V^{\pi}_1(\sinit) - V^{\pi^*}_1(\sinit) = \sum_{h=1}^H \bbE_{s\sim q_h^*} \Big[ \l\langle \pi_h(\cdot \mid s) - \pi^*_h(\cdot \mid s), Q_h^\pi(s,\cdot) \r\rangle \Big].
    \]
\end{lemma}

\begin{lemma}[Azuma–Hoeffding inequality]
    \label{lemma:Azuma_Hoeffding}
    Let $\{ X_t \}_{t\geq 1}$ be a real valued martingale difference sequence adapted to a filtration $\calF_1 \subseteq \calF_2 \subseteq...$ (i.e., $\bbE[X_t \mid \calF_t] = 0$). 
    If $|X_t| \leq R$ a.s. then with probability at least $1-\delta$,
    \[
        \sum_{t=1}^T X_t 
        \leq 
         R \sqrt{T \ln\frac{1}{\delta}}.
    \]
\end{lemma}

\begin{lemma}[A special form of Freedman's Inequality, Theorem 1 of \citet{beygelzimer2011contextual}]
    \label{lemma:freedman}
    Let $\{ X_t \}_{t\geq 1}$ be a real valued martingale difference sequence adapted to a filtration $\calF_1 \subseteq \calF_2 \subseteq...$ (i.e., $\bbE[X_t \mid \calF_t] = 0$). 
    If $|X_t| \leq R$ a.s. then for any $\alpha \in (0,1/R), T \in \mathbb{N}$ it holds with probability at least $1-\delta$,
    \[
        \sum_{t=1}^T X_t 
        \leq 
        \alpha \sum_{t=1}^T \bbE[X_t^2 \mid \calF_{t}] + \frac{\log(1/\delta)}{\alpha}.
    \]
\end{lemma}

\begin{lemma}[Consequence of Freedman’s Inequality, e.g., Lemma E.2 in \cite{cohen2021minimax}]
    \label{lemma:cons-freedman}
     Let $\{ X_t \}_{t\geq 1}$ be a sequence of random variables, supported in $[0,R]$, and adapted to a filtration $\calF_1 \subseteq \calF_2 \subseteq...$. For any $T$, with probability $1-\delta$,
     \[
        \sum_{t=1}^T X_t \leq 2 \bbE[X_t \mid \calF_t] + 4R \log\frac{1}{\delta}.
     \]
     
\end{lemma}

\begin{lemma}[Lemma A.2 of \citet{luo2021policy}]
    \label{lemma:concentration}
    Given a filtration $\calF_0 \subseteq \calF_1 \subseteq \dots$, let $z^k_h(s,a)\in [0,R]$ and $\tilde q_h^k(s,a)\in [0,1]$ be  sequences of $\calF_k$-measurable functions. If $Z_h^k(s,a) \in [0,R]$ is a sequence of random variables such that $\bbE[Z_h^k(s,a)\mid \calF_k] = z_h^k(s,a)$ then with probability $1-\delta$,
    \[
        \sum_{k=1}^{K}\sum_{h,s,a}\frac{\mathbb{I} \{s_{h}^{k}=s,a_{h}^{k}=s\}Z_{h}^{k}(s,a)}{\tilde{q}_{h}^{k}(s,a) + \gamma}
        - \sum_{k=1}^{K}\sum_{h,s,a} \frac{q_{h}^{k}(s,a)z_{h}^{k}(s,a)}{\tilde{q}_{h}^{k}(s,a)} \leq \frac{RH}{2\gamma}\ln\frac{H}{\delta}
    \]
\end{lemma}

\begin{lemma}[Lemma 9 of \citet{thune2019nonstochastic}]
    \label{lemma:delayed-exp-weights-regret-positive}
    Let $\eta > 0$, variables delays $\{d^k\}_{k=1}^K$, and loss vectors $\ell^k \in [0,\infty)^A$ for all $k \in [K]$. Define,
    \[
        \pi^1(a) = \frac{1}{A}
        \quad ; \quad
        \pi^{k+1}(a) = \frac{\pi^k(a) e^{-\eta \sum_{j:j + d^j = k} \ell^j(a)}}{\sum_{a' \in \calA} \pi^k(a') e^{-\eta \sum_{j:j + d^j = k} \ell^j(a')}}.
    \]
    Then, for any $\pi^* \in \Delta_\calA$:
    \[
        \sum_{k=1}^K \sum_{a \in \calA} (\pi^{k + d^k}(a) - \pi^*(a)) \ell^k(a)
        \le
        \frac{\ln A}{\eta} + \eta \sum_{k=1}^K \sum_{a \in \calA} \pi^{k+d^k}(a) \ell^k(a)^2.
    \]
\end{lemma}

\begin{corollary}
    \label{corollary:delayed-exp-weights-regret}
    Let $\eta > 0$, variables delays $\{d^k\}_{k=1}^K$, and loss vectors $\ell^k \in [-M, \infty)^A$ for all $k \in [K]$. Define,
    \[
        \pi^1(a) = \frac{1}{A}
        \quad ; \quad
        \pi^{k+1}(a) = \frac{\pi^k(a) e^{-\eta \sum_{j:j + d^j = k} \ell^j(a)}}{\sum_{a' \in \calA} \pi^k(a') e^{-\eta \sum_{j:j + d^j = k} \ell^j(a')}}.
    \]
    Then, for any $\pi^* \in \Delta_\calA$:
    \[
        \sum_{k=1}^K \sum_{a \in \calA} (\pi^{k + d^k}(a) - \pi^*(a)) \ell^k(a)
        \le
        \frac{\ln A}{\eta} + 2 \eta \sum_{k=1}^K \sum_{a \in \calA} \pi^{k+d^k}(a) \ell^k(a)^2
        + 2\eta K M^2.
    \]
\end{corollary}
\begin{proof}
    Note that,
    \[
        \pi^1(a) = \frac{1}{A}
        \quad ; \quad
        \pi^{k+1}(a) = \frac{\pi^k(a) e^{-\eta \sum_{j:j + d^j = k} (\ell^j(a)+M)}}{\sum_{a' \in \calA} \pi^k(a') e^{-\eta \sum_{j:j + d^j = k} (\ell^j(a') + M)}}.
    \]
    The statement now follows immediately by applying \cref{lemma:delayed-exp-weights-regret-positive} on the losses $\ell^k + M$.
\end{proof}

\begin{lemma}
    \label{lemma:delayed-exp-weights-regret eta Q < 1}

Let $\eta > 0$, variables delays $\{d^k\}_{k= 1}^K$, and loss vectors $\ell^k \in [ - M, \infty)^A$ for all $k \in [K]$. 
Define,
\[
	\pi^{k}(a)= \frac{\exp\l( - \eta\sum_{j:j+d^{j}<k} \ell^{j}(a) \r)}{\sum_{a'} \exp\l( - \eta\sum_{j:j+d^{j}<k} \ell^{j}(a') \r)},
\]
where the empty sum is zero. If $\eta\sum_{j:j+d^{j}=k} \ell^{k}(a)> - 1$ for all $k\in[K]$, then,
\[
	\sum_{k= 1}^{K} \sum_{a} \l(\pi^{k+d^{k}}(a) - \pi^{*}(a) \r) \ell^{k}(a) \leq \frac{\ln A}{\eta} + \eta\sum_{k= 1}^{K} \sum_{a} \pi^{k+d^{k}}(a) m^{k+d^{k}} \l(\ell^{k}(a) \r)^{2},
\]
where $m^{k}=|\{j:j+d^{j}=k\}|$.
\end{lemma}

\begin{proof}
The proof is based in part on the proof of \citet[Lemma 9]{thune2019nonstochastic}.
Denote $W^{k}= \sum_{a} \exp\l( - \eta\sum_{j:j+d^{j}<k} \ell^{k}(a) \r)$. We have that
\begin{align*}
	\frac{W^{k+1}}{W^{k}} & = \frac{\sum_{a} \exp\l( - \eta\sum_{j:j+d^{j}<k} \ell^{j}(a) \r) 
	                            \exp\l( - \eta\sum_{j:j+d^{j}=k} \ell^{j}(a) \r)}{W^{k}}
	                      \\
	                      & = \sum_{a} \pi^{k}(a) \exp\l( - \eta\sum_{j:j+d^{j}=k} \ell^{j}(a) \r)
	                      \\
	                      & \leq \sum_{a} \pi^{k}(a) \l(1 - \eta\sum_{j:j+d^{j}=k} \ell^{j}(a)+\eta^{2} \l(\sum_{j:j+d^{j}=k} \ell^{j}(a) \r)^{2} \r)
	                      \\
	                      & \leq \sum_{a} \pi^{k}(a) \l(1 - \eta\sum_{j:j+d^{j}=k} \ell^{j}(a)+\eta^{2} \sum_{j:j+d^{j}=k}m^{k} \l(\ell^{j}(a) \r)^{2} \r)
	                      \\
	                      & = 1 - \eta\sum_{j:j+d^{j}=k} \sum_{a} \pi^{k}(a) \ell^{j}(a)+\eta^{2} \sum_{j:j+d^{j}=k} \sum_{a} \pi^{k}(a) m^{k} \l(\ell^{j}(a) \r)^{2}
	                      \\
	                      & = \exp\l( - \eta\sum_{j:j+d^{j}=k} \sum_{a} \pi^{k}(a) \ell^{j}(a)+\eta^{2} \sum_{j:j+d^{j}=k} \sum_{a} \pi^{k}(a) m^{k} \l(\ell^{j}(a) \r)^{2} \r),
\end{align*}
where the first inequality is since $e^{x} \leq1+x+x^{2}$ for $x\leq1$,
the second inequality is since $\Vert x\Vert_{1}^{2} \leq n\Vert x\Vert_{2}^{2}$
for any $x\in\bbR^{n}$, and the last inequality is since $1+x\leq e^{x}$.
Telescoping the ratio above for $k= 1,....,K$ we get,
\begin{align*}
	\frac{W^{K+1}}{W^{1}} & \leq \exp\l( - \eta\sum_{k= 1}^{K} \sum_{j:j+d^{j}=k} 
	                            \sum_{a} \pi^{k}(a) \ell^{j}(a)+\eta^{2} \sum_{k= 1}^{K} \sum_{j:j+d^{j}=k} \sum_{a} \pi^{k}(a) m^{k} \l(\ell^{j}(a) \r)^{2} \r). 
	                      \\
	                      & \leq \exp\l( - \eta\sum_{k= 1}^{K} \sum_{a} \pi^{k+d^{k}}(a) 
	                            \ell^{k}(a)+\eta^{2} \sum_{k= 1}^{K} \sum_{a} \pi^{k+d^{k}}(a) m^{k+d^{k}} \l(\ell^{k}(a) \r)^{2} \r).                   
\end{align*}
where we used that for all $j$, $j+d^{j} \leq K$ (we can assume w.l.o.g that all the missing feedback is observed in the end of the interaction).  
On the other hand,
\begin{align*}
	\frac{W^{K+1}}{W^{1}} & \geq \frac{\sum_{a} \exp\l( - \eta\sum_{j:j+d^{j} \leq K} \ell^{j}(a) \r)}{A}           \\
	                      & \geq \frac{\max_{a} \exp\l( - \eta\sum_{j:j+d^{j} \leq K} \ell^{j}(a) \r)}{A}           \\
	                      & \geq \frac{\exp\l( - \min_{a} \eta\sum_{j:j+d^{j} \leq K} \ell^{j}(a) \r)}{A}           \\
	                      & \geq \frac{\exp\l( - \eta\sum_{j:j+d^{j} \leq K} \sum_{a} \pi^{*}(a) \ell^{j}(a) \r)}{A} \\
	                      & = \frac{\exp\l( - \eta\sum_{k= 1}^{K} \sum_{a} \pi^{*}(a) \ell^{k}(a) \r)}{A},           
\end{align*}
where again we used that for all $j$, $j+d^{j} \leq K$ for the last
equality. Combining the last two inequalities taking $\ln$ on both
sides and rearranging the terms we get,
\[
	\sum_{k= 1}^{K} \sum_{a} \l(\pi^{k+d^{k}}(a) - \pi^{*}(a) \r) \ell^{k}(a) 
	\leq 
	\frac{\ln A}{\eta} + \eta\sum_{k= 1}^{K} \sum_{a} \pi^{k+d^{k}}(a) m^{k+d^{k}} \l(\ell^{k}(a) \r)^{2}. \qedhere
\]

\end{proof}

\begin{lemma}[Lemma 1 of \citet{cesa2019delay} adapted to negative losses]
    \label{lemma:elementwise diff}
    Let $\pi,\tilde\pi \in \Delta_A$ and $\ell \in [-M,\infty)^A$ such that 
    \[
        \tilde \pi(a) = \frac{\pi(a)e^{-\eta\ell(a)}}{\sum_{a'}\pi(a')e^{-\eta\ell(a')}}.
    \]
    
    It holds that,
    \begin{align*}
        - \eta \pi(a) \l( \ell(a) + M \r)
        \leq
        \tilde{\pi} (a) - \pi(a) 
        & \leq
        \eta \tilde{\pi} (a) \sum_{a'} \pi(a') \l( \ell(a') + M \r).
    \end{align*}
\end{lemma}

\begin{proof}
By the condition in the lemma,
\begin{align*}
    \tilde{\pi} (a) 
    & = \frac{\pi(a) 
    \exp\l( - \eta\ell(a) \r)}
    { \sum_{a'} \pi(a') \exp\l( - \eta\ell(a')\r)} 
    \\
    & = 
    \frac{\pi(a) 
    \exp\l( - \eta(\ell(a) + M)\r)}
    { \sum_{a'} \pi(a')
    \exp\l( - (\ell(a') + M)\r)} 
    \\
    & \geq
    \frac{\pi(a) \exp\l( - \eta(\ell(a) + M)\r)}{ \sum_{a'} \pi(a')} 
    \\
    & = 
    \pi(a) \exp\l( - \eta(\ell(a) + M)\r)
    \\
    & \geq
    \pi(a) \l(1 - \eta(\ell(a) + M)\r),
\end{align*}
where in the first inequality we use the fact that $\ell(a') + M \geq 0$
and so the exponent at the denominator $\leq 1$; and the second inequality
is by $e^{ - x} \geq1 - x$. Thus,
\[
    \tilde{\pi} (a) - \pi(a) 
    \geq
    - \eta \pi(a) \l(\ell(a) + M \r).
\]
Similarly,
\begin{align*}
    \tilde{\pi} (a) 
    & = \frac{\pi(a) 
    \exp\l( - \eta \ell(a) \r)}
    { \sum_{a'} \pi(a') 
    \exp\l( - \eta\ell(a')\r)} \\
    & = 
    \frac{\pi(a) 
    \exp\l( - \eta(\ell(a) + M)\r)}
    { \sum_{a'} \pi(a') 
    \exp\l( - \eta( \ell(a') + M)\r)} 
    \\
    & \leq
    \frac{\pi(a) }
    { \sum_{a'} \pi(a') 
    \exp\l( - \eta( \ell(a') + M)\r)},
\end{align*}
where the inequality is since $\ell(a) + M \geq0$.
Thus,
\begin{align*}
    \tilde{\pi} (a) - \pi(a)  & \leq\tilde{\pi} (a) \l(1 - \sum_{a'} \pi(a') \exp \l( - \eta( \ell(a') + M)\r)\r)
    \\
    & = \tilde{\pi} (a) \sum_{a'} \pi(a') \l(1 - \exp\l( - \eta( \ell(a') + M)\r)\r)
    \\
    & \leq
    \eta \tilde{\pi} (a) \sum_{a'} \pi(a') \l( \ell(a') + M \r). \qedhere
\end{align*}
\end{proof}


\begin{lemma}[\citet{thune2019nonstochastic}]
    \label{lemma:sum-delayed-indicators}
    Let $\{ d^k \}_{k=1}^K$ be a sequence of non-negative delays such that $\sum_{k=1}^K d^k = D$.
    Then, 
    $$
        \sum_{k=1}^{K} \sum_{i=1}^{K} \indevent{k\leq i+d^{i}<k+d^{k}}
        \leq D+K.
    $$
\end{lemma}
\begin{proof}
    The proof appears as part of the proof of Theorem 1 in \citet{thune2019nonstochastic} or as a separate lemma in \citet[Lemma C.7]{jin2022near}
\end{proof}

\begin{lemma}
    \label{lemma:PSDs}
    If $A$ is a positive semi-definite (PSD) matrix and $\gamma > 0$, then for any vector $x\in\bbR^d$, 
    \begin{align}
        \label{eq:PSDs invA}
        x^{T}(A+\gamma I)^{-1}x   & \leq    \frac{1}{\gamma} \Vert x \Vert^2_2
        \\
        \label{eq:PSDs invA invA}
        x^{T}(A+\gamma I)^{-2}x   & \leq    \frac{1}{\gamma^2} \Vert x \Vert^2_2
        \\
        \label{eq:PSDs invA A}
        x^{T}(A+\gamma I)^{-1} A x   & \leq  \Vert x \Vert^2_2
        \\
        \label{eq:PSDs invA A 2}
        x^{T}(A+\gamma I)^{-1} A^2 (A+\gamma I)^{-1}  x   & \leq  \Vert x \Vert^2_2
        \\
        \label{eq:PSDs invA A invA}
        x^{T}(A+\gamma I)^{-1}A(A+\gamma I)^{-1}x   & \leq    x^{T}(A+\gamma I)^{-1}x
    \end{align}

\end{lemma}

\begin{proof}
By the spectral decomposition $A=U^{T}DU$ where $U$
is orthogonal matrix and $D=diag(\lambda_{1},...,\lambda_{d})$ with
$\lambda_{i}\geq0$. Note that,
\[
	A+\gamma I=U^{T}(D+\gamma I)U\Longrightarrow(A+\gamma I)^{-1}=U^{T}(D+\gamma I)^{-1}U.
\]

Denote $y=Ux$. Then,
\begin{align*}
	x^{T}(A+\gamma I)^{-1}x = y^{T} (D+\gamma I)^{-1} y                           =\sum_{i=1}^{d}\frac{1}{\lambda_{i} + \gamma}y_{i}^{2} \leq\sum_{i=1}^{d}\frac{1}{\gamma}y_{i}^{2}       
	                                            = \frac{1}{\gamma} \Vert y \Vert^2_2                                        
	                                           =\frac{1}{\gamma} \Vert x \Vert^2_2,
\end{align*}
which establishes \cref{eq:PSDs invA}. \cref{eq:PSDs invA invA} is done similarly  by noting that $(A+\gamma I)^{-2}=U^{T}(D+\gamma I)^{-2}U$ since $U^{T}U=I$.
For \cref{eq:PSDs invA A},
\[
	(A+\gamma I)^{-1}A=U^{T}(D+\gamma I)^{-1} D U.
\]
Hence,
\begin{align*}
	x^{T}(A+\gamma I)^{-1}Ax = y^{T} (D+\gamma I)^{-1}D y                           =\sum_{i=1}^{d}\frac{\lambda_{i}}{\lambda_{i} + \gamma}y_{i}^{2} 
	\leq\sum_{i=1}^{d} y_{i}^{2}
	                                            =  \Vert y \Vert^2_2              
	                                          = \Vert x \Vert^2_2.
\end{align*}
\cref{eq:PSDs invA A 2} is done similarly.
For \cref{eq:PSDs invA A invA},
\[
	(A+\gamma I)^{-1}A(A+\gamma I)^{-1}=U^{T}(D+\gamma I)^{-1}D(D+\gamma I)^{-1}U.
\]
Thus,
\begin{align*}
	x^{T}(A+\gamma I)^{-1}A(A+\gamma I)^{-1}x & =y^{T}(D+\gamma I)^{-1}D(D+\gamma I)^{-1}y                           =\sum_{i=1}^{d}\frac{\lambda_{i}}{(\lambda_{i}+\gamma)^{2}}y_{i}^{2} \\
	                                          & \leq\sum_{i=1}^{d}\frac{1}{\lambda_{i}+\gamma}y_{i}^{2}               =y^{T}(D+\gamma I)^{-1}y                                             \\
	                                          & =x^{T}U^{T}(D+\gamma I)^{-1}Ux                                        =x^{T}(A+\gamma I)^{-1}x. \qedhere                                            
\end{align*}
\end{proof}

\newpage
\section{DAPPO Implementation Details and Additional Experiments}
\label{appendix: experiments}

\subsection{DAPPO Implementation Details}

Our experiments are based on the implementation of PPO from the Stable-Baselines3 library \cite{raffin2021stable}. 
All of the implementation details remain identical to the original implementation (including the architecture of the Deep Neural Networks and the default hyper-parameters), except for the two following modifications: (i) The objective of DAPPO, and (ii) We mimic learning with delayed feedback by withholding feedback from the algorithm for $d$ steps.

PPO maintains a policy network $\pi^{\theta}$ and a value network $V^{\phi}$. 
In each round $k$, the algorithm collects a rollout $(s_1^k,a_1^k,r_1^k,...,s_H^k,a_H^k,r_H^k)$ of length $H = 2048$ (notice that for the experiments we switched from costs to rewards).
Note that the rollout is of length $H$, regardless of the number of episodes it takes to fill the rollout buffer to be of that length. 
That is, if for example the environment has a termination state and the episode ends before time $H$, we start a new episode and keep filling the buffer until we reach $H$ interactions of the policy with the environment. 
That way, we can emulate fixed finite horizon MDPs as in our setting, even if the environment is not of fixed horizon. 
To this end, we treat each rollout of length $H$ as a single episode. 
Apart from the state, action, reward and next state, the rollout buffer also stores the probability to take the chosen action $\pi^{\theta^k}(a_h^k\mid s_h^k)$ for each $h\in[H]$.

Now, since we want to simulate delayed feedback, we do not use the buffer of round $k$ to update $\pi^{\theta^k}$. 
Instead, we store this buffer and load the buffer from round $k-d$, where $d$ is the delay in terms of episodes and not in terms of timesteps. 
I.e., if the delay in terms of timesteps is $\tilde d$, then $d = \lfloor \tilde d/H \rfloor$.
At this point, the policy network objectives for DPPO and DAPPO are $L_D^k(\theta)$ and $L_{DA}^k(\theta)$, respectively, where,
\begin{align*}
     L^k_{D}(\theta) & = 
     \sum_{h=1}^H \min \l \{ g^k_h(\theta) \hat{A}_h^{k - d} 
     , \text{clip}_{1 \pm \epsilon}\l( g^k_h(\theta) \r) \hat{A}_h^{k - d} 
     \r\}; 
    \\
     L^k_{DA}(\theta) & = 
     \sum_{h=1}^H \min \l \{ R_h^k(\theta) \hat{A}_h^{k - d}  
     , \text{clip}_{1 \pm \epsilon}\l( R^k_h(\theta) \r) \hat{A}_h^{k - d} 
     \r\},
     \label{eq:DPPO and DAPPO objective appendix}
\end{align*}
for $g_h^k(\theta) =\frac{\pi^{\theta}(a_h^{k-d}\mid s_h^{k-d})}{\pi^{\theta^{k - d}} (a_h^{k-d}\mid s_h^{k-d}) }$ and $R_h^k(\theta) = \frac{\pi^{\theta}(a_h^{k-d}\mid s_h^{k-d})}{ \max\{\pi^{\theta^k} (a_h^{k-d}\mid s_h^{k-d}) , \pi^{\theta^{k - d}} (a_h^{k-d}\mid s_h^{k-d})  \}}$.
$\hat{A}_h^{k-d} = L_h^{k-d} - V^{\phi}(s_h^{k-d})$ is an estimate of the advantage function, where $L_h^{k-d}$ is the realized cost-to-go from $(s_h^{k-d},a_h^{k-d})$ until the first termination state in the rollout buffer.
Note that $L_D^k(\theta)$ is computed solely based on the parameters $\theta$ and on data stored in rollout buffer, and does not require any modification to the the original algorithm. 
On the other hand, DAPPO computes, in addition, the probabilities of the last policy $\pi^{\theta^k}$ over the trajectory of $\pi^{\theta^{k-d}}$ (which has a relatively small computational cost).

The value network is trained simply by optimizing the mean-squared error (MSE) loss $\sum_{h=1}^H (L_h^{k-d} - V^{\phi}(s_h^{k-d}))^2 $.
Finally, the optimization of both the policy network and the value network is done using the Adam optimizer \cite{kingma2014adam} with learning rate $\eta = 0.0003$, batch size of $64$ and for $10$ epochs over the rollout buffer (these parameters remain unchanged from the original implementation of \citet{raffin2021stable}).

\newpage

\begin{figure*}[ht]
    \begin{center}
    \centerline{\includegraphics[width=0.25\textwidth]{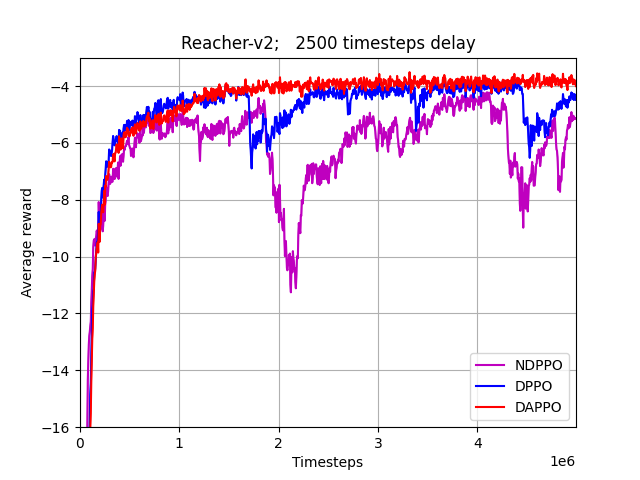}
                \includegraphics[width=0.25\textwidth]{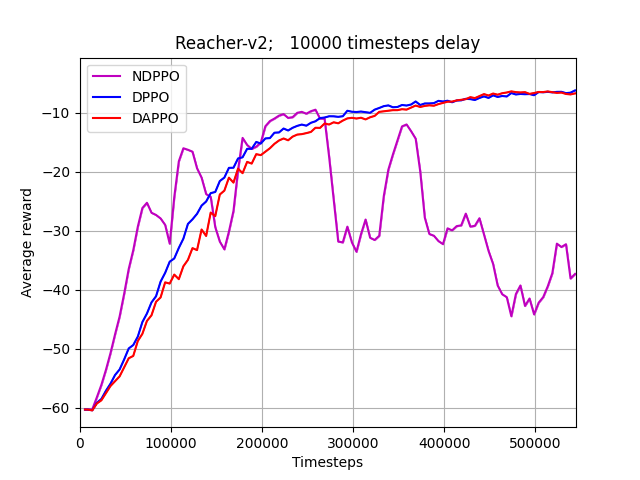}
                \includegraphics[width=0.25\textwidth]{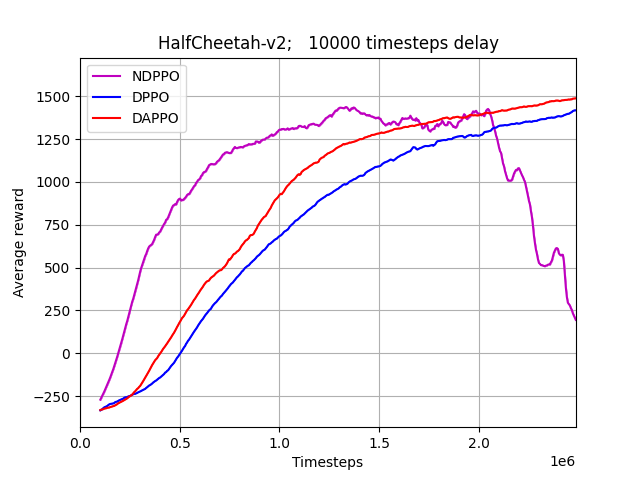}
                \includegraphics[width=0.25\textwidth]{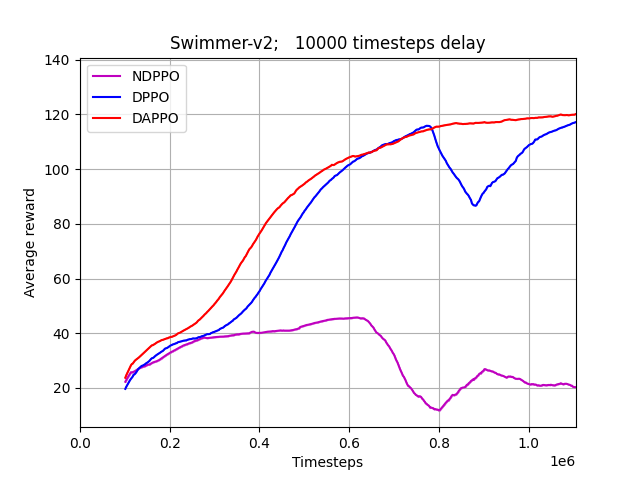}}
    \vskip -0.1in
        \caption{\textbf{Instability of NDPPO:}  The learning curve of NDPPO vs DAPPO and DPPO under various settings. Plots show reward over a single run. x-axis is the number of timesteps until massive drop in performance (or up to 5M)}.
    \label{fig:NDPPO}
    \end{center}
\end{figure*}

\subsection{Additional Experiments -- Instability of NDPPO}

We conducted experiments to show that NDPPO is an unstable algorithm.
Note that $\pi^{\theta^{k - d}} (a_h^{k-d} \mid s_h^{k-d})$ is not likely to be very small since $a_h^{k-d}$ was sampled from $\pi^{\theta^{k - d}} ( \cdot \mid s_h^{k-d})$. 
On the other hand, $\pi^{\theta^{k}} (a_h^{k-d} \mid s_h^{k-d})$ may be effectively $0$ if the probability to choose $a_h^{k-d}$ has decreased dramatically between time $k-d$ and time $k$. 
This emphasizes that NDPPO (described in \cref{sec:PPO}) does not only optimize over a biased objective, but is also highly unstable since the gradient of $L^k_{ND}(\theta)$ is inversely proportional to $\pi^{\theta^{k}} (a_h^{k-d} \mid s_h^{k-d})$. 
To demonstrate this phenomena, we present in \cref{fig:NDPPO} a few runs of NDPPO (compared to DPPO and DAPPO) under various settings. 
In some cases, such as in the \textsc{Swimmer-v2} environment, NDPPO is not able to improve the policy due to the large bias of its estimator.  
In other cases, such as the \textsc{Reacher-v2} or \textsc{HalfCheetah-v2} environments, the learning curve initially behaves similarly to DAPPO and DPPO, and sometimes even slightly better. 
This is due to the fact that $\pi^{\theta^{k}} (a_h^{k-d} \mid s_h^{k-d})$ is likely to be smaller then $\pi^{\theta^{k-d}} (a_h^{k-d} \mid s_h^{k-d})$, leading to larger updates (compared to DPPO and DAPPO). 
However, at some point, the NDPPO's learning curve becomes much noisier due to the dramatic updates whenever $\pi^{\theta^{k}} (a_h^{k-d} \mid s_h^{k-d})\approx 0$. 
This may results in an  unrecoverable drop in performance even when the delay is small (as presented in \cref{fig:NDPPO}), and in general, it gives an unstable algorithm with huge variance.
DAPPO naturally avoids this issue by taking the maximum between $\pi^{\theta^{k}} (a_h^{k-d} \mid s_h^{k-d})$ and $\pi^{\theta^{k - d}} (a_h^{k - d} \mid s_h^{k - d})$.

\subsection{Additional Experiments -- Drop in Performance as Delay Length Increases}
\label{appendix: additional experiments 2}

We conducted experiments to exhibit the drop in performance that occurs when delay length increases.
\cref{fig:different delays} compares the training curves of  DPPO vs DAPPO with different lengths of $\tilde d \in \{10000,25000,50000,75000,100000\}$, alongside the training curve of PPO without delay.
As expected, when the delay is relatively small (e.g., $\tilde d = 10000$), there is no significant difference between learning with or without delayed feedback.
As the delay becomes larger, the performance of all algorithms drops (but at different rates).

One exception is \texttt{HalfCheetah-v2} where even for small delay performance drops. 
However, this is mainly due to the fact that the performance across different seeds is very noisy. 
Specifically, in 3 out of the 5 seeds, PPO without delays converges to a local maxima which has average reward of $\approx 1500$ (similar to DPPO and DAPPO with sufficiently large delay). 
In the two other seeds it converges to a much better policy, which explains the large std in these graphs.

Whenever the delay becomes sufficiently large ($\sim 25K - 50K$), there is a massive drop in the performance. This emphasize the great challenge that online algorithms need to face in the presence of delays. Namely, the algorithm updates its current policy based on estimated advantage function of a very different policy  than the current one. This is also the point where the way we estimate the advantage function becomes important and the difference between DPPO and DAPPO becomes much more significant.

\begin{figure*}[ht]
    \begin{center}
    \centerline{\includegraphics[width=\textwidth]{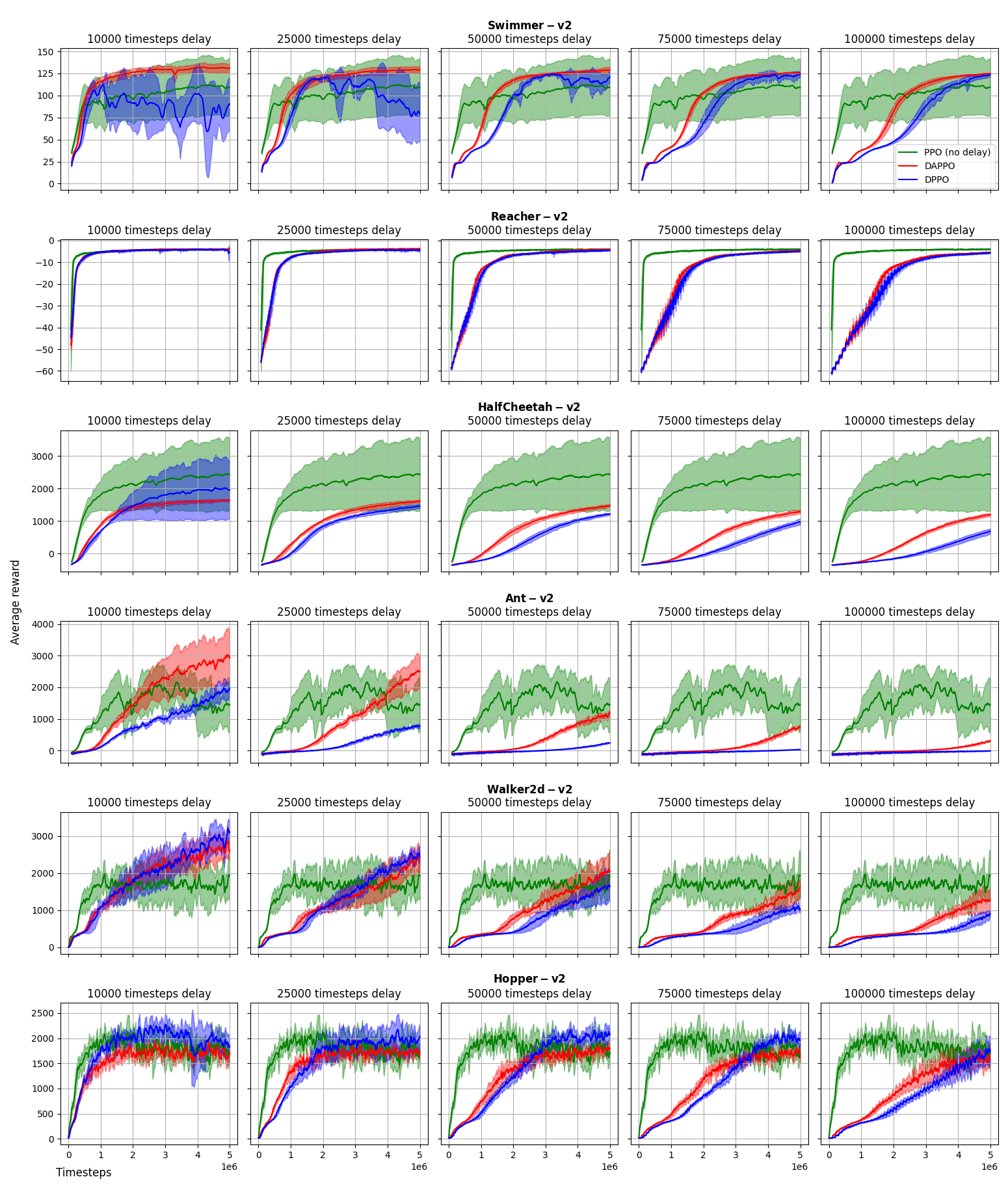}}
    \vskip -0.1in
        \caption{\textbf{Training curves in different enviorments and different fixed delay length:}  DAPPO vs DPPO with different delay, alongside PPO without delays. Plots show average reward and std over 5 seeds. x-axis is number of timesteps up to 5M.}
    \label{fig:different delays}
    \end{center}
\end{figure*}
\end{document}